\documentclass[a4paper]{sn-jnl}
\usepackage{algorithm}
\usepackage{algpseudocode}
\usepackage{amsmath,amsthm}
\usepackage{amsfonts}
\usepackage{array}
\usepackage{graphicx}
\usepackage{blkarray}
\usepackage{soul}
\usepackage{amssymb}
\usepackage{comment}
\usepackage{mathtools}
\usepackage{mathptmx}
\usepackage[table]{xcolor}
\usepackage{tabularray}
\usepackage{geometry}
\usepackage{natbib}

\usepackage[shortlabels]{enumitem}
\newcommand{\notminus}{\mathrel{\mathrlap{\;/}\:-}}

\DeclareFontFamily{U}{mathx}{}
\DeclareFontShape{U}{mathx}{m}{n}{<-> mathx10}{}
\DeclareSymbolFont{mathx}{U}{mathx}{m}{n}
\DeclareMathAccent{\widehat}{0}{mathx}{"70}
\DeclareMathAccent{\widecheck}{0}{mathx}{"71}

\newtheorem{theorem}{Theorem}

\newtheorem{proposition}{Proposition}
\newtheorem{remark}{Remark}
\newtheorem{lemma}{Lemma}
\newcommand{\ci}{\mathrel{\perp\mspace{-10mu}\perp}}
\newcommand{\nindep}{\mathrel{\not\!\ci}}
\usepackage{relsize}
\newcommand{\mstar}{ \mathlarger{*}}
\usepackage{hyperref}
\hypersetup{
    colorlinks=true,
    linkcolor=blue,
    filecolor=blue,      
    urlcolor=blue,
    citecolor=blue,
    pdftitle={Overleaf Example},
    pdfpagemode=FullScreen,
    }
\date{}
\begin{document}
\title{Constraint- and Score-Based Nonlinear Granger Causality Discovery with Kernels}
\author*[1]{\fnm{Fiona} \sur{Murphy}}\email{murphf20@tcd.ie}
\author[1]{\fnm{Alessio} \sur{Benavoli}}\email{alessio.benavoli@tcd.ie}

\affil*[1]{\orgdiv{School of Computer Science and Statistics}, \orgname{Trinity College Dublin}, \orgaddress{\city{Dublin},  \country{Ireland}}}

\abstract{
    Kernel-based methods are used in the context of Granger Causality to enable the identification of nonlinear causal relationships between time series variables. In this paper, we show that two state of the art kernel-based Granger Causality (GC) approaches can be theoretically unified under the framework of Kernel Principal Component Regression (KPCR), and introduce a method based on this unification, demonstrating that this approach can improve causal identification. Additionally, we introduce a Gaussian Process score-based model with Smooth Information Criterion penalisation on the marginal likelihood, and demonstrate improved performance over existing state of the art time-series nonlinear causal discovery methods. Furthermore, we propose a contemporaneous causal identification algorithm, fully based on GC, using the proposed score-based $GP_{SIC}$ method, and compare its performance to a state of the art contemporaneous time series causal discovery algorithm.
}


\maketitle

\section{Introduction}

Granger causality (GC) \citep{granger_1969} is a time series causal discovery framework that uses predictive modeling to identify the underlying causal structure of a time series system. Relying on the assumption that cause precedes effect, GC assesses whether including the lagged information from one time series in the autoregressive model of a second time series enhances its predictions. This improvement indicates 
a predictive relationship between the time series variables, where one time series provides supplemental information about the future of another time series, thereby signifying the presence of a (Granger) causal relationship. GC requires only observational data, and has been used for time series causal discovery across diverse domains, including climate science \citep{runge_inferring_2019}, political and social sciences \citep{ivascu_new_2022}, econometrics \citep{baum_dynamics_2025}, and biological systems studies \citep{emad_caspian_2014}.

The original formulation of GC requires several assumptions to be satisfied for causal identifiability. In regards to the candidate time series system, it is assumed that the time series variables are stationary, and that all variables are observed (absence of latent confounders). GC was initially proposed for bivariate time series systems, but was generalised for the multivariate setting to accommodate the assumption that all relevant variables are included in the analysis \citep{granger_1969}. 
Additional assumptions are made with regard to the types of causal relationships that can be identified within the time series system. GC cannot estimate a causal relationship between time series at an instantaneous time point, relying on the relationship between the lags and predicted values to determine a GC relationship. The initial framework was also limited to assessing only linear relationships. Some of these limitations have been addressed for the time-series causal discovery setting through extensions for nonlinear multivariate analysis \citep{marinazzo_kernel_2008, wismuller_large-scale_2021, runge_detecting_2019}, and the identification of instantaneous causal relationships 
\citep{runge_discovering_2022, peters_causal_nodate}. We introduce the original mathematical formulation of GC in Section \ref{sec:GC_background}.

We aim to compare the theoretical foundations and causal identification performance of several existing constraint- and score-based nonlinear GC methods, as well as introduce two novel kernel-based methods for causal identification: one constraint-based and one score-based. The score-based approach, based on Gaussian Processes (GPs), can also be used to identify contemporaneous causal relationships.

Constraint-based GC methods apply hypothesis testing 
in order to uncover the underlying causal structure. Alternatively, score-based approaches search through possible causal graphs based on an associated scoring function, such as an information criterion, to identify the graph structure with the optimal score \citep{runge_causal_2023}. The constraint-based method proposed by \cite{marinazzo_kernel_2008} introduces a kernel Granger causality (KGC) framework which uses kernel methods to project the data into a Reproducing Kernel Hilbert Space (RKHS), enabling the identification of nonlinear causal relationships. The constraint-based large-scale nonlinear Granger causality (lsNGC) method proposed in \cite{wismuller_large-scale_2021} approaches nonlinear causal identification by coupling a dimensionality reduction step via k-means clustering with a nonlinear transformation of the input data using a generalised radial basis function (GRBF) network. These methods are discussed in Section \ref{sec:joint}. These methods, both employing kernel-based regression to perform nonlinear GC analysis, are closely related.

Our first contribution is to show that they can be unified under the framework of Kernel Principal Component Regression (KPCR). This unification allows us to derive a novel kernel-based method that combines the strengths of both approaches, discussed in Section \ref{sec:kpcr}. 

Another state of the art constraint-based approach in time-series causal discovery is the PCMCI method \citep{runge_detecting_2019}, which uses iterative conditional independence testing to determine the structure of the causal graphical model. For nonlinear causal identification, the PCMCI method is usually implemented using a Gaussian Process Distance Correlation (GPDC) conditional independence test, which can flexibly learn a variety of nonlinear relationships.

Score-based approaches for causal discovery in time series that employ kernel methods typically rely on GPs and use the marginal likelihood as the scoring function. In 
\cite{amblard_gaussian_2012}, Granger causality is determined if the maximised marginal likelihood of a GP model containing the driver variable is greater than a model which excludes it. This comparison of marginal likelihoods is performed in \cite{zaremba_statistical_2022} via a Generalised Likelihood Ratio Test (GLRT). 
The work \cite{cui_gaussian_2022} also uses GPs for learning the topology of directed graphs via a marginal-likelihood scoring function, but employs an Automatic Relevance Determination (ARD) kernel and determines causality directly from the optimised kernel hyperparameters such that only one GP model is required. In \cite{cui_topology_2024}, this approach is extended to a fully Bayesian model using Markov Chain Monte Carlo (MCMC) sampling and incorporates a sparsity prior. 

Our second contribution is to improve the approach in  \cite{cui_gaussian_2022} by performing model selection using an information-based criterion added to the marginal likelihood, resulting in a more traditional score-based criterion for causal discovery. In particular, we include the Smooth Information Criterion (SIC) \citep{oneill_variable_2023}, a differentiable approximation to the $L_0$ norm. SIC provides a smoothed approximation of the Minimum  Information Criterion \citep{su2018sparse}, which is known to outperform many regularization methods and information-based selection criteria. This score-based approach is discussed in Section \ref{sec:ourgp}.

Our third contribution is the derivation of a fully Granger-causality-based statistical procedure for identifying contemporaneous causal links. In particular, under suitable restrictions on the ground-truth causal graph, we show we can  recover contemporaneous causal adjacencies by running GC twice. First, we estimate a graph $G'$ by conditioning on all contemporaneous and lagged variables. Second, to block spurious lagged pathways, we learn a new graph $G''$ that includes only the lagged parents estimated in the first graph $G'$. Finally, we construct a third graph $G'''$ that contains all lagged edges from $G''$ and all contemporaneous edges from $G'$. We then exploit conditional independence relations and Markov equivalence classes to orient as many contemporaneous adjacencies as possible.

We compare the performance of the GC methods on several simulated datasets that feature different types of nonlinear causal relationships. We additionally discuss an application in pH neutralisation plant modelling, evaluating the GC methods on a MATLAB Simulink \citep{MATLAB} simulated plant \citep{michalowski_modelling_2011}, with broader applicability in control engineering. These results are presented in Section \ref{sec:results}.

\subsection{Granger Causality}
\label{sec:GC_background}
In this section, we present the basic Granger causality method. We denote an element of a time series as $\xi_{li}$, where $l \in \{a,b,c,\dots,n_t\}$ is the label of the time series and $i=1,2,\dots,k$ is the time index.
 As proposed in \cite{granger_1969}, Granger causality tests the causal effect of one time series $\boldsymbol{\xi}_a = [\xi_{a1}, \dots, \xi_{a(n+m)}]$ on a second time series $\boldsymbol{\xi}_b = [\xi_{b1}, \dots, \xi_{b(n+m)}]$, where there are observations for a total of $n+m$ time points, and $m$ is the selected lag order. GC is assessed via the linear autoregressive models, 
\begin{align}
    &\xi_{bi} = \beta_0 + \beta_1\xi_{b(i-1)} + \dots + \beta_m\xi_{b(i-m)} + \epsilon_i, ~~\epsilon_i \sim N(0, \sigma^2), \label{eq:restricted_model} \\
    &\xi_{bi} = \beta_0 + \beta_1\xi_{b(i-1)} + \dots + \beta_m\xi_{b(i-m)} + \alpha_1\xi_{a(i-1)} + \dots + \alpha_m\xi_{a(i-m)} + \epsilon_i, ~~\epsilon_i \sim N(0, \sigma^2), \label{eq:unrestricted_model} 
\end{align}
where the restricted model \eqref{eq:restricted_model} estimates the regression coefficients $\boldsymbol{\beta} = [\beta_0, \beta_1, \dots, \beta_m]$ using only the lags of time series $\boldsymbol{\xi}_b$. The unrestricted model \eqref{eq:unrestricted_model} includes the lagged values of both time series $\boldsymbol{\xi}_b$ and $\boldsymbol{\xi}_a$, and additionally estimates the regression coefficients $\boldsymbol{\alpha} = [\alpha_1, \dots, \alpha_m]$ associated with the lags of $\boldsymbol{\xi}_a$. Both models assume additive Gaussian noise, which has zero mean and fixed variance $\sigma^2$, and is independent across time points. 

The regression coefficients can be interpreted such that a Granger causal relationship exists if any of the estimated regression coefficients of time series $\boldsymbol{\xi}_a$ differ from $0$, indicating that time series $\boldsymbol{\xi}_a$ contributes additional information to the predictive model of time series $\boldsymbol{\xi}_b$. This is formalised as a comparison of the variance of the residuals using an f-test, such that  $\boldsymbol{\xi}_a \xrightarrow{GC} \boldsymbol{\xi}_b$ if the variance of the prediction error for the unrestricted model is less than that of the restricted model, i.e. $Var(\boldsymbol{\xi}_{b} - \hat{\boldsymbol{\xi}}'_{b}) < Var(\boldsymbol{\xi}_{b} - \hat{\boldsymbol{\xi}}_{b})$, where $\hat{\boldsymbol{\xi}}_{b}$ are the predictions of the restricted model and $\hat{\boldsymbol{\xi}}'_{b}$ are the predictions of the unrestricted model. Again, it can be understood that if the unrestricted model reduces the variance of the predictive error, then there is significant predictive gain by including information from the additional time series. The test statistic for the f-test is
\begin{align}
    \label{eq:fstat}
    F_{a \rightarrow b} = \frac{(SSR_r - SSR_u)/(p_u - p_r)}{(SSR_u)/(n - p_u)}, 
\end{align}
where $SSR_r, SSR_u$ are the sum of squared residuals for the restricted and unrestricted models, $p_r, p_u$ are the number of parameters in each model, and $n$ is the number of time points at which predictions were made. 

The multivariate extension assesses Granger causal relationships while taking into account all relevant time series variables. This precludes learning spurious relationships due to omitted confounders. In this extension, causal links are identified in a time series system with a total of $n_t$ time series, $Z = [\boldsymbol{\xi}_a, \boldsymbol{\xi}_b,\dots,\boldsymbol{\xi}_{n_t}]^\top$ using vector autoregressive (VAR) models,
\begin{align}
   Z_i = \sum^m_{j=1}A_jZ_{i-j} + \boldsymbol{\epsilon}_i, ~~\boldsymbol{\epsilon}_i \sim N(0,\Sigma),
\end{align}
where $A_j$ is the $n_t \times n_t$ matrix of coefficients for each time lag, $Z_{i-j}$ are the lagged values for every time series in the system, and $\boldsymbol{\epsilon}_i$ is a vector of error terms with covariance $\Sigma$. The error terms are assumed to be independent across time points. Similar to the bivariate case, the unrestricted model includes all time series in the system, while the restricted multivariate model excludes the causal time series in question. 

\subsection{Gaussian Process}
\label{sec:GP_background}
We include some background on Gaussian processes, which are used in several nonlinear Granger causal frameworks for their advantages in nonlinear regression. GPs are non-parametric Bayesian models \citep{rasmussen2006gaussian}, which can be used to flexibly model a regression problem, 
\begin{align}
    \mathbf{y} = f(X) + \boldsymbol{\epsilon},~~ \boldsymbol{\epsilon} \sim N(0, \sigma^2I_n), 
\end{align}
where $\mathbf{y} \in \mathbb{R}^n$ is a vector of $n$ observations,  $X = [\mathbf{x}_1, \dots, \mathbf{x}_n]^\top$ is a matrix of inputs $\mathbf{x} \in \mathbb{R}^d$, $f : \mathbb{R}^d \rightarrow \mathbb{R}$ is the function describing the relationship between the inputs and the observations, and $\boldsymbol{\epsilon}$ is additive Gaussian noise, assumed to be independent across observations and with fixed variance $\sigma^2$. The GP defines a prior distribution over $f$,
\begin{align}
    \label{gp_prior}
    f \sim GP(\mu(\mathbf{x}), k(\mathbf{x}, \mathbf{x}')),
\end{align}
 with mean function $\mu( \mathbf{x})$, often set to $0$, and covariance or kernel function $k(\mathbf{x}, \mathbf{x}') = cov(f(\mathbf{x}), f(\mathbf{x}'))$ for every pair $\mathbf{x}$ and $\mathbf{x}' \in \mathbb{R}^d$. This corresponds to the assumption of a multivariate normal distribution over $f$ across the $n$ inputs, where the mean vector is $\mu(X) = [\mu(\mathbf{x}_1), \dots, \mu(\mathbf{x}_n)]$ and the kernel matrix is $K(X, X) = K(\mathbf{x}_i, \mathbf{x}_j)$ for $i,j = 1,\dots,n$. The kernel function defines the expected similarity of the function outputs given the similarity of the inputs, and directly impacts the properties of the functions defined by the GP prior. A commonly used kernel function, able to flexibly learn many types of nonlinear relationships, is the squared-exponential kernel,
\begin{align}
    \label{rbfkernel}
    k(\mathbf{x}, \mathbf{x}') = \tau^2 e^{-\frac{||\mathbf{x} - \mathbf{x}'||^2}{2l^2}},
\end{align}
with kernel variance $\tau^2 \in \mathbb{R}^+$ and lengthscale $l \in \mathbb{R}^+$. The hyperparameters of the model, including the variance of the noise, are denoted by $\boldsymbol{\theta} = [\tau^2, l,\sigma^2]$. 

The GP prior can be conditioned on observations $\mathbf{y}$ at training inputs $X$ in order to obtain a predictive posterior distribution over the function $f$ at new inputs $X^\mstar = [\mathbf{x}_1^\mstar, \dots, \mathbf{x}_m^\mstar]^\top$, i.e. the predictive posterior distribution of $\mathbf{f}^\mstar = [f(\mathbf{x}_1^\mstar), \dots, f(\mathbf{x}_m^\mstar)]^\top$. This is defined as \citep[Sec. 2.2]{rasmussen2006gaussian},
\begin{equation}
\label{posterior}
   p(\mathbf{f}^\mstar| X, \mathbf{y}, X^\mstar, \boldsymbol{\theta}) = N(\mathbf{f}(X^\mstar);\hat{\mu}(X^\mstar | X, \mathbf{y}; \boldsymbol{\theta}), \hat{K}(X^\mstar, X^\mstar | X; \boldsymbol{\theta})),
\end{equation}
where the posterior mean and kernel functions are
\begin{align}
    &\hat{\mu}(X^\mstar | X, \mathbf{y}; \boldsymbol{\theta})= \mu(X^\mstar;\boldsymbol{\theta})+K(X^\mstar, X;\boldsymbol{\theta})(K(X, X;\boldsymbol{\theta}) + \sigma^2I_n)^{-1}({\bf {y}}-\mu(X;\boldsymbol{\theta})),\\
    &\hat{K}(X^\mstar, X^\mstar | X; \boldsymbol{\theta})=K(X^\mstar,X^\mstar;\boldsymbol{\theta})- K(X^\mstar, X;\boldsymbol{\theta})(K(X, X;\boldsymbol{\theta}) + \sigma^2I_n)^{-1}K(X, X^\mstar;\boldsymbol{\theta}).
\end{align}
The posterior equations depend on the values of the hyperparameters $\boldsymbol{\theta}$, such that model selection can be performed by maximising the log-marginal likelihood of the GP, defined as \citep[Sec. 5.4]{rasmussen2006gaussian}:
\begin{equation}
    \label{marginal-likelihood}
    \log p(\mathbf{y}|X,\boldsymbol{\theta}) = -\frac{1}{2}(\mathbf{y}^\top (K(X, X;\boldsymbol{\theta}) + \sigma^2I_n)^{-1}\mathbf{y} + \log|K(X, X;\boldsymbol{\theta}) + \sigma^2I_n| + n\log(2\pi)).
\end{equation}

\section{Constraint-based methods for nonlinear Granger causality}
\label{sec:constraint}

\subsection{Joint Presentation of KGC and lsNGC}
\label{sec:joint}
In this section, we jointly present the kernel-based methods KGC \citep{marinazzo_kernel_2008} and lsNGC \citep{wismuller_large-scale_2021}, which will later allow us to show that they can be unified under a more general model.
\subsubsection{Input Data}
In both methods, we first consider testing whether the time series $\boldsymbol{\xi}_a = [\xi_{a1}, \dots, \xi_{a(n+m)}]^\top$ Granger causes the time series $\boldsymbol{\xi}_b = [\xi_{b1}, \dots, \xi_{b(n+m)}]^\top$, where $m$ is the lag order and each $\xi_{\cdot i} \in \mathbb{R}$ for $i = 1,\dots,n+m$ is an observation at time step $i$. These two time series belong to a larger set $\mathcal{T}$ of time series including a total of $n_t$ time series. For each time series $c \in \mathcal{T}$, we define the  phase-space reconstruction $X_c$ as a matrix with rows
\begin{align}
    {\bf x}_{ci} = [\xi_{ci},\dots,\xi_{c(i+m-1)}],
\end{align}
for $i = 1,\dots,n$. It follows that the phase-space reconstructions of the two time-series being tested for causality are defined as $X_a$ and $X_b$. The embedding dimension (the lag order) for the phase-space reconstruction $m$ is determined via Cao's method  \citep{cao_practical_1997} for lsNGC, while \cite{marinazzo_kernel_2008} suggests the use of the Akaike information criterion or other nonlinear-dynamics embedding techniques. 
We will then introduce two additional matrices. The matrix $X\in \mathbb{R}^{n \times (n_t-1)m}$  includes data from all time series excluding the \textit{driver} time series $X_a$, that is its rows are $X_i = [{\bf x}_{1i},\dots,{\bf x}_{(a-1)i},{\bf x}_{(a+1)i},\dots,{\bf x}_{n_t i}]$. The matrix  $Z \in \mathbb{R}^{n \times n_tm}$ includes data from all time series, including $X_a$, that is each row is defined as $Z_i = [{\bf x}_{1i}, \dots, {\bf x}_{ai}, \dots, {\bf x}_{n_ti}]$.

\subsubsection{Kernel Granger Causality}
\label{sec:kgc}
Kernel Granger Causality (KGC) \citep{marinazzo_kernel_2008} is a kernel-based method for nonlinear Granger causality. The method works as follows. Assume for instance we aim to test if time series $a$ Granger-causes time series $b$, then the matrices $X,Z$ are defined as:
\begin{equation}
\label{eq:XZ}
\resizebox{0.98\hsize}{!}{%
$\begin{aligned}
X &=\left[ \begin{array}{*{13}c}
\xi_{b1}& \xi_{b2}& \dots & \xi_{bm} & \xi_{c1}& \xi_{c2}&  \dots & \xi_{cm}& \dots & \xi_{n_t1}&  \xi_{n_t2}& \dots & \xi_{n_tm}\\
\xi_{b2}& \xi_{b3}& \dots & \xi_{b(m+1)} & \xi_{c2}& \xi_{c3}&  \dots & \xi_{c(m+1)}& \dots & \xi_{n_t2}&  \xi_{n_t3}& \dots & \xi_{n_t(m+1)}
\\
\vdots& \vdots& \dots & \vdots & \vdots& \vdots&  \dots & \vdots& \dots & \vdots&  \vdots& \dots & \vdots\\
\xi_{bn}& \xi_{b(n+1)}& \dots & \xi_{b(m+n-1)} & \xi_{cn}& \xi_{c(n+1)}&  \dots & \xi_{c(m+n-1)}& \dots & \xi_{n_tn}&  \xi_{n_t(n+1)}& \dots & \xi_{n_t(m+n-1)}\\
\end{array}\right] ~~~~ 
{\bf y} =\left[ \begin{array}{*{1}c}
\xi_{b(m+1)}\\
\xi_{b(m+2)}\\
\vdots\\
\xi_{b(m+n)}\\
\end{array}\right]\\&\\
Z &=\left[ \begin{array}{*{17}c}
\xi_{a1}& \xi_{a2}& \dots & \xi_{am} &\xi_{b1}& \xi_{b2}& \dots & \xi_{bm} & \xi_{c1}& \xi_{c2}&  \dots & \xi_{cm}& \dots & \xi_{n_t1}&  \xi_{n_t2}& \dots & \xi_{n_tm}\\
\xi_{a2}& \xi_{a3}& \dots & \xi_{a(m+1)} & 
\xi_{b2}& \xi_{b3}& \dots & \xi_{b(m+1)} & \xi_{c2}& \xi_{c3}&  \dots & \xi_{c(m+1)}& \dots & \xi_{n_t2}&  \xi_{n_t3}& \dots & \xi_{n_t(m+1)}
\\
\vdots& \vdots& \dots & \vdots & \vdots& \vdots& \dots & \vdots & \vdots& \vdots&  \dots & \vdots& \dots & \vdots&  \vdots& \dots & \vdots\\
\xi_{an}& \xi_{a(n+1)}& \dots & \xi_{a(m+n-1)} & \xi_{bn}& \xi_{b(n+1)}& \dots & \xi_{b(m+n-1)} & \xi_{cn}& \xi_{c(n+1)}&  \dots & \xi_{c(m+n-1)}& \dots & \xi_{n_tn}&  \xi_{n_t(n+1)}& \dots &  \xi_{n_t(m+n-1)}\\
\end{array}\right] 
\end{aligned}$}
\end{equation}
where $Z=[X_a,X]$. Note that, without loss of generality, we have ordered the columns of $Z$ so that the first $m$ columns correspond to the lagged values of the time series $a$, the second $m$ columns to the lagged values of the time series $b$, and the remaining columns are the lagged columns of the remaining time series. 
Following \cite{marinazzo_kernel_2008}, we further assume  that $n>n_tm$ and that the columns of $X,Z,\mathbf{y}$ have been mean-centered by subtracting off the original mean of each column and that the vector $\mathbf{y}$ has been normalised to have norm one, that is $\mathbf{y}^\top\mathbf{y} = 1$.

\paragraph{Linear case:} 
In order to understand this kernel-based approach, we first need to consider linear Granger causality. In this case,  the restricted and unrestricted regression  estimates obtained via least-squares are:
\begin{align}
\hat{{\bf y}}&= X \hat{\boldsymbol{\beta}}, \\
\hat{{\bf y}}'&= X \hat{\boldsymbol{\beta}}'+ X_a\hat{\boldsymbol{\gamma}}',
\end{align}
where $\hat{\boldsymbol{\beta}},\hat{\boldsymbol{\beta}}',\hat{\boldsymbol{\gamma}}'$ are the least-squares estimates of the coefficients. Then we can write $\hat{{\bf y}}=P{\bf y}$ and $\hat{{\bf y}}'=P'{\bf y}$, where $P=X(X^\top X)^{-1} X^\top$ and $P'=Z(Z^\top Z)^{-1} Z^\top$ are projection matrices, and use the following ratio  as a measure of the strength
of the causal interaction 
\begin{equation}
\label{eq:delta}
\delta = 1-\frac{SSR_u}{SSR_r}=1-\frac{({\bf y}-\hat{{\bf y}}')^\top ({\bf y}-\hat{{\bf y}}')}{({\bf y}-\hat{{\bf y}})^\top ({\bf y}-\hat{{\bf y}})}=\frac{(\hat{{\bf y}}')^\top\hat{{\bf y}}'-(\hat{{\bf y}})^\top\hat{{\bf y}}}{{\bf y}^\top {\bf y}-(\hat{{\bf y}})^\top\hat{{\bf y}}},%
\end{equation}    
with $SSR_u,SSR_r$ denoting the sum of squared residuals of the unrestricted and, respectively, restricted model, and where  we have exploited that %
$P=P^\top=P^2$ (similarly for $P'$) to derive the last equality.

Focusing on the numerator only 
\begin{equation}
(\hat{{\bf y}}')^\top\hat{{\bf y}}'-(\hat{{\bf y}})^\top\hat{{\bf y}}={\bf y}^\top(P'-P){\bf y}=({\bf y}-P{\bf y})^\top (P'-P)({\bf y}-P{\bf y})=\sum_{i=1}^{n_r} ({\bf e}^\top {\bf w}_i)^2,
\end{equation}
 where $n_r=m$, and  we have defined ${\bf e}={\bf y}-P{\bf y}$ and rewritten  the projection matrix as $P'-P=\sum_{i=1}^{n_r} {\bf w}_i{\bf w}^\top_i$. Then, we can rewrite \eqref{eq:delta} as
 \begin{equation}
\delta=\frac{\sum_{i=1}^{n_r} ({\bf e}^\top {\bf w}_i)^2}{{\bf e}^\top{\bf e}}=\sum_{i=1}^{n_r} \text{corr}({\bf e}, {\bf w}_i)^2,
\end{equation}
since ${\bf e}^\top{\bf e}={\bf y}^\top{\bf y}-{\bf y}^\top P{\bf y}$, where $r_i=\text{corr}({\bf e}, {\bf w}_i)$ is the Pearson correlation coefficient.

 The filtered GC index is defined as $\delta_f=\sum_{i=1}^{n_f} r_i^2$ \citep{marinazzo_kernel_2008}, where $n_f\leq n_r$ denotes the number of  correlation coefficients that are significantly different from zero. This is evaluated through a Student's $t$-test for correlation. Due to multiple comparisons across the $i = 1, \dots, n_r$ correlation coefficients $r_i$, the Bonferroni correction is applied and the null hypothesis is rejected whenever the p-value is  $p_i < \frac{0.05}{n_r(n_t(n_t-1)/2)}$, where $n_t(n_t-1)/2$ represents the number of pairs of time series in the system being tested. Therefore, time series $a$ is said to Granger-cause time series $b$ (denoted as $X_a \xrightarrow{GC} X_b$) if the corresponding $\delta_f$ is different from zero.

\paragraph{Linear Kernel:}
 In order to extend these derivations  to the nonlinear case, \cite{marinazzo_kernel_2008} considers an equivalent dual formulation based on using the spectral decomposition of the associated kernel matrices $K = X X^\top$ and $K' = Z Z^\top$. To estimate the values $\mathbf{\hat{y}}$ for the restricted regression model, the vector $\mathbf{y}$ is projected onto the hypothesis space $H \subseteq \mathbb{R}^n$, equivalent to the range of the kernel matrix $K$. This projection is performed with the projector $\sum_{i=1}^{n_v} \mathbf{v}_i\mathbf{v}_i^\top$ (which coincides with $P$ defined above),  where $\mathbf{v}_i$ for $i = 1, \dots, n_v$ are the eigenvectors of $K$ with non-zero eigenvalue. 
 The unrestricted regression model adheres to the same structure,  with $P'=\sum_{i=1}^{n_v'} \mathbf{v}'_i(\mathbf{v}'_i)^\top$ projecting $\mathbf{x}$ onto the hypothesis space $H' \subseteq \mathbb{R}^n$, equivalent to the range of the kernel $K'$.

In order to assess the impact of the additional predictor time series in the unrestricted model, the space in which this additional feature exists must be isolated. Since $H \subseteq H'$, there exists another vector subspace $H^\perp \subseteq H'$ such that $H' = H \bigoplus H^\perp$. Explicitly, $H^\perp$ is equivalent to the range of the kernel 
\begin{align}
    \widetilde{K} =(I-P)K'(I-P)= K' - K'P - PK' +P K'P.
\end{align}
Then we can equivalently rewrite $\delta$ as $\delta=\sum_{i=1}^{n_r} r_{i}^2$, 
where $r_i$ is the correlation between the $n_r$ eigenvectors of the kernel $\widetilde{K}$ with non-zero eigenvalue and the residual vector $\bf{e}$.
 
\paragraph{Nonlinear Kernel:}
The extension to the nonlinear case is then performed by replacing the linear kernel with a nonlinear one \citep{marinazzo_kernel_2008}. It should be noted that the selected kernel must satisfy the condition that when performing regression with statistically independent variables, the loss-minimising function $f^\mstar(X)$ should coincide with $f^\mstar(Z)$, meaning that the additional variable $X_a$ contributes no new information.
Two kernels that are known to satisfy this condition are the inhomogeneous polynomial (IP) kernel and the  squared-exponential kernel \eqref{rbfkernel}, (or, more generally, stationary kernels) \citep{ancona2006invariance}. The IP kernel is defined as 
\begin{align}
    \label{eq:polynomial_kernel}
    K_p({\bf x}, {\bf x}') = (1 + {\bf x} {\bf x}'^\top)^p, 
\end{align}
where $p$ is the polynomial degree.\footnote{For the bivariate case, the dimensions of the spaces corresponding to each set of features are $m_1 = \frac{p + m!}{p!m!} - 1$ for $H$, $m_2 = \frac{p + 2m}{p!2m!} - 1$ for $H'$, and $m_3 = m_2 - m_1$ for $H^\perp$.}
We then, for instance, would define the kernel matrices as  $K=K_p(X, X)$ and $K'=K_p(Z, Z)$. We also assume that $K$ has been mean-centered  $K \rightarrow K - P_0K - KP_0 + P_0KP_0$ (similarly for $K'$), where  $P_0 = \frac{1}{n}\mathbf{1}_n\mathbf{1}_n^\top$ and $\mathbf{1}_n$ is a $n \times 1$ dimensional vector of ones. 
 Finally, the kernel over the space of the additional time series is evaluated as $\widetilde{K} = K' - PK' - K'P + PK'P$, and its eigenvectors are computed in order to calculate the filtered GC index as described for the linear kernel case.\footnote{
For the squared-exponential, the lengthscale impacts the dimension of the range of the kernel. As a result, the condition $H \subseteq H'$ is not always satisfied. To limit the number of cases in which this condition is violated, the eigenvectors of $K_l$ are filtered such that the projector $P$ is associated with a new space $L$. This new space $L$ is the $m_1$-dimensional span of the eigenvectors of $K_l$ subject to their corresponding eigenvalue satisfying the constraint $\lambda \geq \mu\lambda_{max}$, where $\lambda_{max}$ is the largest eigenvalue of $K_l$ and $\mu$ is a small constant ($10^{-6}$). This is done similarly for  $K'$.}

\subsubsection{Large-scale Nonlinear Granger Causality}
The large-scale Nonlinear Granger Causality (lsNGC) method \citep{wismuller_large-scale_2021} applies generalised radial basis function (GRBF) neural networks to nonlinear GC, while also exploiting dimensionality reduction via $k$-means clustering to accommodate systems with a potentially large number of time series variables.
Assume again we aim to test if time series $a$ Granger-causes time series $b$, and the matrices $X,Z,\mathbf{y}$ are defined as in \eqref{eq:XZ}. The lsNGC method fits the following restricted and unrestricted regression models 
\begin{align}
    \label{eq:lsngc_modelr}
    &\mathbf{\hat{y}} = \mathbf{f}(X)\hat{\boldsymbol{\beta}},\\
    \label{eq:lsngc_modelu}
    &\mathbf{\hat{y}'} = \mathbf{f}(X)\hat{\boldsymbol{\beta}}' + \mathbf{g}(X_a)\hat{\boldsymbol{\gamma}}',
\end{align}
where the covariates $\mathbf{f}(X)=[f_1(X) \dots, f_{c_f}(X)] \in \mathbb{R}^{n \times c_f}$ and $\mathbf{g}(X_a) = [g_1(X_a),\dots,g_{c_g}(X_a)] \in \mathbb{R}^{n \times c_g}$ are the nonlinear activation functions of the GRBF networks and $c_f,c_g$ are the number of cluster centroids used in a preliminary k-means step for dimensionality reduction. The least-squares estimates of the coefficients are $\hat{\boldsymbol{\beta}},\hat{\boldsymbol{\beta}}',\hat{\boldsymbol{\gamma}}'$. The lsNGC method enforces an additive relationship between the driver time series and the rest of the time series in the system, creating a nested structure between the restricted and unrestricted models which allows for an f-test to be used to compare predictions.

The nonlinear covariates $\mathbf{f}(X)$ and $\mathbf{g}(X_a)$ are computed by first using $k$-means clustering to partition the observations into an empirically pre-determined number of clusters, $c_f$ clusters for $X$ and $c_g$ clusters for $X_a$. The cluster centers obtained from the clustering step are used as parameters of the hidden layer of the GRBF network, and are denoted $V \in \mathbb{R}^{c_g \times m}$ and $U \in \mathbb{R}^{c_f \times (n_t - 1)m}$. The number of neurons in the hidden layer is thus equivalent to the number of clusters, and the output of the hidden layer is a transformed phase-space reconstruction with reduced dimensionality from $X \in \mathbb{R}^{n\times (n_t-1)m}$ to $\mathbf{f}(X) \in \mathbb{R}^{n \times c_f}$ and from $X_a \in \mathbb{R}^{n \times m}$ to $\mathbf{g}(X_a) \in \mathbb{R}^{n \times c_g}$.

These nonlinear functions are defined as  
\begin{align}
    \label{lsngc_kernel_g}
    g_j(X_{ai}) = \frac{e^{-\|X_{ai}-\mathbf{v}(j)\|^2/l^2}}{\sum\limits_{k=1}^{c_g} e^{-\|X_{ai}-\mathbf{v}(k)\|^2/l^2}} , ~\text{ for} ~j=1,\dots,c_g,i=1,\dots, n, \\
    \label{lsngc_kernel_f}
     f_j(X_i) = \frac{e^{-\|X_i-\mathbf{u}(j)\|^2/l^2}}{\sum\limits_{k=1}^{c_f} e^{-\|X_i-\mathbf{u}(k)\|^2/l^2}}, ~\text{ for} ~ ~j=1,\dots,c_f,i=1,\dots, n,
\end{align}
where $l\in\mathbb{R}^+$ is the lengthscale of the kernel and is set as the mean distance between cluster centers, $\mathbf{v}(j)$ is the \emph{jth} cluster center corresponding to the feature $X_{a}$, and $\mathbf{u}(j)$ is the \emph{jth} cluster center corresponding to the features $X$.

 The output layer of the neural network is a linear combination of the transformed phase-space reconstruction where the weights $\hat{\boldsymbol{\beta}}$ are estimated for the restricted model via
$   \hat{\boldsymbol{\beta}} = (\mathbf{f}(X)^\top\mathbf{f}(X))^{-1}\mathbf{f}(X)^\top\mathbf{y}$.
Similarly, we define $Z' = [\mathbf{f}(X), \mathbf{g}(X_a)]$ as the concatenation of the matrices obtained from the hidden layer of the GRBF neural network. Then, the weights $\hat{\boldsymbol{\omega}}' = [\hat{\boldsymbol{\beta}}', \hat{\boldsymbol{\gamma}}']$ of the regression equation are learned for the unrestricted model via 
$ \hat{\boldsymbol{\omega}}' = (Z'^\top Z')^{-1}Z'^\top \mathbf{y}$. 

In order to test whether $X_a$ improves to the predictions of $X_b$ in the unrestricted model, an f-test is conducted by calculating an f-statistic for the hypotheses
$H_0 :\hat{\boldsymbol{\gamma}}' = 0$ versus $H_a : \hat{\boldsymbol{\gamma}}' \neq 0$. 
Referring to the test statistic \eqref{eq:fstat}, the number of parameters of the unrestricted and restricted models are determined by $p_u = c_f + c_g$ and $p_r = c_f$, and $n$ is the number of observations $\mathbf{y}$. The False Discovery Rate correction is applied due to multiple comparisons \citep{wismuller_large-scale_2021} . The null hypothesis is rejected for a p-value $p <0.05$, meaning that the addition of $X_a$ in the unrestricted model improves the predictions of $X_b$, hence $X_a \xrightarrow{GC} X_b$.

\subsection{Unifying the previous methods via Kernel PCR}
\label{sec:kpcr}
In this section, we first recall the concepts of Kernel Principal Component Analysis (Kernel PCA) and Kernel Principal Component Regression (Kernel PCR). Then, we demonstrate that KGC and lsNGC are specific instances of Kernel PCR, followed by hypothesis testing for Granger causality discovery.

\subsubsection{Kernel PCA and PCR}
Consider the following data matrix $X=[{\bf x}_1,{\bf x}_2,\dots,{\bf x}_n]^\top \in \mathbb{R}^{n \times d}$ with ${\bf x}_i \in \mathbb{R}^d$. Kernel PCA projects the observations ${\bf x}$   into a feature space using a feature map $\phi({\bf x})$, where $\phi$ is a nonlinear transformation from $\mathbb{R}^d$ to $\mathbb{R}^{d'}$ (with $d' \geq d$). In Kernel PCA, the objective is to solve the eigenvalue problem of PCA implicitly, avoiding the need for explicit definition of the feature space \citep{bishop2006pattern}[Sec.12.3]. Through the use of the kernel function, we will implicitly perform a centering of the projected data:
 \begin{equation}
 \label{eq:centering}
 \widecheck{\phi}({\bf x}_i)=\phi({\bf x}_i) - \frac{1}{n} \sum_{k=1}^n\phi({\bf x}_k),
 \end{equation}
then compute the covariance matrix 
\begin{equation}
\Sigma= \frac{1}{n}\sum_{i=1}^n \widecheck{\phi}({\bf x}_i) \widecheck{\phi}({\bf x}_i)^\top,
\end{equation}
and its eigenvalue-eigenvectors $\lambda_j,{\bf v}_j$ for $j=1,2,\dots,d'$ as:
 \begin{equation}
 \Sigma {\bf v}_j =\frac{1}{n}\sum_{i=1}^n \widecheck{\phi}({\bf x}_i) \widecheck{\phi}({\bf x}_i)^\top {\bf v}_j=\lambda_j  {\bf v}_j.
\end{equation}
From the above equality, provided that $\lambda_j>0$, we can see that the vector ${\bf v}_j$ corresponds to a linear combination
of the $\widecheck{\phi}({\bf x}_i)$, and so can be written in the form
\begin{equation}
{\bf v}_j = \sum_{i=1}^n \mathrm{a}_{ji} \widecheck{\phi}({\bf x}_i),
\end{equation}
where ${\bf a}_j$ is the vector including the coefficients of the linear combination. By introducing the kernel function $\widecheck{K}({\bf x}_i,{\bf x}_j)= \widecheck{\phi}({\bf x}_i)^\top  \widecheck{\phi}({\bf x}_j)$, we can write the projection  of a point ${\bf x}$ onto the principal components ${\bf v}_j$ as
\begin{equation}
 \widecheck{\phi}({\bf x})^\top {\bf v}_j = \sum_{i=1}^n \mathrm{a}_{ji} \widecheck{K}({\bf x},{\bf x}_i)=\widecheck{K}({\bf x},X){\bf a}_j.
\end{equation}
The vector ${\bf a}_j \in \mathbb{R}^n$ can be computed by solving  the following eigenvalue-eigenvector problem $\widecheck{K}(X,X){\bf a}_j=\lambda_j n {\bf a}_j$ \citep{bishop2006pattern}[Sec.12.3]. Note that the kernel function can be rewritten as:
\begin{equation}
\widecheck{K}({\bf x}_i,{\bf x}_j)= \widecheck{\phi}({\bf x}_i)^\top  \widecheck{\phi}({\bf x}_j)=K({\bf x}_i,{\bf x}_j)-\frac{1}{n}\mathbf{1}_n^\top K(X,{\bf x}_j)-\frac{1}{n}K({\bf x}_,X) \mathbf{1}_n +\frac{1}{n^2}\mathbf{1}_n^\top K(X,X)\mathbf{1}_n,
 \end{equation}
 where $K({\bf x}_i,{\bf x}_j)= {\phi}({\bf x}_i)^\top  {\phi}({\bf x}_j)$, such that $K$ is centered to obtain $\widecheck{K}$ without needing to explicitly compute the centered feature vectors \eqref{eq:centering}. Indeed, we have that $ \widecheck{K}(X,X)=K(X,X)- P_0K(X,X) - K(X,X)P_0 + P_0 K(X,X) P_0$, where  $P_0 = \frac{1}{n}\mathbf{1}_n\mathbf{1}_n^\top$.
 
With an abuse of notation, we use   $\lambda_j , {\bf a}_j$ to also denote the eigenvalues and eigenvectors of $\widecheck{K}(X,X)$. Assume that the eigenvectors ${\bf a}_j$ are ordered such that ${\bf a}_1$ corresponds to the largest eigenvalue, followed by the others in decreasing order. 
In Kernel PCA, the resulting principal component projections are:
\begin{equation}
 \label{eq:PCApc}
 \left[\widecheck{K}({\bf x}_i,X)\tfrac{{\bf a}_1}{\sqrt{\lambda_1}},\dots, \widecheck{K}({\bf x}_i,X)\tfrac{{\bf a}_s}{\sqrt{\lambda_s}}\right],
 \end{equation}
 for each ${\bf x}_i \in X$,  where $s\leq n$ denotes the first $s$-principal components (with non-zero eigenvalues). 

Kernel PCR reduces to solving a regression problem where the covariates are given by the projections of the data onto the principal components in the feature space induced by the kernel function, i.e., the covariates are the columns in \eqref{eq:PCApc}.

\paragraph*{Reduced-rank approximation:} A drawback of Kernel PCA and PCR is the need to compute the eigenvalue-eigenvector decomposition of an $n \times n$ matrix, which can be computationally prohibitive for large $n$. A common approach to speed up this computation is to use a reduced-rank approximation of this matrix.
Let $X_J=[{\bf x}_1,{\bf x}_2,\dots,{\bf x}_{n_j}]^\top$ be a set of \textit{inducing points}, then the reduced-rank approximation of the kernel matrix is
\begin{equation}
\label{reduced_rank_kernel}
\widetilde{K} = K(X,X_J)K^{-1}(X_J,X_J)K(X_J,X),
\end{equation}
which is called  Nystr{\"o}m approximation \citep{williams2000using}. 

We can rewrite the kernel matrix in \eqref{reduced_rank_kernel} as $\phi(X)\phi(X)^\top$, where  $\phi({\bf x}_i)=K({\bf x}_i,X_J)K^{-1/2}(X_J,X_J)$, directly providing a feature approximation of the original kernel. Since $\phi(X) \in \mathbb{R}^{n \times n_j}$, we can formulate Kernel PCR using ``primal formulation'', that is by computing the eigenvalue-eigenvector decomposition of $\phi(X)^\top\phi(X)$ which belongs to $\mathbb{R}^{n_J\times n_J}$, resulting in the principal component projections:
\begin{equation}
 \label{eq:PCApcnystrom}
 \left[{\phi}({\bf x}_i)^\top{\bf v}_1,{\phi}({\bf x}_i)^\top{\bf v}_2,\dots,{\phi}({\bf x}_i)^\top{\bf v}_s\right],
 \end{equation}
 for each ${\bf x}_i \in X$,  where ${\bf v}_i$ denotes the \emph{i-th} eigenvector of ${\phi}(X)^\top {\phi}(X)$ and $s\leq n_j$ denotes the first $s$-principal components (with non-zero eigenvalues). Note that, also in this case, the features $\phi({\bf x}_i)$ need to be centered, as in \eqref{eq:centering}, before performing the above operations.
 
\subsubsection{KGC and lsNGC are particular cases of Kernel PCR}
\label{sec:unified_propositions}
In the following propositions, we prove that lsNGC and KGC are particular cases of Kernel PCR with and, respectively,  without a reduced rank approximation. The proofs for Proposition \ref{prop:1} and Proposition \ref{prop:2} are included in the supplementary material.

\begin{proposition}
\label{prop:1}
KGC is Kernel PCR with the two regression models 
\begin{align}
    \label{eq:KGC1}
    &\mathbf{\hat{y}} = \mathbf{f}(X)\hat{\boldsymbol{\beta}},\\
    \label{eq:KGC2}
    &\mathbf{\hat{y}'} = \mathbf{f}'(Z)\hat{\boldsymbol{\beta}}',
\end{align}
where $X, Z$ are defined by \eqref{eq:XZ}. The covariate matrices $\mathbf{f}(X)=[f_1(X), \dots, f_{s}(X)]$ with each $f_i(X) \in \mathbb{R}^n$ and $\mathbf{f}'(Z)=[f_1(Z), \dots, f_{s'}(Z)]$ with each $f'_i(Z) \in \mathbb{R}^n$ thus have columns equivalent to the principal component projections \eqref{eq:PCApc}, where the number of principal components with eigenvalue greater than $0$ are denoted as $s, s'$ for the respective models. 
\end{proposition}

\begin{proposition}
 \label{prop:2}
The model \eqref{eq:lsngc_modelr} of the lsNGC method implements a Kernel PCR with Nystr{\"o}m approximation. Assume that the inducing points of the Nystr{\"o}m approximation are selected as the cluster centroids obtained by applying k-means clustering to $X$, resulting in $X_J = [{\bf x}_1, {\bf x}_2, \dots, {\bf x}_{n_J}]^\top$, and consider the following kernel function\footnote{This is a scaled squared-exponential kernel, which is a valid kernel, see for instance \cite{williams2006gaussian}.}
\begin{equation}
\label{eq:rbfscaled}
K({\bf x}_i,{\bf x}_j) = \frac{e^{-\|{\bf x}_i-{\bf x}_j\|^2/l^2}}{\sum_{k=1}^{n_J} e^{-\|{\bf x}_i-{\bf x}_k\|^2/l^2}\sum_{k=1}^{n_J} e^{-\|{\bf x}_k-{\bf x}_j\|^2/l^2}},
\end{equation}
then Kernel PCR is equivalent to \eqref{eq:lsngc_modelr}.
\end{proposition}

\begin{remark} 
The use of k-means to choose the location of the inducing points in the Nystr{\"o}m approximation  is common in the literature \cite{kumar2012sampling,zhang2008improved} and theoretically motivated in \cite{oglic2017nystrom}. The kernel \eqref{eq:rbfscaled} is also frequently used in the literature as it is simply a scaled kernel \cite[Sec.\ 4.2]{rasmussen2006gaussian}.
\end{remark}
By Proposition \ref{prop:2}, we establish the equivalence between Kernel PCR and equation \eqref{eq:lsngc_modelr} in lsNGC. For equation \eqref{eq:lsngc_modelu}, lsNGC assumes that Kernel PCA is applied independently to $X$ and $X_a$ to build the Kernel PCR covariates $f(X),g(X_a)$ and, therefore, construct a linear regression model that includes the covariates of \eqref{eq:lsngc_modelr}  when $\gamma' = 0$. In other words, these two independent nonlinear mappings  $f(X),g(X_a)$ ensure that \eqref{eq:lsngc_modelr}  is nested within \eqref{eq:lsngc_modelu}. This property is leveraged in lsNGC to perform hypothesis testing for causal discovery.  

However, this separated non-linear  mapping $f(X),g(X_a)$ is not the most general structure that preserves the nesting between the restricted and unrestricted models. As discussed in \cite{ancona2006invariance}, applying a polynomial kernel to $X$ (for the restricted model) and, respectively, $Z$ (unrestricted model) maintains this nesting property, as do stationary kernels.  
Note that, \eqref{eq:rbfscaled} is not stationary due to the scaling terms. However, in the numerical experiments, we will show that the scaling terms can be removed without affecting the performance of the method for causality discovery.

\subsubsection{Unification -- KPCR method for causal discovery}
\label{sec:KPCR}
Based on Propositions \ref{prop:1}--\ref{prop:2}, we propose a unified KPCR method for nonlinear causality discovery, which combines the strengths  of KGC and lsNGC. The new KPCR model is defined as follows.

\begin{enumerate}
\item Set the kernel $K(\cdot,\cdot)$ to the squared-exponential kernel \eqref{rbfkernel} (or another stationary kernel). For datasets with $n > 1000$, implement the Nystr\"om approximation as described in Section~\ref{sec:unified_propositions}.
\item  
Construct a restricted and an unrestricted regression model,
\[
\hat{\mathbf{y}} = \mathbf{f}(X)\hat{\boldsymbol{\beta}} = P\mathbf{y}, 
\qquad
\hat{\mathbf{y}}' = \mathbf{f}'(Z)\hat{\boldsymbol{\beta}}' = P'\mathbf{y},
\]
where $X$, $Z$, and $\mathbf{y}$ are defined in \eqref{eq:XZ}.  Following Section \ref{sec:kgc}, the projection matrix is defined as $P = \sum_{i=1}^{s} \mathbf{a}_i\mathbf{a}_i^\top$, where $\mathbf{a}_i, i =1, \dots, s$ are the $s$ leading eigenvectors of the restricted kernel $K(X,X)$, and $P'$ is similarly defined by the $s'$ leading eigenvectors of $K(Z, Z)$. 

\item   
Set the degrees of freedom of the restricted and unrestricted models as
\[
p_r = \text{rank}(P), \qquad p_u = \text{rank}(P'),
\]
corresponding to the number of eigenvectors with non-zero eigenvalues.  To avoid violating the condition that the degrees of freedom of the unrestricted model must exceed those of the restricted model, $p_u > p_r$, the eigenvectors defining the projectors $P,P'$ are further filtered according to the constraint that the associated eigenvalue for $\mathbf{a}_i$ is greater than the maximum eigenvalue across $i = 1, \dots, s$ multiplied by a small constant $\mu$, $\lambda_i > \mu\lambda_{max}$, as described in the footnote of section \ref{sec:kgc}.
\item   
An f-test with test-statistic \eqref{eq:fstat} and Bonferroni correction ($\alpha' = \frac{0.05}{n_t(n_t - 1)}$) is then implemented with $SSR_r = \mathbf{y} - \hat{\mathbf{y}}$, $SSR_u = \mathbf{y} - \hat{\mathbf{y}'}$ to test if the time series $a$ Granger-causes time series $b$.
\end{enumerate}
This new KPCR GC discovery method improves upon KGC and lsNGC in two  ways. First, the testing procedure described above provides a better estimate of causal edges, replacing the correlation-based index used in KGC with an explicit joint hypothesis test. At the same time, it retains the nonlinear modeling capabilities of kernel-based methods. 
Second, thanks to the Nystr{\"o}m approximation, KPCR scales to larger datasets, similarly to lsNGC. However, KPCR improves on lsNGC by using a more general structure (as opposed to the additive model in \eqref{eq:lsngc_modelu}), which preserves the nesting between the restricted and unrestricted models, therefore enabling an F-test to assess causality. In the numerical experiments, we will demonstrate that KPCR is statistically and practically superior to both KGC and lsNGC.

Finally, observe that, for the squared-exponential kernel, we set the lengthscale according to the heuristic $\ell = C\,n_{t}m$, where $n_{t}m$ denotes the number of lagged features, and $C$ is a constant chosen to be sufficiently large so that the degrees of freedom of the unrestricted model exceed those of the restricted model. For the Nystr{\"o}m approximation, the number of inducing points is fixed in advance, and the locations of the inducing points are  selected based on a random subset of the data.

\subsection{PCMCI}

The PCMCI method \citep{runge_detecting_2019} is another state of the art constraint-based approach for time series causal discovery. This algorithm is comprised of two stages, the first of which finds adjacencies within the causal graph using the PC causal-discovery algorithm \citep{spirtes_algorithm_1991}, adapted for time series. Beginning with a fully-connected causal graph, conditional independence tests are iteratively applied to each pair of variables $(X_{a(t-\tau)}, X_{bt})$, up to the maximum selected lag order, i.e. for $\tau =1, \dots,m$. If it is determined $X_{a(t-\tau)} \nindep X_{bt} |S$, where $S$ is some set of conditioning variables, there exists an edge between $X_{a(t-\tau)}$ and $X_{bt}$ in the time series causal graph. This stage yields an estimated parent set of the response variable, $\hat{\mathcal{P}}(X_{bt})$, which may include some false positives. The second stage is the momentary conditional independence (MCI) test, which further prunes edges by including the parent set of the driver variable $\hat{\mathcal{P}}(X_{a(t-\tau)})$ in a conditional independence test of $X_{a(t-\tau)} \nindep X_{bt} | \hat{\mathcal{P}}(X_{bt}) \backslash \{X_{a(t-\tau)}\} \cup \hat{\mathcal{P}}(X_{a(t-\tau)})$. This stage accounts for autocorrelation within time series and controls for false positives. 

Careful selection of the conditional independence test determines which types of causal relationships can be identified within the system. A Gaussian process-based conditional independence test, the Gaussian process distance correlation coefficient (GPDC) \citep{runge_detecting_2019}, is one proposed nonlinear conditional independence test. The GPDC tests $X_{a(t-\tau)} \ci X_{bt} |S$ by assessing independence between the residuals of two Gaussian process regression models
\begin{align}
    &X_{a(t-\tau)} = f(S) + \epsilon_a, ~~\epsilon_a \sim N(0, \sigma^2),\\
    &X_{bt} = g(S) + \epsilon_b, ~~\epsilon_b \sim N(0, \sigma^2).
\end{align}
The independence of residuals $\epsilon_a, \epsilon_b$ implies conditional independence of $X_{a(t-\tau)} , X_{bt}$ given $S$. Independence between residuals is tested for using the distance correlation coefficient, defined in \cite{szekely_measuring_2007} as
\begin{align}
    dCor(X, Y) = \frac{dCov(X, Y)}{\sqrt{dVar(X)dVar(Y)}},
\end{align}
where $dCov^2(X, Y) = \frac{1}{n^2}\sum_{k,l=1}^nA_{k,l}B_{k,l}$, and $A_{k,l}, B_{k,l}$ are distance matrices corresponding to $X, Y$. The distance correlation coefficient has the key property that $dCor(X,Y) = 0$ \emph{iff} the variables $X$ and $Y$ are independent, unlike Pearson's correlation coefficient. This enables the identification of potentially nonlinear dependencies between residuals. 

 The p-values obtained from the conditional independence tests are compared to $\alpha_{PC}$ in the $PC$ phase of the algorithm in order to determine which variables are removed from the initial parent sets. The $\alpha_{PC}$ parameter may be optimised using a model selection approach. 
The edges that appear in the final causal graph after the MCI phase have corresponding p-value below a selected significance level, which we set to $\alpha=0.05$. The PCMCI method performs $n_t^2m$ tests, which can be controlled using a multiple testing correction.

\section{Score-based nonlinear granger causality methods}
\label{sec:ourgp}
The previous methods rely on a large number of statistical tests, resulting in relatively high computational complexity. In this section, we discuss a score-based approach that evaluates how well a given causal graph fits the observed data. Specifically, we introduce the $GP_{SIC}$ method, in which we use the marginal likelihood of a GP model with a kernel that implements Automatic Relevance Determination (ARD) as the underlying scoring metric, combined with the Smooth Information Criterion (SIC) to enhance model selection. We then implement a search strategy to identify the graph that maximises this score.  

GPs have previously been proposed for GC analysis in \cite{amblard_gaussian_2012, zaremba_statistical_2022}. The work \cite{amblard_gaussian_2012} introduces a GC framework in which the maximised log-marginal likelihoods of a restricted and unrestricted GP regression model are compared, where a GC relationship is identified if the marginal likelihood of the unrestricted model exceeds that of the restricted model. The work \cite{zaremba_statistical_2022} also proposes two GP models equipped with ARD kernels, using a Generalised Likelihood Ratio Test (GLRT) to compare the marginal-likelihoods of these models.

The $GP_{SIC}$ approach treats the maximised marginal-likelihood of the unrestricted model as a score by which we select the best causal graph, removing the need for a two model approach. A similar framework to the $GP_{SIC}$ is proposed in \cite{cui_gaussian_2022}, which also maximises the marginal likelihood of a GP model with ARD kernel for a score-based approach to time series causal graph discovery (but without any information criterion). 
We will show that the model performs comparatively to other state of the art nonlinear Granger Causality methods in Section \ref{sec:results}. 

\subsection{$GP_{SIC}$ model} 
First, each time series $\boldsymbol{\xi}_c$ in the original time series system is normalised via $\frac{{\xi}_{ci} - \bar{\xi}_c}{\sigma_c}$ for $i = \ 1, \dots, n + m$, where $\bar{\xi}_c, \sigma_c$ are the mean and standard deviation, respectively, of $\boldsymbol{\xi}_c$. Each time series $c \in \mathcal{T}$ then has mean 0 and standard deviation 1. We define $X, Z, \mathbf{y}$ as in (\ref{eq:XZ}).\footnote{Lag selection for the $GP_{SIC}$ method can be performed via a Leave-One-Out Cross Validation approach. This can be efficiently evaluated by fitting the GP once for each lag $i=1,\dots,m$, and computing the log pseudo-likelihood \citep[Sec. 5.4]{rasmussen2006gaussian}, such that the lag order with the maximum log pseudo-likelihood is considered optimal. }
We then solve the aforementioned regression problem
\begin{equation}
    \mathbf{y} = \mathbf{f}(Z) + \boldsymbol{\epsilon}, ~~\boldsymbol{\epsilon} \sim N (0, \sigma^2I_n),
\end{equation}
where the $n \times 1$ dimensional vector of observations $\mathbf{y}$ is modeled by some nonlinear function of the covariates $\mathbf{f}(Z)$ with the addition of independent Gaussian noise $\boldsymbol{\epsilon}$.
In order to identify nonlinear relationships between the predictor and response variables, we propose a GP regression model for $i = 1, \dots, n$,
\begin{align}
    &y_{i+m} = f(Z_i) + \epsilon_{i+m}, ~~\epsilon_{i+m} \sim N(0, \sigma^2), \\
    &f \sim GP(0, k(Z_i, Z_{j}; \boldsymbol{\theta})),\\
    &k(Z_i, Z_{j}; \boldsymbol{\theta}) = \tau^2 e^{-\frac{1}{2}(Z_i - Z_{j}) (diag(\boldsymbol{l})^{-2}) (Z_i - Z_{j})^\top}, 
\end{align}
such that $y_{i + m}$ is the value of the observation at time instant $i + m$, $f$ is  GP distributed with mean $0$ and kernel function $k(Z_i, Z_{j}; \boldsymbol{\theta})$,  $\epsilon_{i+m}$ is the additive Gaussian noise at time instant $i+m$, $\boldsymbol{l} = [l_1, \dots l_{n_tm}]$ are lengthscales and $\tau^2$ is the kernel variance. As the time series have been standardised, we fix the kernel variance $\tau^2 \approx 1$, such that the optimisation is occurring only in the lengthscales. The hyperparameters of the model can be summarised by $\boldsymbol{\theta} = [\boldsymbol{l}, \sigma^2]$. 
Note that there is a lengthscale associated with each of the $c_i = 1, \dots, n_tm$ features of the design matrix. This implements ARD as irrelevant features are effectively removed from the regression model, if  the corresponding lengthscale is sufficiently large.

\paragraph{Hyperparameters estimation and Smooth Information Criterion} 
 The ``smooth $L_0$ norm'' \citep{oneill_variable_2023},  defined as 
\begin{align}
    \left\|\frac{1}{\boldsymbol{l}}\right\|_{0,\omega} = \sum_{c_i=1}^{n_tm}\phi_\omega \left(\frac{1}{l_{c_i}}\right) ~~\text{ with }~~   \phi_\omega\left(\frac{1}{l_{c_i}}\right) = \frac{\frac{1}{l_{c_i}}^2}{\frac{1}{l_{c_i}}^2 + \omega^2},
\end{align}
is a differentiable\footnote{The ``smooth $L_0$ norm'' is differentiable for $\omega > 0$.} approximation to the $L_0$ norm, where $l_{c_i}$ for $c_i = 1,\dots, n_tm$ is the lengthscale associated with each of the $n_tm$ covariates. As $\omega$ approaches $0$, $\phi_\omega$ converges to the $L_0$ norm,$\lim\limits_{\omega\rightarrow0}\phi_\omega(x) = \|x\|_0$, such that the SIC encourages sparsity and is suitable for gradient-based optimisation. 
We perform hyperparameter estimation in the GP by maximising the marginal likelihood penalised by the Smooth Information Criterion (SIC), resulting in
\begin{align}
\label{eq:mlsic}
    \log p(\mathbf{y}|Z,\boldsymbol{\theta}) - \frac{\log(n)}{2} \left\|\frac{1}{\boldsymbol{l}}\right\|_{0,\omega}.
\end{align}
The inverse lengthscale is penalised as large lengthscales promote sparsity and remove irrelvant features.\footnote{We follow the procedure described in  \cite{oneill_variable_2023} of optimising with telescoping $\omega$, where we begin with a large value of $\omega$ ($\omega = 100$) and optimise over a series of 50 decreasing $\omega$, approaching zero ($\omega = 10^{-5}$). This allows us to smoothly converge to the $L_0$ norm. Thus, the optimised model hyperparameters are 
$        \boldsymbol{\hat{\theta}} = \arg\max_{\boldsymbol{\theta}}\left(\log p(\mathbf{y}|Z,\boldsymbol{\theta}) - \frac{\log(n)}{2} \left\|\frac{1}{\boldsymbol{l}}\right\|_{0,\omega}\right)_{\omega \rightharpoonup 10^{-5}}$, 
where $\omega \rightharpoonup 10^{-5}$ denotes the telescopic limit procedure. The optimised lengthscales may then be used to examine the presence of Granger causality.}
As the lengthscale corresponding to feature $c_i$ becomes very large, the contribution of that feature to the covariance shrinks toward zero. 
If lengthscale $l_{c_i} < 50$ \footnote{$50$ is an empirically determined threshold, at which we assume that a lengthscale $l \geq 50$ results in significantly reduced contribution of the corresponding covariate to the model. In our experimentation, we find that the size of the lengthscale for a relevant feature is often much less than $50$, while the size of legthscale for an irrelevant feature is much greater than $50$.}, then feature $c_i$,  corresponding to one of the lagged vectors of each time series $X_a$ for $a = 1, \dots, n_t$ in $Z$, is determined to be relevant to the regression model of time series $X_b$. 
We use this criterion to determine if a time series Granger causes another, and refer to Algorithm \ref{alg:gpgc} for how to identify all (Granger) causal edges for a time series system using the $GP_{SIC}$ model. 
\begin{algorithm}[!ht]
    \label{alg:gpgc}
    \caption{$GP_{SIC}$-based Granger Causality}
    \begin{algorithmic}[1]
        \Require Phase-space reconstruction of time series system: $Z = [X_{a(t-m)}, \dots, X_{a(t-1)}, X_{b(t-m)},$ $\dots,X_{b(t-1)}, X_{n_t(t-m)}, \dots, X_{n_t(t-1)}]^\top \in \mathbb{R}^{n \times n_tm}$, lag order $m$
        \For {$b \in \{1,\dots,n_t\}$}
            \State Maximise the criterion \eqref{eq:mlsic} of the $GP_{SIC}$ model, where $\mathbf{y} = [\xi_{b(m+1)},\dots,\xi_{b(m+n)}]$
            \For {$ c_i \in \{1,\dots, n_tm\}, ~\text{s.t.}~ \ell_{c_i} < 50$}
                \State The time series $X_a$ associated with feature $c_i$ Granger Causes $X_b$
            \EndFor
        \EndFor
    \end{algorithmic}
\end{algorithm}

\paragraph{Contemporaneous Causal Identification with GC}
Identifying causal interactions at an instantaneous time point presents a more challenging task as we cannot rely on the forward relationship in time to orient the causal link. The PCMCI+ method \citep{runge_discovering_2022} detects contemporaneous causal interactions by adding conditional independence testing between observations at contemporaneous time points, i.e. $X_{at} \ci X_{bt} | S$, to the PCMCI framework. Contemporaneous adjacencies are oriented using the identified separating sets $S$ and several rules based on Markov Equivalence Class relations and assumptions such as acyclicity. 

As noted in the Introduction, standard GC cannot estimate causal relationships between time series at an instantaneous time point. In this section, we prove that GC can nevertheless be used to identify contemporaneous causal interactions under certain restrictions on the causal graph.  
\begin{algorithm}[!ht]
    \label{alg:contemp}
    \caption{Contemporaneous Edge Identification \& Orientation}
    \begin{algorithmic}[1]
        \Require Phase-space reconstruction of time series system: $X = [X_{a(t-m)}, \dots, X_{a(t-1)}, X_{at}, X_{b(t-m)},$ $\dots,X_{b(t-1)}, X_{bt}, X_{n_t(t-m)}, \dots, X_{n_tt}]^\top \in \mathbb{R}^{n \times n_t(m+1)}$, lag order $m$
        \State \textbf{Adjacency Finding Phase}
        \State Estimate graph of contemporaneous and lagged edges $\hat{\mathcal{G}'}$ with $GP_{SIC}$ by regressing on each $b = 1, \dots, n_t$ using the covariates $X \backslash X_{bt}$, where $t$ is the current time step
        \State From $\hat{\mathcal{G}'}$, identify the set of lagged parents $\hat{\mathcal{P}}(X_{bt}) = \{X_{a(t-\tau)}\}$  where $X_{a(t-\tau)} \rightarrow X_{bt}$ and lag $
        m \geq \tau > 0$
        \State Estimate graph of lagged edges $\hat{\mathcal{G}''}$ with $GP_{SIC}$ using only the covariates $\hat{\mathcal{P}}(X_{bt})$  for each $b = 1, \dots, n_t$
        \State Define graph $\hat{\mathcal{G'''}}$ as all lagged edges $e_m := X_{a(t-\tau)} \rightarrow X_{bt} \in \hat{\mathcal{G}''}$ and contemporaneous edges $e_0 := X_{at} - X_{bt} \in \hat{\mathcal{G}'}$
        \State \textbf{Orientation Phase}
        \State Identify the set $U$ of unshielded lagged triples in $\hat{\mathcal{G'''}}$, i.e. $X_{a(t-\tau)} \rightarrow X_{ct} - X_{bt}$ and $X_{a(t-\tau)}\not\rightarrow X_{bt}$
        \State Filter U to remove any triples where $X_{a(t-\tau)} \rightarrow X_{ct} - X_{bt}$ and $X_{a(t-\tau)} \rightarrow X_{dt} - X_{bt}$
        \For{$u \in U$}
            \If {$X_{a(t-\tau)} \rightarrow X_{bt} \in \hat{\mathcal{G}'}$ and $X_{a(t-\tau)} \rightarrow X_{bt} \notin \hat{\mathcal{G}'''}$}
            \State $X_{ct}$ is a collider variable of $(X_{a(t-\tau)}, X_{bt})$. Orient the triple as $X_{a(t-\tau)} \rightarrow X_{ct} \leftarrow X_{bt}$. 
            \State Mark conflicting orientations. 
            \ElsIf {$X_{a(t-\tau)} \rightarrow X_{bt} \in \hat{\mathcal{G}'}$ and $X_{a(t-\tau)} \rightarrow X_{bt} \in \hat{\mathcal{G}'''}$}
                \State $X_{ct}$ is not a collider of $(X_{a(t-\tau)}, X_{bt})$. Orient the triple $X_{a(t-\tau)} \rightarrow X_{ct} \rightarrow X_{bt}$.
                \State Mark conflicting orientations.
            \EndIf 
        \EndFor
        \For{$e_0 \in \hat{\mathcal{G'''}}$ that remain undirected}
        \If {$e_0$ is in a contemporaneous triple $X_{at} \rightarrow X_{ct} - X_{bt}$}
            \State $X_{ct}$ is not a collider of $(X_{at}, X_{bt})$. Orient $X_{ct} \rightarrow X_{bt}$.
            \State Mark conflicting orientations. 
            \EndIf
        \EndFor
    \end{algorithmic}
\end{algorithm}

 In our proposed approach, described in Algorithm \ref{alg:contemp}, the lagged and contemporaneous causal adjacencies are learned via the score-based method  $GP_{SIC}$ described in the previous section, but in two phases. In phase 1 (the adjacency-finding phase), we estimate the causal graph $G'$, which is over-conditioned on all contemporaneous and lagged variables, and thus expected to contain false positives in the set of lagged adjacencies. This graph will also contain contemporaneous edges. 
 Then, from the lagged parents estimated in $G'$, another graph $G''$ is estimated, which excludes contemporaneous variables to block spurious lagged pathways. Orientations for lagged adjacencies are automatically oriented forward in time. 
 
 Phase 2 (orientation phase), following the structure of constraint-based methods like PCMCI+, uses rules based on conditional independence relations and Markov Equivalence classes to orient as many contemporaneous adjacencies as possible. Under the assumptions of Causal Sufficiency, Causal Faithfulness, and the Causal Markov condition \citep{pearl_2000,sprites_2000}, our proposed approach is able to identify causal adjacencies consistent within a restricted class of Structural Causal Models (SCMs). Indeed, in order to consistently identify the true causal adjacencies, we must make the assumption\footnote{This restriction on SCMs can be relaxed by including additional testing to filter out false positives induced by conditioning on the contemporaneous colliders.}  that there are no fully contemporaneous colliders in the underlying SCM, i.e. the structure $X_{at} \rightarrow X_{bt} \leftarrow X_{ct}$ is not present in the ground-truth causal graph. With this assumption, due to the use of maximal conditioning sets and score-based approach, our method requires less statistical tests than the PCMCI+ method and  is fully based on GC.
 
Under these assumptions, we state the following theorems:
 \begin{theorem}
    \label{theorem:adj}
    The causal adjacencies identified by the $GP_{SIC}$-based contemporaneous algorithm,  see Algorithm \ref{alg:contemp}, are asymptotically consistent with the true causal graph (underlying SCM). 
\end{theorem}
\begin{theorem}
    \label{theorem:cont}
    The orientations identified by the $GP_{SIC}$-based contemporaneous algorithm are correct. 
\end{theorem}

\section{Results}
\label{sec:results}
In Section \ref{sec:simulations}, we assess the performance of the aforementioned nonlinear GC discovery methods as well as the contemporaneous adjacency-finding algorithms on simulated datasets. We further assess the performance of the contemporaneous algorithms on a real-world application in Section \ref{sec:realdata}.

\subsection{Numerical Simulations}
\label{sec:simulations}

We compare the performance of the constraint-based Granger causality methods, KGC \citep{marinazzo_kernel_2008}, lsNGC \citep{wismuller_large-scale_2021}, to our proposed kernel-based methods: (i) the KPCR method discussed in Section \ref{sec:kpcr}; (ii) the $GP_{SIC}$-based method. Additionally, we include a comparison to the constraint-based causal discovery method PCMCI \citep{runge_detecting_2019}. We also include a comparison between $GP_{SIC}$ and the other GP-based methods \citep{cui_gaussian_2022,amblard_gaussian_2012,zaremba_statistical_2022} for Granger causality. 
The implementations of these models are discussed in the supplementary material. The models are tested on several simulated benchmark time series systems (many of these benchmarks  were first considered in \cite{martinez-sanchez_decomposing_2024,wismuller_large-scale_2021}), representing various types of causal relationships as described in Table \ref{table:experiment_desc}. The simulated examples vary in the number of variables, $n_t$, such that we can assess the performance of the models for both small and large time-series systems. All methods were tested using the true lag order $m$ for the simulated experiments, following \cite{cui_gaussian_2022, marinazzo_kernel_2008}. We evaluate uncertainty across $100$ Monte Carlo (MC) simulations for each simulated time-series system, with differing noise and/or initial conditions. The true models for each benchmark system  are reported in detail in the supplementary material. 

\begin{table}[!h]
\caption{Description of simulated experiments.}
\label{table:experiment_desc}
\centering
{\footnotesize
  \begin{tabular}{|c|c|c|c|c|c|}
    \hline
    \rowcolor{gray!40}
    \bf{$n_t = 2$} & \bf{$n_t = 3$} & \bf{$n_t = 5$} & \bf{$n_t = 8$} & \bf{$n_t = 20$} & \bf{$n_t = 30$} \\
    \hline
    \rowcolor{gray!20}
    \textbf{m=1} & \textbf{m=1} & \textbf{m=3} & \textbf{m=1} & \textbf{m=1} &  \textbf{m=1}\\
    1-way Logistic & Confounder & 5 Linear & 8 Nonlinear & 20 Nonlinear & 30 Nonlinear \\
    2-way Logistic & Mediator & 5 Nonlinear & & &\\
     & Synergistic Collider & \cellcolor{gray!20}\textbf{m=5} & & & \\
     & Redundant Collider & Moran Effect & & & \\
     & 3 Fan-in  & & & & \\
     & 3 Fan-out & & & & \\
     & Sync. (1-way int.) & & & & \\
     & Sync. (1-way strong) & & & & \\
     & Sync. (2-way strong) & & & & \\
     & \cellcolor{gray!20}\textbf{m=2} & & & & \\
     & Stochastic Linear & & & & \\
     & Stochastic Nonlinear & & & & \\
    \hline
    \end{tabular}
    }
\end{table}

The F1 score is used to compare the performance of the models. The F1 score is the harmonic mean of the Precision and Recall, defined as $\frac{2 \text{Precision}\text{Recall}}{\text{Precision} + \text{Recall}}$, where $\text{Precision} = \frac{\text{TP}}{\text{TP} + \text{FP}}$ and $\text{Recall} = \frac{\text{TP}}{\text{TP} + \text{FN}}$ \citep{murphy2012machine}. The F1 score is thus on a scale of $0.00$ to $1.00$, where a score of $1.00$ indicates all positive and negative causal relationships were correctly identified. Self-causal links were set to $0$ and thus not considered for correct identification for all methods, to be consistent with \cite{marinazzo_kernel_2008}.

Figure \ref{fig:gp_comparison} compares the GP-based GC methods -- GP$_{SIC}$, $GP$ \citep{cui_gaussian_2022}, $GP \Delta_\ell$ \citep{amblard_gaussian_2012}, $GP~GLRT $ \citep{zaremba_statistical_2022}. These methods were compared across a representative selection of the simulated experiments, where the Confounder, Mediator, and Synergistic Collider experiments are $n_t=3, m=1$ variable systems, representing the common structural relationships. The $5$ Nonlinear system has the parameters $n_t=5, m=3$, and the 20 Nonlinear system has $n_t=20, m=1$. The GP$_{SIC}$ method has a consistently high median F1 score across the experiments, and demonstrates the best performance in the experiment with the largest number of variables $n_t = 20$. This is consistent with what is shown in \cite{oneill_variable_2023,su2018sparse}, where SIC outperforms other information-based criteria and prior-based regularisation. For this reason, in the following comparisons, we report only GP$_{SIC}$, as it is the best method among the GP-GC approaches.

Figure \ref{fig:boxplots_lagged} compares the F1 score for the five nonlinear causal discovery methods, $GP_{SIC}$, lsNGC, KGC, KPCR, and PCMCI, across 19 simulated experiments. In Figure \ref{fig:boxplots_lagged}, it is seen that the $GP_{SIC}$ method performs comparatively or better than the other methods across the simulations. The $GP_{SIC}$ method consistently has high median F1, often with a small spread. Additional results tables comparing the average F1 score and standard deviation across the MC runs are also included in the supplementary material, in which the $GP_{SIC}$ method is shown to have the highest average F1 score in several of the simulated experiments, especially when there is less data available ($n = 250$).

From Figure \ref{fig:boxplots_lagged}, the KPCR method performs comparatively to the lsNGC and KGC methods in most simulations, and improves the F1 in several simulations (e.g. increased median in 3 Fan-out, Synchronous 2-way Strong, Moran Effect). The KGC method appears to perform slightly better for time series systems with a large number of variables ($n_t = \{20, 30\})$ when more data is available ($n = 500$).\footnote{Some of the simulated experiments are deterministic (3 Fan-in and 3 Fan-out), or nearly deterministic (Synchronized systems). In these examples, we observe that the lsNGC, KGC, and to some extent, KPCR methods have decreased performance for increasing $n$ due to the use of fixed parameter values (lengthscale in KGC, KPCR and $c_f$, $c_g$, lengthscale in lsNGC). Using fixed parameters or heuristics to determine parameter values can lead to numerical instabilities in the kernel when variance of noise is zero, resulting in increased identification of spurious edges.  The KPCR method reduces this effect in some examples as a result of employing the more conservative f-test and Bonferroni correction, while the $GP_{SIC}$ method is able to avoid this by optimizing the values of the lengthscales. For large values of $n_t$, KGC performs better than KPCR. We verified that this is due to the heuristic currently used to select the lengthscale in KPCR, see Section \ref{sec:KPCR}, and we leave a more principled investigation and improvement of this heuristic to future work.}

These claims are supported by a statistical analysis. We compared the five methods using the nonparametric Bayesian signed-rank test \citep{benavoli2017time,benavoli2014bayesian} applied to the median F1 score computed in each of the 19 simulated experiments. For each pair of causal discovery methods $(C_1, C_2)$ (for example, $C_1 = GP_{SIC}$ and $C_2 = \text{PCMCI}$), this test returns the posterior probability 
that $C_1 > C_2$, $C_1 \equiv C_2$ and $C_1 < C_2$ denoted respectively as $[p(C_1),\; p(rope),\; p(C_2)]$,
which can be visualised in the   probability simplex (Figure \ref{fig:baycomp500} for $n=500$, while Figure \ref{fig:baycomp} for $n=250$ is in the supplementary material). We declare two methods practically equivalent, $(C_1 \equiv C_2$, when the difference in F1 is less than 0.01 (1\%). The interval $[-0.01, 0.01]$ therefore defines a region of practical equivalence (rope).
From Figure \ref{fig:baycomp500}, we observe that $GP_{SIC}$ is practically superior to PCMCI with probability one. It also outperforms lsNGC with probability $0.99$, and it is better than KGC  with probability $0.85$. The statistical comparison between KPCR and KGC/lsNGC (Figure \ref{fig:baycomp}, second row)  shows that KPCR is generally better than either KGC or lsNGC. Its performance is comparable to KGC, but unlike KGC, KPCR can be efficiently scaled to large datasets.

\begin{figure}
    \centering
    \includegraphics[width=0.7\linewidth]{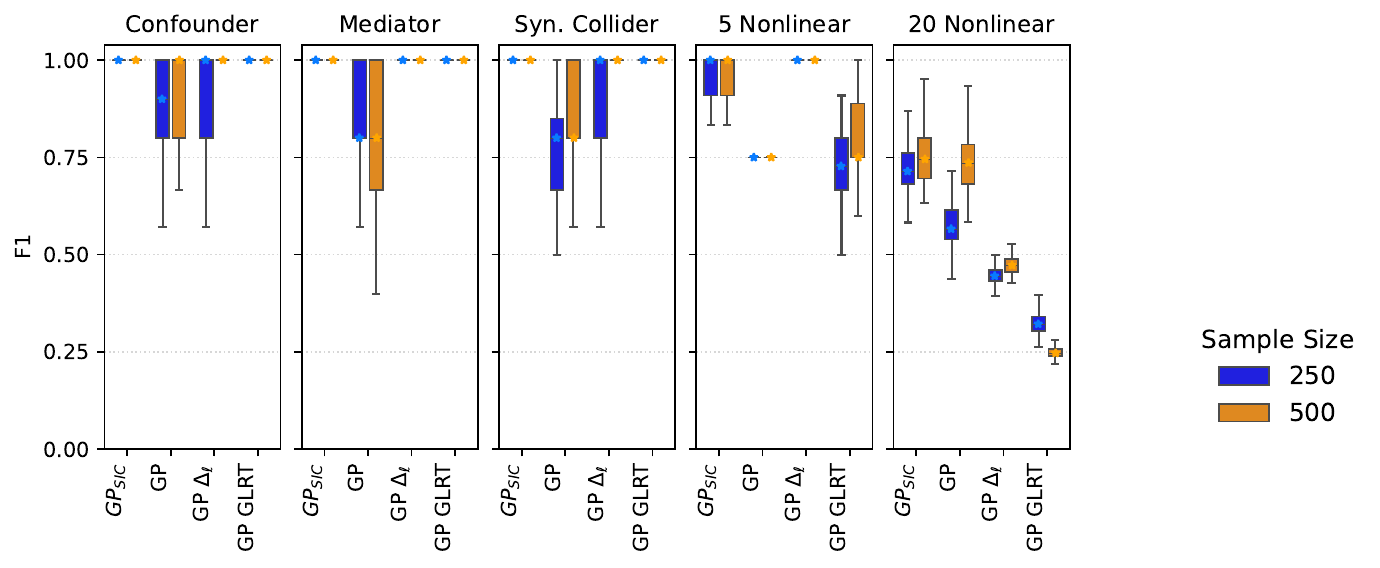}
    \caption{Comparison of four GP-based Granger Causality Methods -- the \emph{$GP_{SIC}$} method, single model \emph{GP} method without SIC \citep{cui_gaussian_2022}, two-model likelihood difference method \emph{$GP \Delta_\ell$} \citep{amblard_gaussian_2012}, and the GP GLRT method \citep{zaremba_statistical_2022}. The methods were compared across a selection of the simulated experiments.}
    \label{fig:gp_comparison}
\end{figure}

\begin{figure}
    \centering
    \includegraphics[width=0.95\linewidth]{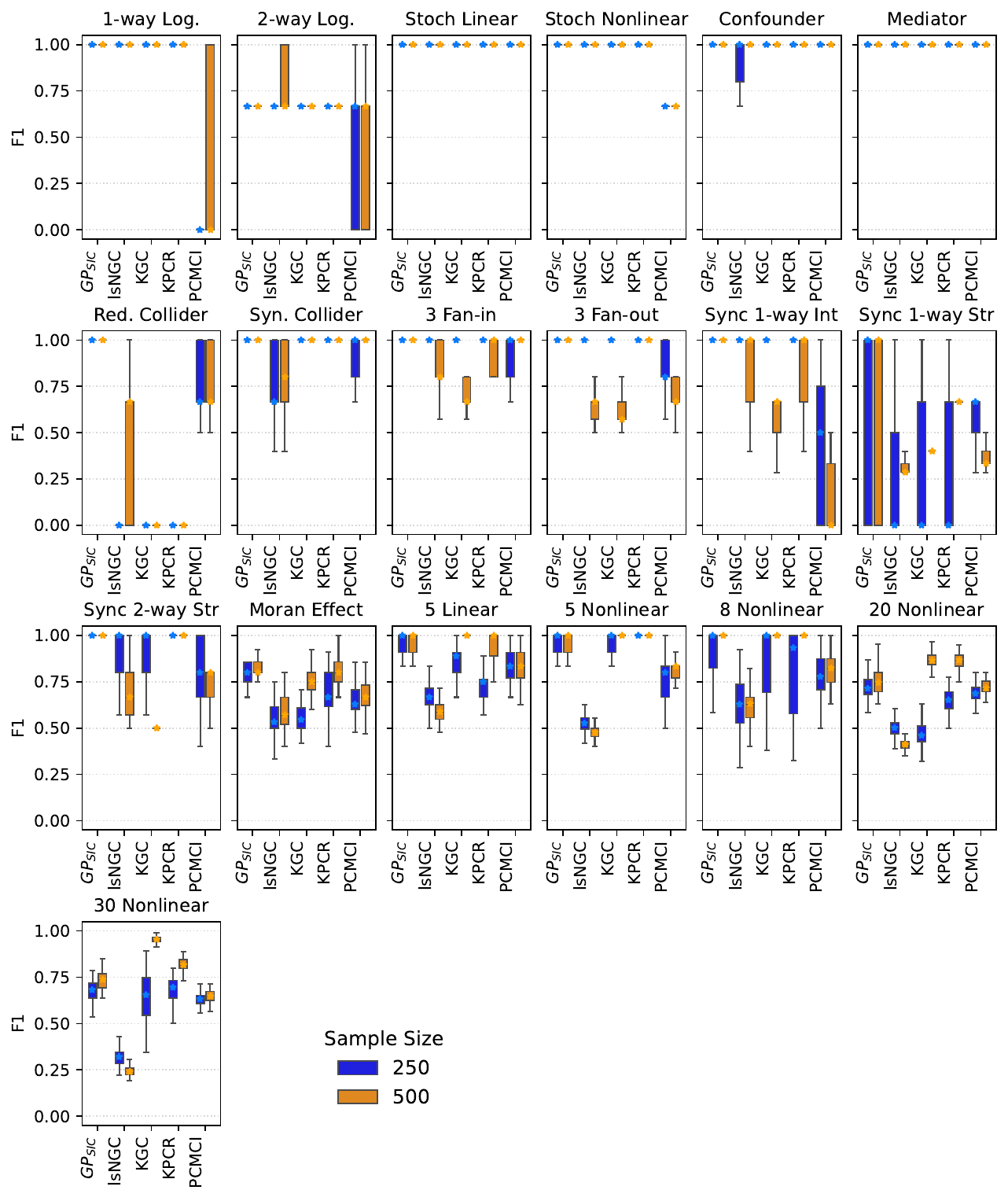}
    \caption{F1 score for five nonlinear causal discovery methods: $GP_{SIC}$, lsNGC, KGC, KPCR method, and PCMCI. 100 MC experiments were run for each simulated dataset for $n=250, 500$. The median is marked with a $\star$.}
    \label{fig:boxplots_lagged}
\end{figure}

We additionally compare the performance of the contemporaneous $GP_{SIC}$ causal algorithm in \ref{alg:contemp} to the PCMCI+ algorithm \citep{runge_discovering_2022}. Figure \ref{fig:contemp} shows results for two simulated examples, an $n_t=5, m=3$ linear system with Gaussian noise, and an $n_t =5, m=3$ nonlinear Gaussian system. Results for additional simulated experiments are included in the supplementary material. We calculate the F1 score for lagged adjacencies among only lagged connections. The F1 score for contemporaneous adjacencies is evaluated among only contemporaneous connections, and assesses whether a correct adjacency was found without considering orientation. The F1 for contemporaneous orientation then considers how many contemporaneous edges were correctly oriented, where Precision is defined as the ratio of correctly oriented edges to all estimated contemporaneous adjacencies and Recall is the ratio of correctly oriented edges to all contemporaneous adjacencies in the ground-truth causal graph. We test time series systems for an increasing number of time series $n_t=\{5, 10, 20\}$, varying autocorrelation strengths $a = \{0.1, 0.4, 0.8\}$, differing sample sizes $n = \{250, 500\}$, and differing functional types (linear or nonlinear).

In the linear Gaussian setting, both methods have similar performance in identification of lagged adjacencies. The PCMCI$+$ method has higher F1 for contemporaneous adjacency identification and orientation across the experiments. We did not test high autocorrelation ($a = 0.8$) for the linear Gaussian setting as the systems were non-stationary. In the nonlinear Gaussian setting, the $GP_{SIC}$ method consistently has higher F1 for lagged adjacency identification. 
For contemporaneous adjacency and orientation identification, the PCMCI$+$ performs better across the experiments when autocorrelation strength $a$ is lower, while the $GP_{SIC}$ method performs better for higher autocorrelation. As a GC method, the $GP_{SIC}$ contemporaneous algorithm relies on a maximal conditioning set, and as such will have lower power in comparison to methods that use reduced conditioning sets. However, the $GP_{SIC}$-based algorithm is fully defined within the GC framework, as it performs only $2n_t$ tests, compared to PCMCI$+$, which performs $n_t^2(m+1)$ tests, as shown in the below table.

\begin{center}
{\footnotesize
  \begin{tabular}{|c|c|c|}
    \hline
    \rowcolor{gray!40} 
    & $GP_{SIC}$ & PCMCI$+$ \\
    \hline 
    $n_t=5,m=1$ & 10 & 50 \\
    \hline 
    $n_t=5,m=3$ & 10 & 100\\ 
    \hline 
    $n_t=10,m=2$ & 20 & 300\\ 
    \hline 
    $n_t=20,m=1$ & 40 & 800\\ 
    \hline 
    \end{tabular}}\end{center}

\begin{figure}
    \includegraphics[width=0.31\linewidth,trim={2cm 1cm 2cm 2cm},clip]{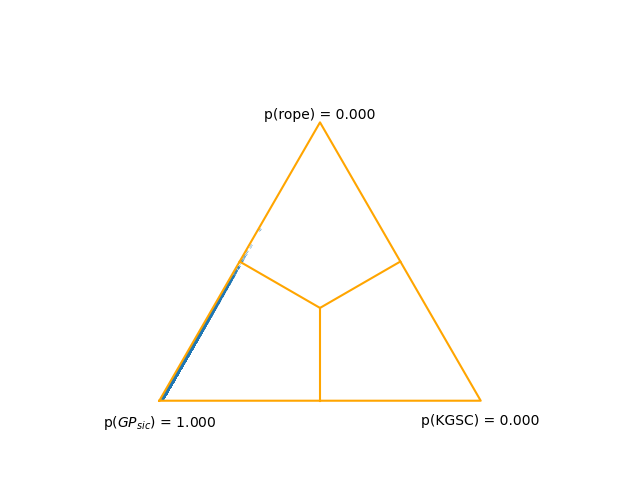}            \includegraphics[width=0.31\linewidth,trim={2cm 1cm 2cm 2cm},clip]{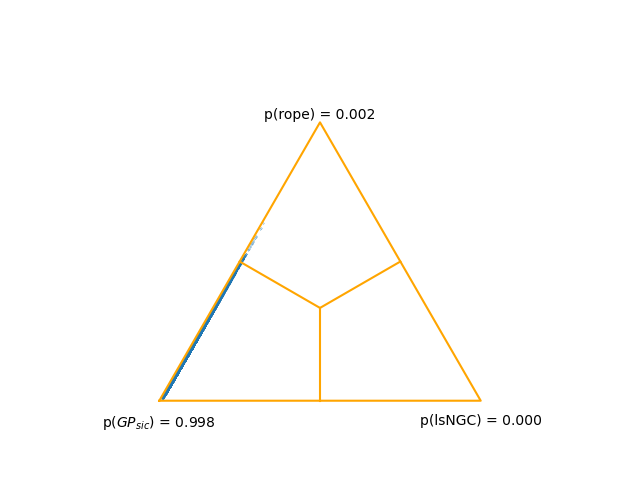}
            \includegraphics[width=0.31\linewidth,trim={2cm 1cm 2cm 2cm},clip]{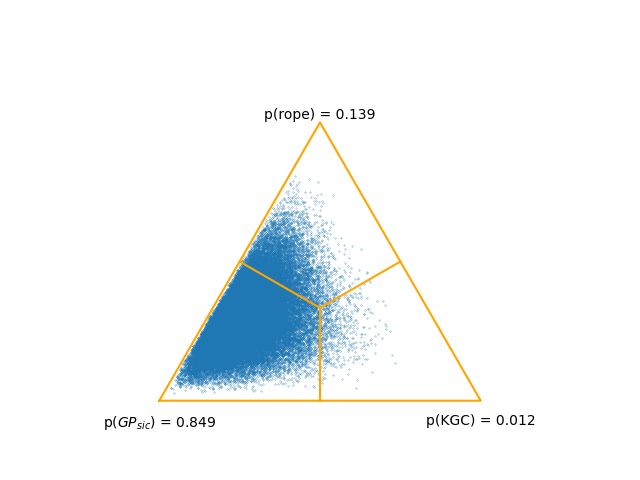}
            \\
        \includegraphics[width=0.31\linewidth,trim={2cm 1cm 2cm 2cm},clip]{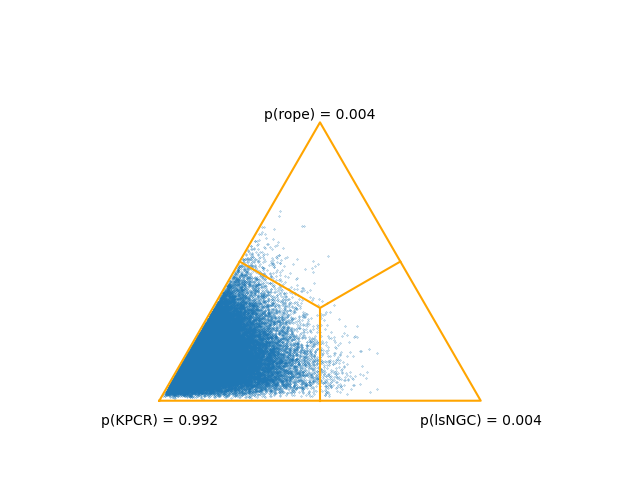}
            \includegraphics[width=0.31\linewidth,trim={2cm 1cm 2cm 2cm},clip]{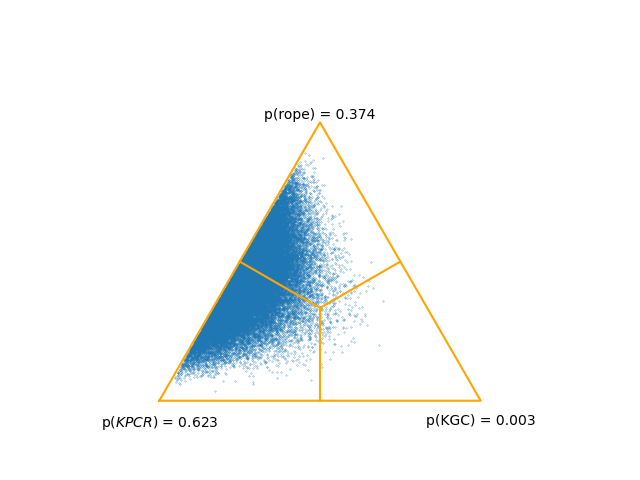}

    \caption{$n=500$ -- For each pair of causal discovery methods $(C_1, C_2)$ (e.g., $C_1 = GP_{SIC}$ and $C_2 = PCMCI$), we report the posterior probability vector
that $C_1 > C_2$, $C_1 \equiv C_2$ and $C_1 < C_2$ denoted respectively as $[p(C_1),\; p(rope),\; p(C_2)]$,
obtained via the Bayesian Wilcoxon signed-rank test. The cloud of points in the figure corresponds to samples from this posterior distribution: each point is a probability vector
$[p(C_1),\; p(rope),\; p(C_2)]$
which we plot in the probability simplex.}
    \label{fig:baycomp500}
\end{figure}

\begin{figure}
    \centering
    \includegraphics[width=0.48\linewidth]{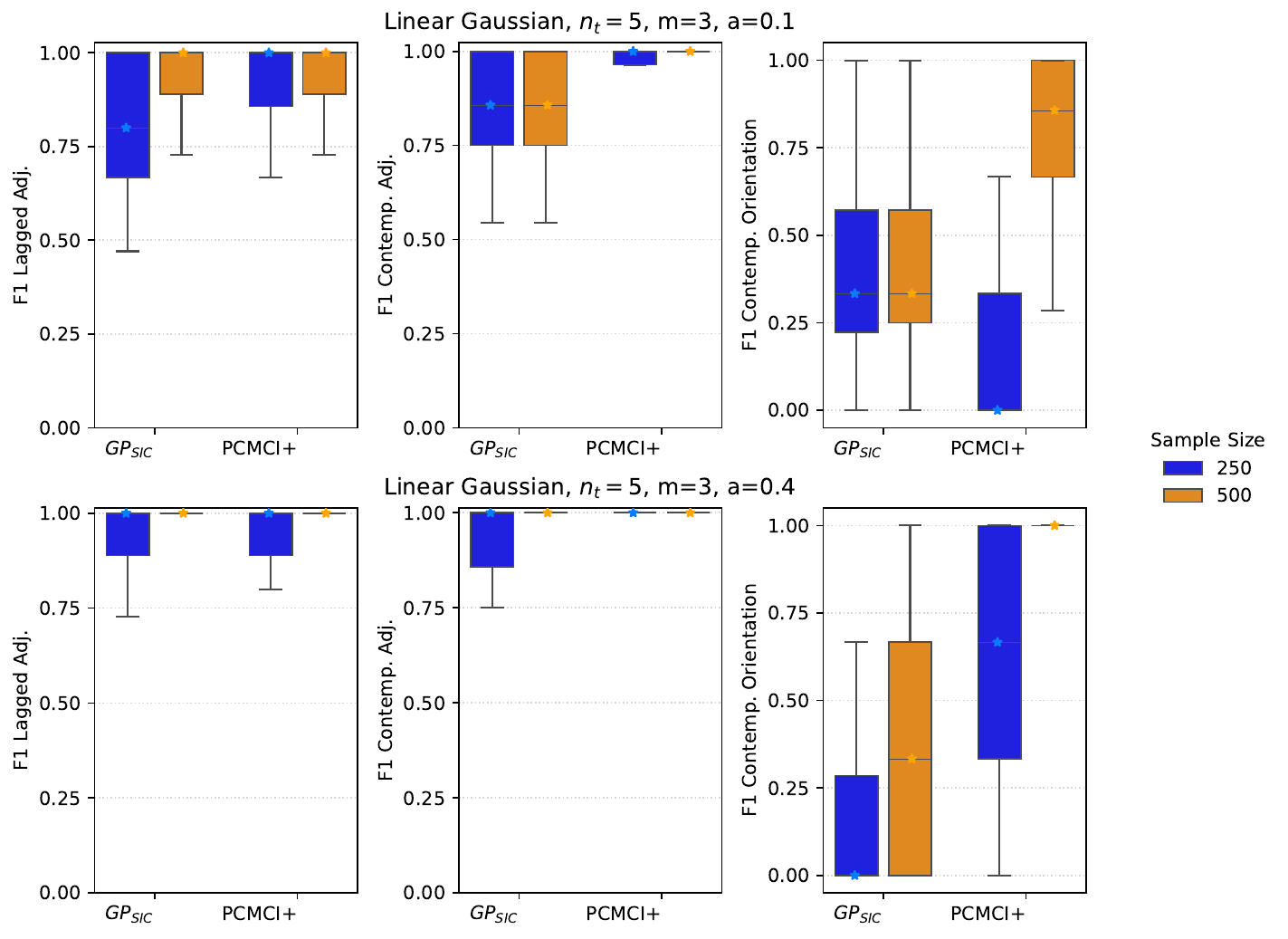}
    \includegraphics[width=0.48\linewidth]{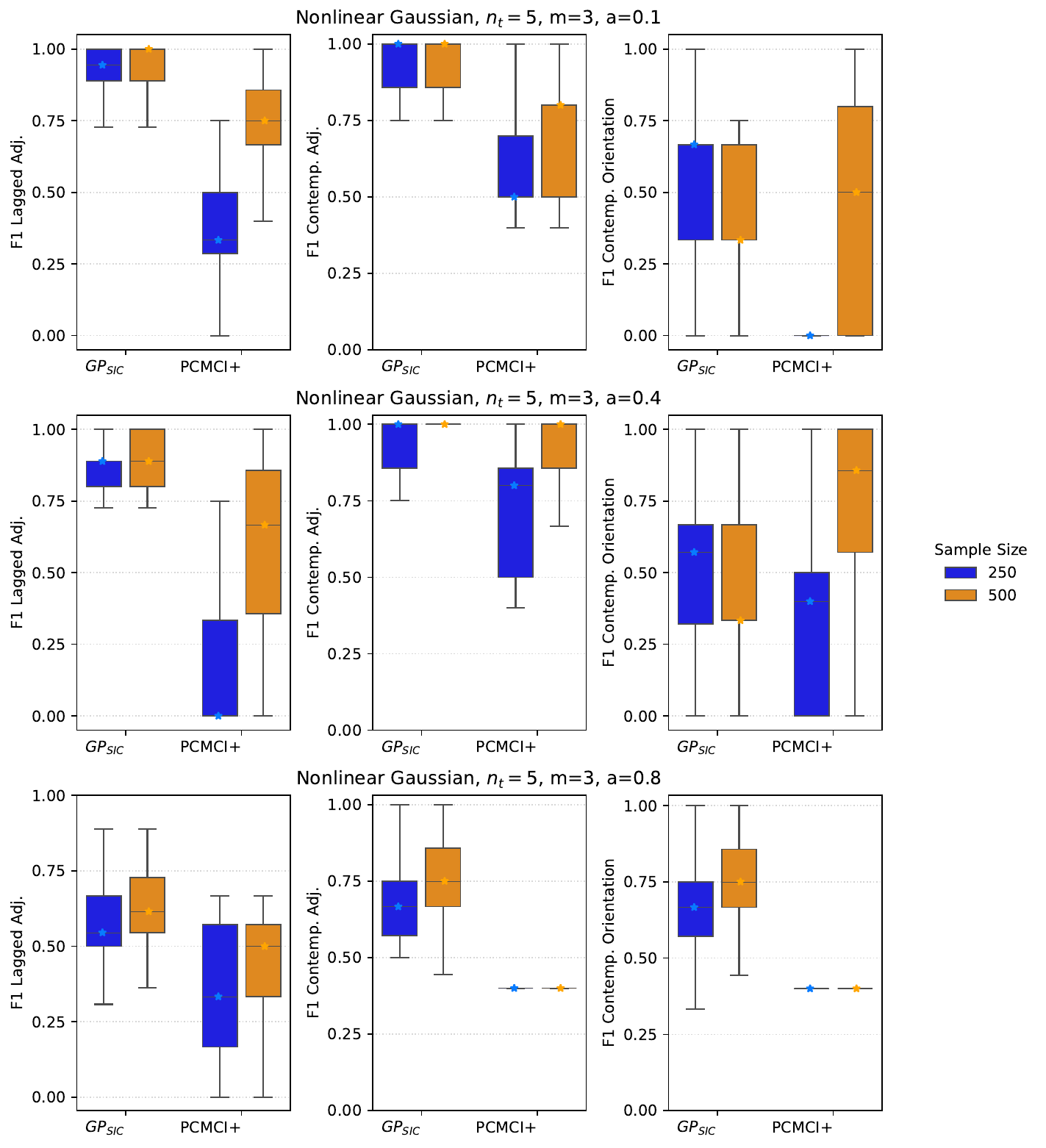}
    \caption{Box plots of the F1 score for the contemporaneous $GP_{SIC}$ algorithm and $PCMCI+$. The F1 score is evaluated separately for the lagged adjacencies, the contemporaneous adjacencies, and the contemporaneous orientations. The plots on the left are linear Gaussian systems, and the plots on the right are nonlinear Gaussian systems.}
    \label{fig:contemp}
\end{figure}

\subsection{Real-world Application}
\label{sec:realdata}

In control engineering, the process variables within a dynamical system are monitored and adjusted in order to drive the system to a desired state. Understanding causal dynamics within these systems can aid in the design of models that can effectively govern the changes between desired states. The work \cite{michalowski_modelling_2011} simulates a pH neutralisation plant, in which a Continuous Stirred Tank Reactor (CSTR) is used to maintain desired states (set points) for both level of fluid in the tank and pH. There is a complex interaction of variables between the level control loop and pH control loop, compounded by nonlinear causal relationships, feedback within the control loops, and instantaneous causal effects. 

Due to the presence of instantaneous causal effects, we attempt to recover the causal graph associated with the data from the simulated plant in \cite{michalowski_modelling_2011} using the $GP_{SIC}$-based contemporaneous algorithm and PCMCI$+$. We simulated data with the pump enabled across two different modes: \emph{servo} mode uses constant acid flow with fluctuating pH set point, whereas \emph{regulatory} mode maintains a constant pH set point while varying the acid flow disturbance. As such, we expect to observe differing causal dynamics between the two modes, where the variable that is held constant should be eliminated as a driver from the causal graph. We simulate 50 datasets with $n=1000$ data points for both modes using different seeds, and compare the performance between the two models in Table \ref{phplant_results}. We additionally show the rate of recovered causal connections across the simulations in the supplementary material.
\begin{table}[!h]
  \caption{\label{phplant_results} Comparison of $GP_{SIC}$-based contemporaneous algorithm and PCMCI$+$ on pH neutralisation plant data. The average across 50 simulated datasets is reported, and standard deviation is included in parentheses.}
\centering
{\footnotesize
  \begin{tabular}{|c|c|c|c|c|}
    \hline
    \rowcolor{gray!40}
     & \bf{Precision} & \bf{Recall/TPR} & \bf{FPR} & \bf{F1} 
    \\
    \hline
    \rowcolor{gray!20}
    \bf{Servo} & & & &\\
    \hline
    $GP_{SIC}$ & 0.54 (0.15) &0.46 (0.06) &0.07 (0.04) &0.49 (0.07)\\
    \hline
    PCMCI$+$ & 0.46 (0.14) & 0.32 (0.11) & 0.06 (0.02) & 0.37 (0.11)\\
    \hline
    \rowcolor{gray!20}
    \bf{Regulatory} & & & &\\
    \hline
    $GP_{SIC}$ &0.55 (0.15) & 0.32 (0.07)&0.05 (0.02) &0.40 (0.07)\\
    \hline
    PCMCI$+$ & 0.34 (0.11) & 0.21 (0.07) & 0.08 (0.03) & 0.26 (0.08)\\
    \hline
    \end{tabular}}
\end{table}

Both methods are able to distinguish the control loops (level and pH) across the simulations, and have similar False Positive Rate ($FPR = \frac{FP}{FP + TN}$). Additionally, both methods detect the absence of a causal relationship when the acid flow is held constant in servo mode, or the pH set point is held constant in regulatory mode. However, only the $GP_{SIC}$ detects the relationship between acid flow and pH in some of the regulatory simulations; the PCMCI$+$ method never detects a causal relationship.

For both methods, some causal edges are rarely or never detected, e.g. $\{Acid Flow, Base Flow, Outflow\} \rightarrow pH$, and $\{Acid Flow, Base flow, Outflow\} \rightarrow Level$. The supplementary material includes a plot of the normalised time-series data for these subsystems, to show that there is a lack of identifiability. Specifically, the causal effect is negligible compared to the strong autocorrelation of the integrated variable (level or pH).
 While the causal interactions regarding the mechanics of the controllers are consistently identified, the lack of identifiability of these subsystem dynamics explains the low F1 score. However, we note in comparison to the PCMCI$+$ method, the $GP_{SIC}$ method has higher average F1 for both data modes.

\section{Conclusions}

Kernel-based methods for Granger Causality enable the identification of complex, nonlinear causal relationships. These GC methods fall into two main camps of either constraint or score-based approaches. In this work, we have reviewed two prominent kernel constraint-based approaches to nonlinear GC, and have shown that they can be unified under the framework of Kernel Principal Component Regression. Unifying these methods can improve causal identification, as shown in the simulated experiments. 

Furthermore, we introduce a Gaussian Process score-based approach to nonlinear GC with the incorporation of an information criterion. This approach has comparable or improved performance in many of the simulated experiments, and we show how it can be extended to identify contemporaneous causal relationships. 

The $GP_{SIC}$ model assumes Gaussian noise, while other causal identification methods can account for non-Gaussian noise terms. Future work may include investigating non-Gaussian likelihoods in the $GP_{SIC}$ model in order to account for non-Gaussian noise, which would entail approximate inference methods. 

\section*{Acknowledgments}
This publication has emanated from research
conducted with the financial support of Taighde
Éireann – Research Ireland Centre for Research Training in Artificial
Intelligence under Grant No. 18/CRT/6223.

\section*{Declarations}
The authors have no conflicts of interest to disclose. The data and source code are made available, as detailed in the supplementary material.

\bibliographystyle{abbrv}
\bibliography{refs}

\newpage
\section*{Supplementary Material}

\subsection*{Proof of Proposition 1}

\begin{proof}
    Kernel PCR solves a regression problem where the covariates are given by the data projected onto the principal components \eqref{eq:PCApc}, such that the full covariate matrix can be defined as $\widecheck{K}(X, X)A\Lambda^{-1/2}$, where $\widecheck{K}(X,X) \in \mathbb{R}^{n \times n}$ is the full kernel matrix,  $A = [\mathbf{a}_1, \dots, \mathbf{a}_s] \in \mathbb{R}^{n \times s}$ is the matrix defined by the $s$ leading eigenvectors of the kernel, and $\Lambda = diag(\lambda_i) \in \mathbb{R}^{s \times s}, i = 1, \dots, s$ is a diagonal matrix of the $s$ leading eigenvalues. The function $\mathbf{f}(X)$ can thus be defined as $\mathbf{f}(X) = \widecheck{K}(X, X)A\Lambda^{-1/2}$, such that the regression model in \eqref{eq:KGC1} is equivalent to 
    $$
    \begin{aligned}
        \mathbf{\hat{y}} = \widecheck{K}(X, X)A\Lambda^{-1/2}\hat{\boldsymbol{\beta}}.
    \end{aligned}
    $$
    The estimated coefficients are given by least-squares:
    $$
    \begin{aligned}
        \hat{\boldsymbol{\beta}} = \Big((\widecheck{K}(X, X)A\Lambda^{-1/2})^\top \widecheck{K}(X, X)A\Lambda^{-1/2}\Big)^{-1}(\widecheck{K}(X, X)A\Lambda^{-1/2})^\top \mathbf{y},
    \end{aligned}
    $$
    such that the regression model can be written as a projection of the observations $\mathbf{y}$ as it is done in the KGC method, where the orthogonal projection matrix is now equivalent to 
    $$
    \begin{aligned}
        P = \widecheck{K}(X, X)A\Lambda^{-1/2} \Big((\widecheck{K}(X, X)A\Lambda^{-1/2})^\top \widecheck{K}(X, X)A\Lambda^{-1/2}\Big)^{-1}(\widecheck{K}(X, X)A\Lambda^{-1/2})^\top.
    \end{aligned}
    $$
    In the KGC method, the projection matrix is defined by the $n_v$ eigenvectors of $K$ with non-zero eigenvalue, i.e. $P = \sum_{i=1}^{n_v} \mathbf{a}_i\mathbf{a}_i^\top = AA^\top$, which will coincide with establishing $s = n_v$ in Kernel PCR. By using the eigenvector-eigenvalue decomposition of the kernel matrix $\widecheck{K}(X, X) = A\Lambda A^\top$,
    the projection matrix from Kernel PCR is seen to be equivalent to the projection matrix in the KGC method 
    $$
    \begin{aligned}
         P = A\Lambda A^\top A\Lambda^{-1/2} ((\Lambda^{-1/2})^\top A^\top (A \Lambda A^\top)^\top A\Lambda A^\top A\Lambda^{-1/2})^{-1} (\Lambda^{-1/2})^\top A^\top (A\Lambda A^\top)^\top = AA^\top,
    \end{aligned}
    $$
    i.e. the regression problem solved in KGC is equivalent to that solved in Kernel PCR. The same rationale follows for the unrestricted regression model, where \eqref{eq:KGC2} is equivalent to $\mathbf{\hat{y'}} = \widecheck{K}(Z, Z)A'\Lambda'^{-1/2}\hat{\boldsymbol{\beta}'}$.
\end{proof}

\subsection*{Proof of Proposition 2}

\begin{proof}
    The lsNGC method uses $c_f$ cluster centroids, identified via k-means clustering, within the normalized Radial Basis Functions in \eqref{lsngc_kernel_f}, reducing the dimensionality of the covariates $c_f < (n_t - 1)m$. This results in a covariate matrix with dimensionality $n \times c_f$, where each element of the matrix is defined by the nonlinear transformation $f_j(X_i)$ for $j=1,\dots,c_f$ and $i=1,\dots,n$.
    
    Assume that the number of inducing points used in the Nystr{\"o}m approximation $\widetilde{K}$ is equivalent to the number of cluster centroids used in the lsNGC method, $n_J = c_f$. Using the proposed kernel \eqref{eq:rbfscaled}, we can write the components of $\widetilde{K}$ as $K(X, X_J) = D_X^{-1}K_0(X, X_J)D_J^{-1}$ and $K(X_J, X_J) =  D_J^{-1}K_0(X_J, X_J)D_J^{-1}$, where $K_0({\bf x}_i, {\bf x}_j) = e^{-\|{\bf x}_i - {\bf x}_j\|^2/l^2}$ is the standard SE kernel, $D_X$ is an $n$-dimensional diagonal matrix of the normalisation terms $\sum_{k=1}^{n_J} e^{-\|{\bf x}_i-{\bf x}_k\|^2/l^2}$ and $D_J$ is an $n_J$-dimensional diagonal matrix of the normalisation terms $\sum_{k=1}^{n_J} e^{-\|{\bf x}_k-{\bf x}_j\|^2/l^2}$. As such, $\widetilde{K}$ can be written as
    $$
    \begin{aligned}
        &\widetilde{K} = D_X^{-1}K_0(X, X_J)K_0^{-1}(X_J, X_J)K_0(X_J, X)D_X^{-1},
    \end{aligned}    
    $$
    such that its feature vectors can be written as $\phi(X) =  D_X^{-1}K_0(X, X_J)K_0^{-1/2}(X_J, X_J)$. It can be seen that $D_X^{-1}K_0(X, X_J)D_J^{-1}$ is equivalent to $\mathbf{f}(X)$ in the lsNGC framework. Thus, the covariates $\phi(X)V$ of a KPCR with Nystr{\"o}m approximation approach, where $V$ is the matrix of $s = n_J$ leading eigenvectors of the covariance matrix $\phi(X)^\top\phi(X)$, is equivalent to the approach used in lsNGC, up to the factor $D_JK_0^{-1/2}(X_J, X_J)V$. However, since this factor is multiplied by  $\boldsymbol{\beta}$ it is absorbed into the parameters during optimisation:
    $$
\begin{aligned}
   \tilde{\boldsymbol{\beta}} &= (D_JK_0^{-1/2}(X_J, X_J)V)^{-1}(\mathbf{f}(X)^\top\mathbf{f}(X))^{-1}(D_JK_0^{-1/2}(X_J, X_J)V)^{-\top} (D_JK_0^{-1/2}(X_J, X_J)V)^{\top}\mathbf{f}(X)^\top\mathbf{y}\\
   &=(D_JK_0^{-1/2}(X_J, X_J)V)^{-1}(\mathbf{f}(X)^\top\mathbf{f}(X))^{-1}\mathbf{f}(X)^\top\mathbf{y},\\
\end{aligned}
$$
and so the relationship $\hat{\boldsymbol{\beta}}=(D_J K_0^{-1/2}(X_J, X_J)V)\tilde{\boldsymbol{\beta}}$ proves the equivalence between lsNGC and Kernel PCR.
\end{proof}

\subsection*{Proof of Contemporaneous Algorithm}
\paragraph*{Assumptions} We consider the phase-space reconstruction of a time series system, $X = [X_{a(t-m)},\dots,$ $X_{a(t-1)}, X_{at}, X_{b(t-m)}, \dots,X_{b(t-1)}, X_{bt}, \dots, X_{n_t(t-m)}, \dots, X_{n_tt}]^\top$, which adheres to the underlying Structural Causal Model (SCM)
$$
\begin{aligned}
    X_{bt} = f(\mathcal{P}(X_{bt})) + \epsilon_{bt}, \epsilon_{bt} \sim \mathcal{N}(0, \sigma^2).
\end{aligned}
$$
The assumed SCM is a linear or nonlinear function of the (lagged and contemporaneous) parents $\mathcal{P}(X_{bt})$ of the variable with additive Gaussian noise that is independent across time points and between time series.

The assumptions made regarding causal graph identifiability \citep{sprites_2000, pearl_2000, runge_discovering_2022} include
\begin{enumerate}
    \item \emph{Causal Sufficiency}: All variables are observed within the system $X$, i.e. there are no hidden confounders.
    \item \emph{Causal Faithfulness}: 
    \begin{enumerate}
        \item Adjacency Faithfulness: The conditional dependence relations implied by the edges present in the underlying causal graph hold in the joint probability distribution. 
        \item Standard Faithfulness: The conditional independence relations implied by the $d$-separation structures in the underlying causal graph hold in the joint probability distribution.
    \end{enumerate}
    \item \emph{Causal Markov Condition}: Each variable $X_{bt}$ is independent of all non-descendants given its parents $\mathcal{P}(X_{bt})$.
\end{enumerate}
Regarding the time series, stationarity of time-series is assumed, although non-stationary time-series can be evaluated with GP-based methods by modifying the kernel function \citep{cui_inference_2024}. It is also assumed that the correct lag order $m$ is selected for the model. 

Furthermore, the proposed method assumes the absence of contemporaneous collider triples, i.e. $X_{at} \rightarrow X_{ct} \leftarrow X_{bt}$. It will be shown that spurious lagged parents due to conditioning on colliders within a lagged triple, i.e. $X_{a(t - \tau)} \rightarrow X_{ct} \leftarrow X_{bt}$, are identifiable and may be filtered, but the proposed method precludes consistent graph identification when fully contemporaneous colliders exist. 

\begin{lemma}
    \label{lemma:conditional_independence}
    The Granger non-Causal relationship $X_a \not\rightarrow X_b$ identified by the $GP_{SIC}$ method implies the conditional independence relationship $X_a \ci X_b | X_{\cdot(t-m)}$, where $X_{\cdot(t-m)}$ includes all variables conditioned on in the $GP_{SIC}$ model for each $b = 1, \dots, n_t$, up to lag $m$.
\end{lemma}

\begin{proof}
    The connection between linear Granger Causality and the concept of transfer entropy can be understood based on statistical risk minimisation \citep{marinazzo_kernel_2008}. Consider the stochastic variables $X_b$ and $y$, where both variables are recorded across $t= 1, \dots, n$ time instances. The regression function which minimises the risk functional $R[f] = \int p(X_b, y)(y - f(X_b))^2dydX_b$ is the conditional expectation of the current observation given its past values, $f^*(X_b) = \int p(y|X_b)ydy$. The inclusion of another variable $X_a$, which may be multivariate and include observations at time instant $t$ that are contemporaneous to the observation $y$, yields the solution $g^*(X_b, X_a) = \int p(y|X_b,X_a)ydy$. The statement that $X_a$ does not Granger cause $X_b$ indicates that the inclusion of $X_a$ in the regression function does not improve the prediction of $y$, such that $p(y|X_b) = p(y|X_a,X_b)$ and thus $f^*(X_b) = g^*(X_b, X_a)$. Transfer entropy evaluates whether $p(y|X_b) = p(y|X_b,X_a)$ via
    the assessment of the conditional mutual information between $X_a$ and $X_b$,
    $$
    \begin{aligned}
        &I(y; X_a|X_b) = \int\int\int p(y, X_b, X_a)log(\frac{p(y|X_b,X_a)}{p(y|X_b)})dydX_bdX_a,\\
        &=\int\int\int p(y,X_b,X_a)log(\frac{p(y,X_a|X_b)}{p(y|X_b)p(X_a|X_b)})dydX_bdX_a,
    \end{aligned}
    $$
    which may be considered equivalent to the evaluation of the conditional independence relationship $y \ci X_a | X_b$, which tests $p(y,X_a|X_b) = p(y|X_b)p(X_a|X_b)$ and thus whether $TE_{X_a\rightarrow X_b} = 0$ \citep{runge_causal_2018}. As Granger non-causality implies conditional independence, Granger causality implies conditional dependence.   
    
    As linear Granger Causality is equivalent to Kernel Granger Causality with a linear kernel \citep{marinazzo_kernel_2008}, Kernel Granger Causality with a suitable choice of nonlinear kernel generalises GC to the nonlinear setting \citep{marinazzo_kernel_2008}. The regression equations $f^*(X_b)$ and $g^*(X_b,X_a)$ remain the solutions to the minimisation of the risk functional, and so a nonlinear GC relationship identified via KGC can also be interpreted in terms of conditional independence relations. When using kernel (ridge) regression, the solution to the risk-minimising function at one realisation, $f^*(X_{bt^*})$, with the addition of noise has the analytical form $y^*(X_{bt^*}) = k(X_{bt^*})^\top(K(X_b, X_b) + \lambda I_n)^{-1}\mathbf{y}$, where $X_{bt^*}$ is a new test point, $k$ and $K$ are kernel functions, $\lambda$ is the regularisation parameter, and $\mathbf{y}$ is the vector of $n$ observations.
    
   As the Bayesian approach to kernel regression, the solution to the risk-minimising function a GP regression model coincides with that in kernel regression, i.e. the expected value of the GP posterior predictive distribution, 
    $$
    \begin{aligned}
        \mathbb{E}[y^*(X_{bt^*})|X_b] = k(X_{bt^*})^\top(K(X_b, X_b) + \sigma^2 I_n)^{-1}\mathbf{y},
    \end{aligned}
    $$
where $\sigma^2$ is the variance of the noise.
   Under the assumption that the correct hyperparameters are learned when maximising the marginal-likelihood, which are faithful to the underlying SCM, we can draw an equivalence between a Granger non-causal relationship identified by $GP_{SIC}$ and a conditional independence relationship between variables.
\end{proof}

\begin{lemma}
    \label{lemma1}
    The set of estimated adjacencies $\hat{\mathcal{A}}(X_{bt})$, for each $b = 1, \dots, n_t$, identified as graph $G'$ is a superset of the true adjacencies, containing exactly
    $\hat{\mathcal{A}}(X_{bt}) = \mathcal{P}^-(X_{bt}),\mathcal{F}^-(X_{bt}),\mathcal{A}_t(X_{bt})$,  where $\mathcal{A}_t(X_{bt})$ are the true contemporaneous adjacencies and $\mathcal{P}^-(X_{bt})$ are the true lagged parents of $X_{bt}$. The remaining lagged conditions $X_{a(t - \tau)} \in \mathcal{F}^-(X_{bt})$ are, at most, the lagged parents of a collider that is contemporaneous to $X_{bt}$, i.e. $X_{a(t - \tau)} \rightarrow X_{ct} \leftarrow X_{bt}$.
\end{lemma}

\begin{proof}
   Under Lemma \ref{lemma:conditional_independence} and the Adjacency faithfulness assumption, the adjacencies $X_{a(t - \tau)} -  X_{bt}$ for $
   \tau \geq 0$ in $G'$ returned by $GP_{SIC}$ imply the corresponding conditional dependence relations $X_{a(t - \tau)} \nindep X_{bt} | S$, where $S = X \backslash X_{bt}$, such that $S$ is a maximal conditioning set. We show for an arbitrary pair of variables $X_{a(t - \tau)} , X_{bt} \in X$ with $\tau \geq 0$:
        \begin{enumerate}[(a)]
            \item $X_{a(t - \tau)} - X_{bt} \notin G' \implies X_{a(t - \tau)} \notin \mathcal{P}^-(X_{bt}),\mathcal{F}^-(X_{bt}),\mathcal{A}_t(X_{bt})$
        \end{enumerate}
        Assume we observe the absence of a lagged edge in $G'$, i.e. $\tau > 0$ and $X_{a(t - \tau)} \not
        \rightarrow X_{bt}$, resulting from Granger non-causality/the conditional independence relation $X_{a(t - \tau)} \ci X_{bt} | S$. This implies $X_{a(t - \tau)} \notin \mathcal{P}^-(X_{bt})$ by Causal Markov condition. Furthermore, due to the maximal conditioning set $S$, we may condition on contemporaneous colliders, i.e. $X_{a(t - \tau)} \rightarrow X_{ct} \leftarrow X_{bt}$, where it may also be $c=a$. This opens the pathway from $X_{a(t - \tau)} \rightarrow X_{bt}$, such that $X_{a(t - \tau)} \in \mathcal{F}^-(X_{bt})$, corresponding to a spurious dependency $X_{a(t - \tau)} \nindep X_{bt} | S$. However, since we observe the contrary conditional independence relation, this implies that $X_{a(t - \tau)} \notin \mathcal{F}^-(X_{bt})$. 

         The remaining case to consider is the contemporaneous non-adjacencies. If $GP_{SIC}$ returns a Granger non-causal relationship, this implies the conditional independence relation $X_{a(t - \tau)} \ci X_{bt} |S$, and by adjacency faithfulness
         the absence of a contemporaneous edge in $G'$, i.e. $\tau = 0$ and $X_{a(t - \tau)} \notminus X_{bt}$. Thus, the lack of a contemporaneous adjacency returned in $G'$ indicates that $X_{a(t - \tau)}$ is truly conditionally independent/not adjacent to $X_{bt}$, $X_{a(t - \tau)} \notin \mathcal{A}_t(X_{bt})$.
        
        \begin{enumerate}[(b)]
        \item $X_{a(t - \tau)} - X_{bt} \in G' \implies X_{a(t - \tau)} \in \mathcal{P}^-(X_{bt}),\mathcal{F}^-(X_{bt}),\mathcal{A}_t(X_{bt})$
        
        \end{enumerate}
        By (a), the presence of an edge $X_{a(t - \tau)} - X_{bt} \in G'$ implies that $X_{a(t - \tau)}$ is in a superset of $\mathcal{P}^-(X_{bt}),$ $\mathcal{F}^-(X_{bt}),\mathcal{A}_t(X_{bt})$. 
        Under the assumptions of standard faithfulness and no fully instantaneous colliders in the SCM, we can proceed with the rules of \emph{d}-separation. We begin simply with direct dependencies, i.e. $X_{a(t - \tau)} - X_{bt}$, for $\tau \geq 0$, are unblocked. This accounts for $X_{a(t - \tau)} \in \mathcal{P}^-(X_{bt})$ for $\tau > 0$. Similarly, for $\tau = 0$, a true parent or child relationship is unblocked, s.t. $X_{a(t - \tau)} \in \mathcal{A}_t(X_{bt})$.

        Then, the crux of this first step is that conditioning on a collider (or descendant of a collider) unblocks the path between parent nodes. 
        When conditioning on all contemporaneous variables, we may condition on a child of $X_{bt}$, $\mathcal{D}_t(X_{bt})$, which is a collider between a lagged node and another contemporaneous node, $X_{a(t - \tau)} \rightarrow \mathcal{D}_t(X_{bt}) \leftarrow X_{bt}$, opening the path $X_{a(t - \tau)} \rightarrow X_{bt}$. This is exactly the case where $X_{a(t - \tau)} \in \mathcal{F}^-(X_{bt})$.
        Conditioning on a descendant of a collider also opens the path $X_{a(t - \tau)} \rightarrow X_{bt}$, but as we condition on all nodes except $X_{bt}$, this effect is already considered by conditioning on the collider itself. 
        
        Now, to show that no other adjacencies are estimated, we consider that if a path $p$ contains a mediator or confounder as an internal node on the path, then this path is blocked, even if a collider is also conditioned on. Time orientation restricts the structure of the causal graph such that for $\tau > 0$, any path $X_{a(t - \tau)} \rightarrow \dots \rightarrow \mathcal{D}_t(X_{bt}) \leftarrow X_{bt}$ is blocked by conditioning on the internal variables preceding $\mathcal{D}_t(X_{bt})$. The internal nodes are guaranteed mediators as all lagged links are oriented forward in time. The other possibility is a path such as
        $X_{a(t - \tau)} \rightarrow \mathcal{D}_t(X_{bt}) \leftarrow \dots \leftrightarrow X_{bt}$, where $\mathcal{D}_t(X_{bt})$ is a collider that is a non-child descendant of $X_{bt}$. The path $X_{a(t - \tau)} \rightarrow X_{bt}$ is still blocked by the internal variables between $X_{bt}$ and $\mathcal{D}_t(X_{bt})$.
        As such, only true lagged parents or unmediated collider pathways are unblocked, corresponding to a conditional dependency $X_{a(t - \tau)} \nindep X_{bt} |S$. 

        Finally, we consider contemporaneous conditions. As we have assumed the absence of colliders within an instantaneous time point, no false pathways may be unblocked by $X_{at} \rightarrow \mathcal{D}_t(X_{bt}) \leftarrow X_{bt}$. This ensures that the only unblocked paths are direct adjacencies, as non-adjacent contemporaneous variables are mediated as a result of the maximal conditioning set, and there is no pathway backward through lagged conditions due to time-orientation. 
        
\end{proof}

\paragraph{Proof of Theorem 1}

\begin{proof}

By Lemma \ref{lemma1}, we have shown that the adjacencies recovered in graph $G'$ are a superset of the true adjacencies, where the extra conditions are at most $\mathcal{F}^-(X_{bt})$. The contemporaneous adjacencies were shown to be consistently estimated in $G'$, and so these are retained in the final graph $G'''$.

The extra conditions are filtered by fitting a second graph, $G''$, which is estimated by conditioning only on the lagged parents identified in $G'$, i.e. $G''$ retests all lagged adjacencies $X_{a(t - \tau)} \rightarrow X_{bt}$ by assessing $X_{a(t - \tau)} \ci X_{bt}| S$, where $S = \mathcal{P}^-(X_{bt}), \mathcal{F}^-(X_{bt})$. We show that any lagged adjacency retained in $G''$ is in the true parent set $\mathcal{P}^-(X_{bt})$, s.t. the final graph $G'''$ is asymptotically consistent with the true adjacencies, $\hat{\mathcal{A}}(X_{bt}) = \mathcal{P}^-(X_{bt}), \mathcal{A}_t(X_{bt})$.
We show for an arbitrary pair of variables $X_{a(t - \tau)}, X_{bt} \in X$ with $\tau > 0$:

\begin{enumerate}[(a)]
    \item  $X_{a(t - \tau)} - X_{bt} \notin G'' \implies X_{a(t - \tau)} \notin \mathcal{P}^-(X_{bt})$
\end{enumerate}

This follows similarly from Lemma \ref{lemma1} (a), in that the  conditional independence relation $X_{a(t - \tau)} \ci X_{bt} |S$ implies the absence of a lagged adjacency in $G''$. By the Causal Markov condition, $X_{a(t - \tau)} \notin \mathcal{P}^-(X_{bt})$.

\begin{enumerate}[(b)]
    \item  $X_{a(t - \tau)} - X_{bt} \in G'' \implies X_{a(t - \tau)} \in \mathcal{P}^-(X_{bt})$
\end{enumerate}

By (a), the edges in $G''$ are a superset of $\mathcal{P}^-(X_{bt})$. As we retest only lagged edges from $G'$, the only possible extra conditions are $\mathcal{F}^-(X_{bt})$. By Lemma \ref{lemma1}, a condition in $\mathcal{F}^-(X_{bt})$ is only opened due to conditioning on a contemporaneous collider. Thus, by removing the contemporaneous conditions from $S$, we block these false dependencies. It follows that any edge left in $G''$ must be in the set of true parents $\mathcal{P}^-(X_{bt})$.

The contemporaneous adjacencies from $G'$ and the lagged parents from $G''$ yield an estimated graph $G'''$ that is consistent with the adjacencies in the underlying SCM. 
\end{proof}

\paragraph{Proof of Theorem 2}

\begin{proof}
    Due to time-ordering constraints, any lagged edge is automatically oriented forward in time, i.e. $X_{a(t - \tau)} \rightarrow X_{bt}$. Given the assumption of stationarity, contemporaneous orientations are assumed to be consistent across $t = 1, \dots,n$.

    The first orientation step orients colliders based on false edges that are identified in graph $G'$. Lemma \ref{lemma1} shows that the lagged parents identified in graph $G'$ are at most the true parents $\mathcal{P}^-(X_{bt})$ or the false parents $\mathcal{F}^-(X_{bt})$ that are a result of conditioning on a collider, i.e. $X_{a(t - \tau)} \rightarrow X_{ct} \leftarrow X_{bt}$ for $\tau > 0$ and it may be $c=a$. These false parents are filtered in $G''$, s.t. only the true parents remain in $G'''$. Under consistent identification of the true parents \emph{and} false parents induced by the aforementioned collider structure, for an unshielded lagged triple $X_{a(t - \tau)} \rightarrow X_{ct} - X_{bt}$ where $X_{a(t - \tau)} \rightarrow X_{bt} \in G'$ and $X_{a(t - \tau)} \rightarrow X_{bt} \notin G'''$, we can identify $X_{ct}$ as a collider, and orient $X_{ct} \leftarrow X_{bt}$. Under our assumptions and Lemma \ref{lemma1}, all false parents due to colliders appear in $G'$ and no additional false parents appear. Thus, any unshielded lagged triple $X_{a(t - \tau)} \rightarrow X_{ct} - X_{bt}$ in $G'''$, where $X_{a(t - \tau)} \rightarrow X_{bt} \notin G'$ and $X_{a(t - \tau)} \rightarrow X_{bt} \notin G'''$, implies $X_{ct}$ is a non-collider, such that we can orient $X_{ct} \rightarrow X_{bt}$.
    
    We note one case where orientation is not identifiable given this schema. In the case described by $X_{a(t-\tau)} \rightarrow X_{ct} - X_{bt}$ and $X_{a(t-\tau)} \rightarrow X_{dt} - X_{bt}$, where $c \neq d$, a false edge from $X_{a(t-\tau)} \rightarrow X_{bt}$ has ambiguous source. This is because it may be that either one or both lagged triples are colliders. Hence, we exclude these lagged triples from being used for orientation.

    Finally, in the assumed SCM, there are no contemporaneous colliders. This allows us to orient any partially oriented fully contemporaneous triples $X_{at} \rightarrow X_{ct} - X_{bt} $ as $X_{ct} \rightarrow X_{bt}$ to avoid colliders. If the assumption of no contemporaneous colliders is relaxed by adding more tests to filter out false pathways, this orientation rule can no longer be applied.

    The arrows oriented by the $GP_{SIC}$ algorithm are asymptotically correct, but not necessarily complete. As we use a maximal conditioning set $S$, we cannot orient based on conditioning sets, and so lose power in some cases. For example, much of the power of this method comes from orienting contemporaneous edges based on lagged triples, including the use of autocorrelated edges. If there are not many lagged edges or strong, identifiable autocorrelation, we expect this method to lose power. 
    
\end{proof}

\subsection*{Supplementary Results Tables}

Tables \ref{f1_sim250} and \ref{f1_sim500} show the average F1 score and standard deviation across 100 MC runs for each simulated experiment and for $n = \{250, 500\}$, coinciding with Figure \ref{fig:boxplots_lagged}. In Figure \ref{fig:baycomp}, we report the comparison between the algorithms performed by using the Bayesian Wilcoxon signed-rank test for the case $n_t=250$.

\begin{table}
  \caption{\label{f1_sim250} Average F1 score and standard deviation (reported in parentheses) across 100 MC simulations for the five models: $GP_{SIC}$, lsNGC, KGC, KPCR, and PCMCI. The name of the time series system and true lag order $m$ are listed. Each simulation was tested on $n+m=250$ samples.}
\centering
  \begin{tabular}{|c|c|c|c|c|c|c|}
    \hline
    \bf{Time Series System} &{$\bf{GP}_{SIC}$} & \bf{lsNGC} & \bf{KGC} & \bf{KPCR} & \bf{PCMCI} \\
    \hline
     1-way Logistic & 0.98 (0.07) & 0.99 (0.07) & 0.99 (0.05) & 0.98 (0.09) & 0.24 (0.42)\\
    \hline
    2-way Logistic & 0.69 (0.09) & 0.71 (0.11) & 0.68 (0.07) & 0.70 (0.10)  & 0.41 (0.33)\\
    \hline
     Stochastic Linear & 0.93 (0.14) & 0.95 (0.12) & 0.95 (0.12) & 0.94 (0.13) & 0.97 (0.10)\\
    \hline
     Stochastic Nonlinear & 0.96 (0.14) & 0.81 (0.39) &  0.99 (0.05) & 1.00 (0.03) & 0.64 (0.27)\\
    \hline
    3 Fan-in & {0.99 (0.05)} & 0.94 (0.12) & 0.96 (0.10) & 0.96 (0.11) & 0.94 (0.11) \\
    \hline
    3 Fan-out & 0.98 (0.07) & 0.95 (0.11) & 0.98 (0.07) & 0.98 (0.07) & 0.85 (0.15)\\
    \hline
    Confounder & 0.97 (0.07) & 0.94 (0.10) & 0.97 (0.08) & 0.96 (0.09) & 0.95 (0.09) \\
    \hline
    Mediator & 0.99 (0.05) & 0.98 (0.07) & 0.99 (0.04) & 0.98 (0.06) & 0.96 (0.08) \\
    \hline
    Redundant Collider & {0.99 (0.04)} & 0.14 (0.28) & 0.02 (0.11) & 0.03 (0.13) & 0.76 (0.27) \\
    \hline
    Synergistic Collider & 0.97 (0.07) & 0.70 (0.31) & 0.98 (0.05) & 0.99 (0.03) & 0.92 (0.12) \\
    \hline
    Sync. (1-way intermediate) & {0.96 (0.14)} & 0.93 (0.18) & 0.90 (0.21) & 0.90 (0.22) & 0.48 (0.40) \\
    \hline
    Sync. (1-way strong) & 0.67 (0.46) & 0.24 (0.31) & 0.24 (0.34) & 0.26 (0.34) & 0.62 (0.21) \\
    \hline
    Sync. (2-way strong) & 0.96 (0.09) & 0.92 (0.14) & 0.92 (0.15) & 0.89 (0.19) & 0.82 (0.21)\\
    \hline
     Moran Effect & {0.79 (0.05)} & 0.55 (0.11) & 0.57 (0.08) &  0.69 (0.10) & 0.65 (0.07)\\
    \hline
     5 Linear & {0.95 (0.06)} & 0.68 (0.09) & 0.86 (0.10) & 0.74 (0.12) & 0.84 (0.08) \\
    \hline
    5 Nonlinear & 0.97 (0.05) & 0.52 (0.05) & 0.96 (0.06) & {0.99 (0.03)} & 0.77 (0.12) \\
    \hline
     8 Nonlinear & {0.90 (0.13)} & 0.63 (0.14) & 0.84 (0.21) & 0.81 (0.23) & 0.78 (0.11)\\
    \hline
     20 Nonlinear & {0.73 (0.06)} & 0.50 (0.04) & 0.47 (0.07) & 0.65 (0.06) & 0.69 (0.05)\\
    \hline
    30 Nonlinear  & 0.68 (0.06) & 0.32 (0.05) & 0.65 (0.14) & 0.68 (0.08) & 0.63 (0.03) \\
    \hline
    \end{tabular}
\end{table}

\begin{table}
  \caption{\label{f1_sim500} Average F1 score and standard deviation (reported in parentheses) across 100 MC simulations for the five models: $GP_{SIC}$, lsNGC, KGC, KPCR, and PCMCI. The name of the time series system and true lag order $m$ are listed. Each simulation was tested on $n+m=500$ samples.}
\centering
  \begin{tabular}{|c|c|c|c|c|c|}
    \hline
    \bf{Time Series System} & {$\bf{GP}_{SIC}$} & \bf{lsNGC} & \bf{KGC} & \bf{KPCR} & \bf{PCMCI} \\
    \hline
    1-way Logistic & 1.00 (0.00) & 0.99 (0.07) & 1.00 (0.03) & 1.00 (0.03) & 0.27 (0.44) \\
    \hline
    2-way Logistic & 0.67 (0.00) & {0.77 (0.15)} & 0.74 (0.14) & 0.72 (0.12) & 0.51 (0.32)\\
    \hline
    Stochastic Linear & 0.97 (0.10) & 0.97 (0.09) & 1.00 (0.03) & 0.99 (0.05) & 0.99 (0.03) \\
    \hline
     Stochastic Nonlinear & 0.95 (0.17) & 0.80 (0.40) & 0.99 (0.05) & 1.00 (0.03) & 0.68 (0.09)\\
    \hline
    3 Fan-in & {1.00 (0.00)} & 0.84 (0.13) & 0.71 (0.11) & 0.95 (0.09) & 0.97 (0.08) \\
    \hline
    3 Fan-out & {1.00 (0.00)} & 0.66 (0.12) & 0.61 (0.06) & 0.97 (0.07) & 0.73 (0.10) \\
    \hline
    Confounder & 0.99 (0.04) & 0.97 (0.08) &  0.99 (0.03) & 1.00 (0.03)  & 0.98 (0.04)\\
    \hline
    Mediator & 0.99 (0.03) & 0.97 (0.08) & 0.99 (0.03) & 0.99 (0.04) & 0.96 (0.08) \\
    \hline
    Redundant Collider & {0.99 (0.03)} & 0.41 (0.34) & 0.00 (0.00) & 0.01 (0.09) & 0.76 (0.27)\\
    \hline
    Synergistic Collider & 0.98 (0.06) & 0.70 (0.36) & 0.98 (0.06) & {1.00 (0.02)} & 0.95 (0.09)\\
    \hline
    Sync. (1-way intermediate) & {1.00 (0.00)} & 0.84 (0.19) & 0.66 (0.22) & 0.88 (0.18) & 0.16 (0.19) \\
    \hline
    Sync. (1-way strong) & 0.55 (0.50) & 0.31 (0.04) & 0.44 (0.10) & 0.62 (0.10) & 0.36 (0.06) \\
    \hline
    Sync. (2-way strong) & 1.00 (0.00) & 0.69 (0.13) & 0.52 (0.03) & 0.93 (0.14) & 0.78 (0.13) \\
    \hline
     Moran Effect & 0.82 (0.04) & 0.58 (0.09) & 0.74 (0.08) & 0.80 (0.08) & 0.68 (0.10) \\
    \hline
    5 Linear & 0.95 (0.06) & 0.59 (0.06) & {0.98 (0.04)} & 0.96 (0.07) & 0.82 (0.09)\\
    \hline
    5 Nonlinear & 0.97 (0.05) & 0.48 (0.04) & 0.99 (0.03) & 1.00 (0.02) & 0.82 (0.09) \\
    \hline
     8 Nonlinear & 1.00 (0.00) & 0.62 (0.09) & 0.99 (0.02) & 1.00 (0.01) & 0.81 (0.08)  \\
    \hline
    20 Nonlinear & 0.75 (0.08) & 0.41 (0.03) & 0.86 (0.04) & 0.86 (0.04) & 0.72 (0.04) \\
    \hline
    30 Nonlinear & 0.74 (0.06) & 0.24 (0.03) & {0.95 (0.02)} & 0.82 (0.04) & 0.65 (0.03) \\
    \hline
    \end{tabular}
\end{table}

\begin{figure}
    \includegraphics[width=0.33\linewidth,trim={2cm 1cm 2cm 2cm},clip]{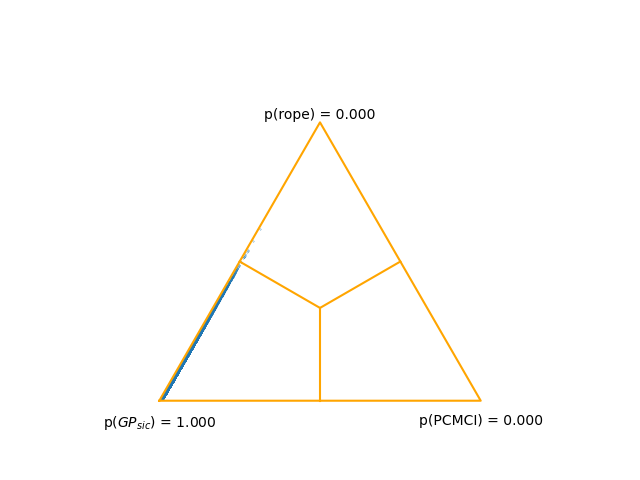}                   \includegraphics[width=0.33\linewidth,trim={2cm 1cm 2cm 2cm},clip]{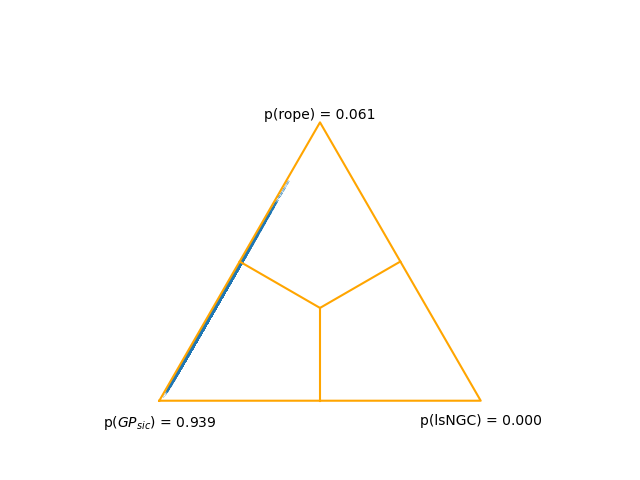}    \includegraphics[width=0.33\linewidth,trim={2cm 1cm 2cm 2cm},clip]{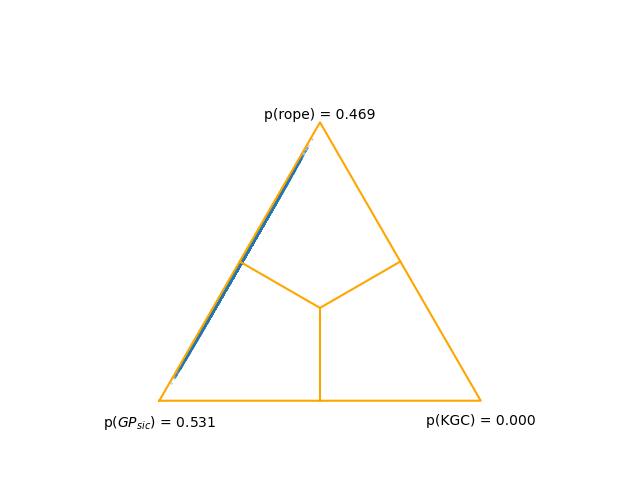}\\
        \includegraphics[width=0.33\linewidth,trim={2cm 1cm 2cm 2cm},clip]{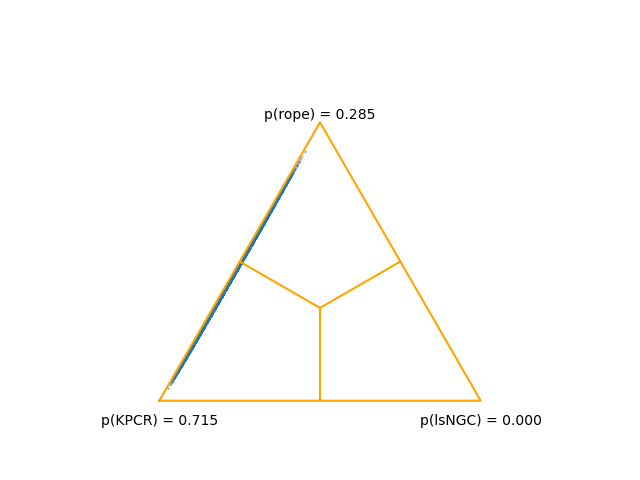}
            \includegraphics[width=0.33\linewidth,trim={2cm 1cm 2cm 2cm},clip]{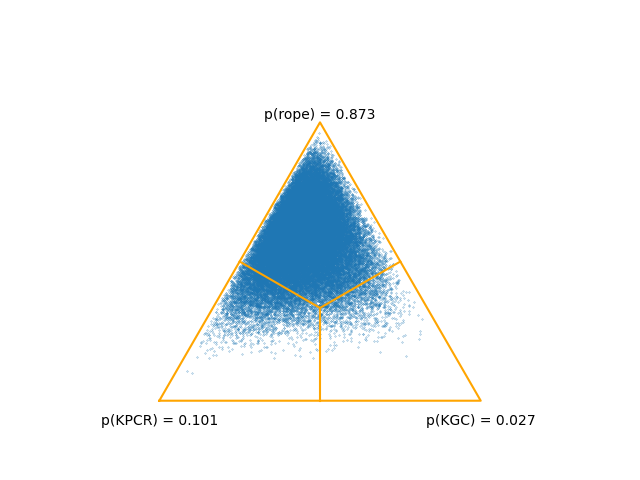}

    \caption{$n = 250$ -- For each pair of causal discovery methods $(C_1, C_2)$ (e.g., $C_1 = GP_{SIC}$ and $C_2 = PCMCI$), we report the posterior probability vector
that $C_1 > C_2$, $C_1 \equiv C_2$ and $C_1 < C_2$ denoted respectively as $[p(C_1),\; p(rope),\; p(C_2)]$,
obtained via the Bayesian Wilcoxon signed-rank test. The cloud of points in the figure corresponds to samples from this posterior distribution: each point is a probability vector
$[p(C_1),\; p(rope),\; p(C_2)]$
which we plot in the probability simplex.}
    \label{fig:baycomp}
\end{figure}

\clearpage

\subsection*{Contemporaneous Algorithm Experiments}

For contemporaneous adjacency and orientation identification, the PCMCI$+$ performs better across the experiments when autocorrelation strength $a$ is lower, while the $GP_{SIC}$ method performs better for higher autocorrelation. As a GC method the $GP_{SIC}$ contemporaneous algorithm relies on a maximal conditioning set, and as such will have lower power in comparison to methods that use reduced conditioning sets. For contemporaneous orientation, this is especially true in cases where there is not strong, identifiable autocorrelation, as the $GP_{SIC}$ method relies on orienting contemporaneous edges based on correctly identified lagged triples, which may include autocorrelated edges. When there is weak autocorrelation and autocorrelated edges go undetected, this limits the $GP_{SIC}$-based algorithm's ability to orient contemporaneous edges. 
However, the $GP_{SIC}$-based algorithm is fully defined within the GC framework, as it performs only $2n_t$ tests, compared to PCMCI$+$, which performs $O(n_t^2)$ tests.

\begin{figure}
    \centering
    \includegraphics[width=0.48\linewidth]{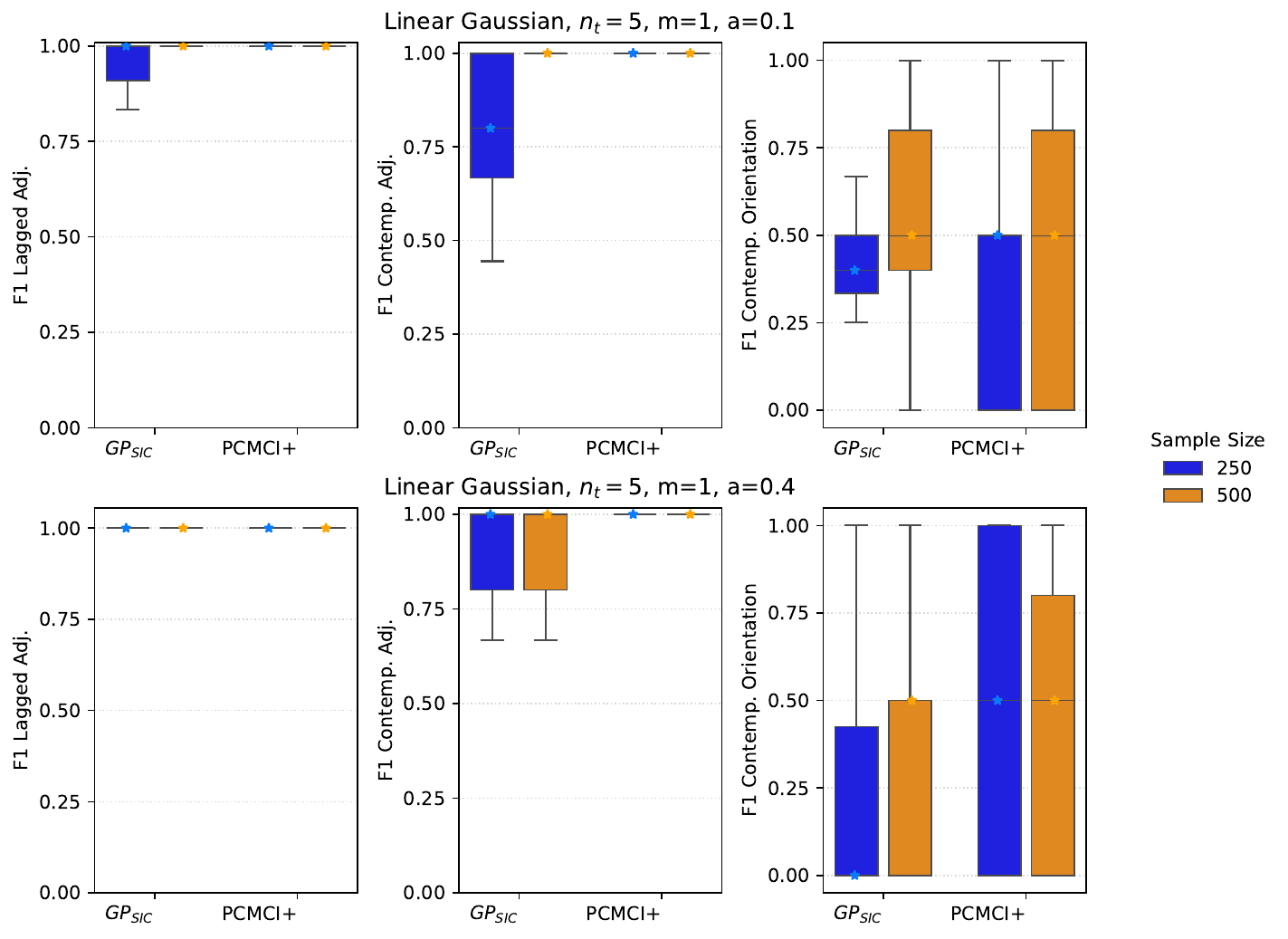}
    \includegraphics[width=0.48\linewidth]{Fig4a.pdf}
    \includegraphics[width=0.48\linewidth]{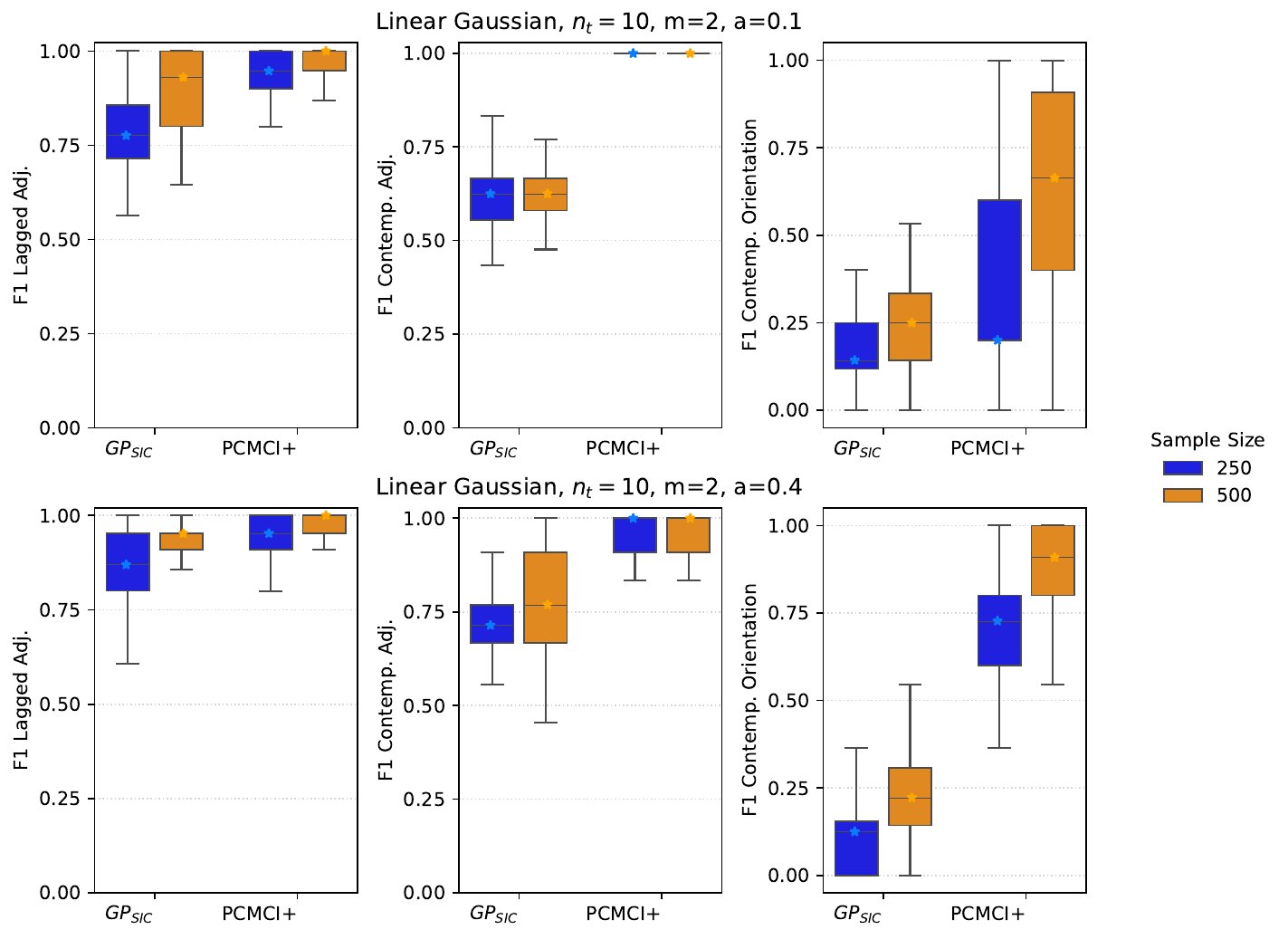}
    \includegraphics[width=0.48\linewidth]{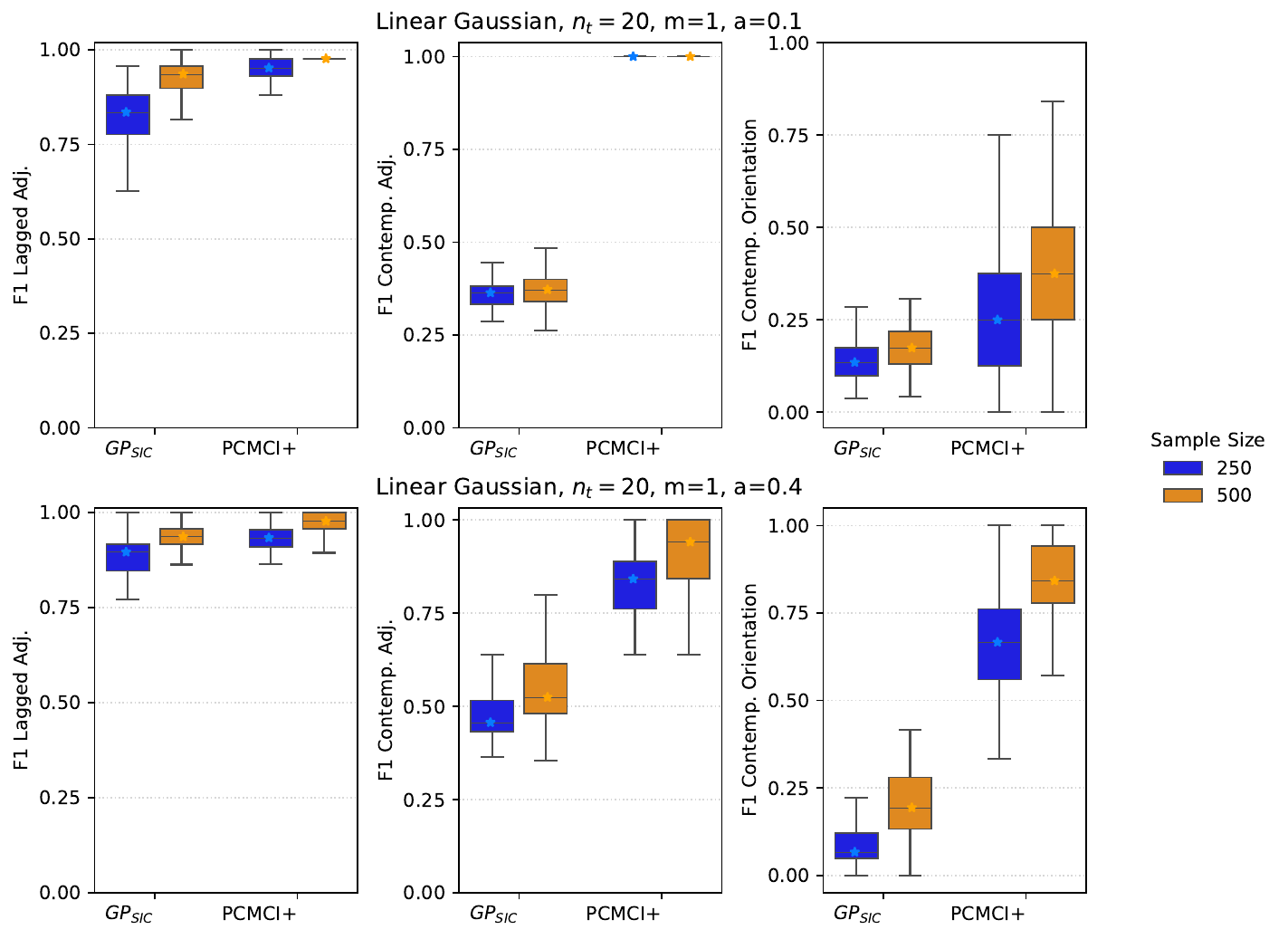}
    \caption{F1 scores for the contemporaneous linear, Gaussian noise experiments. }
    \label{fig:placeholder}
\end{figure}

\begin{figure}
    \centering
    \includegraphics[width=0.48\linewidth]{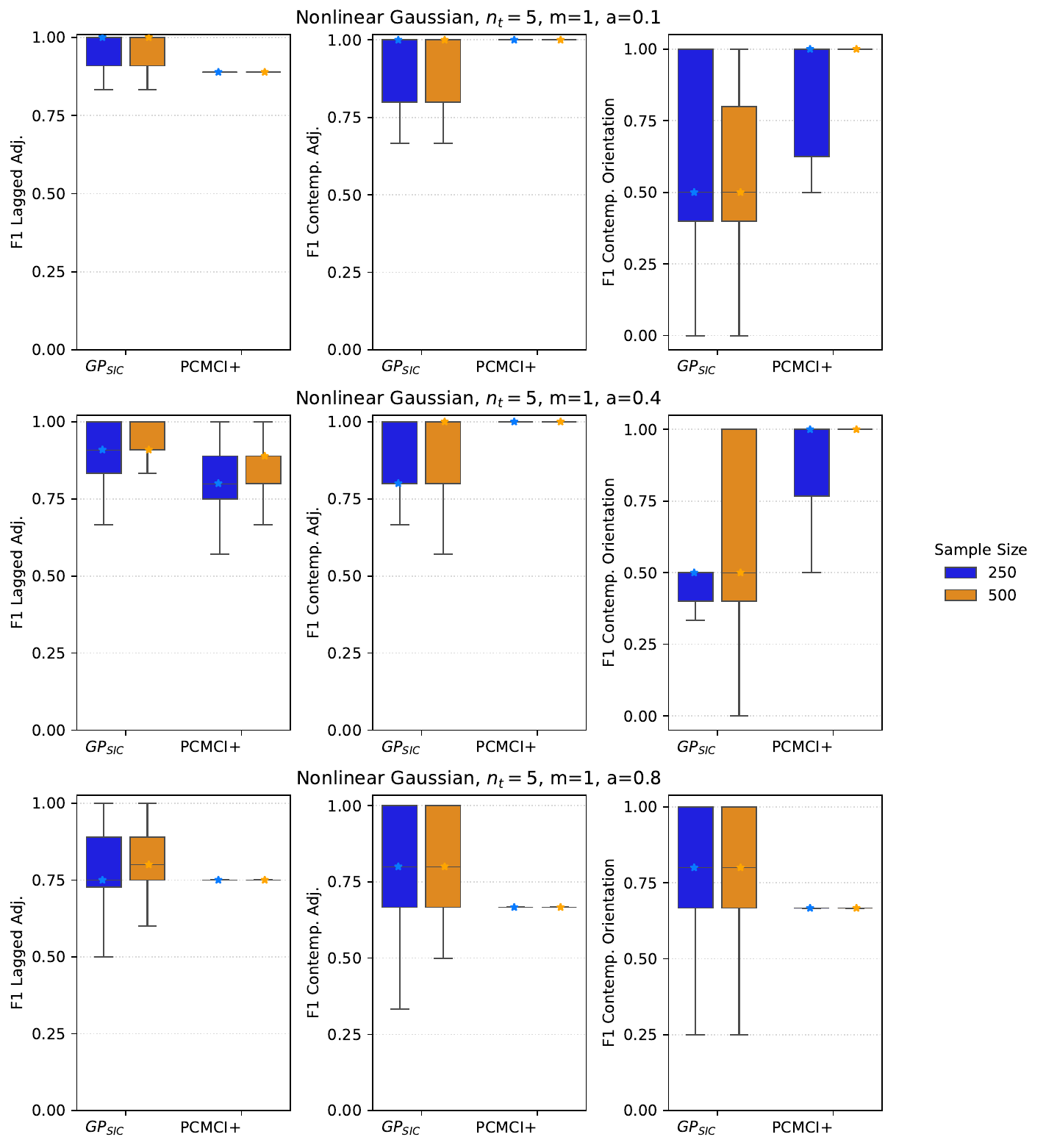}
    \includegraphics[width=0.48\linewidth]{Fig4b.pdf}
    \includegraphics[width=0.48\linewidth]{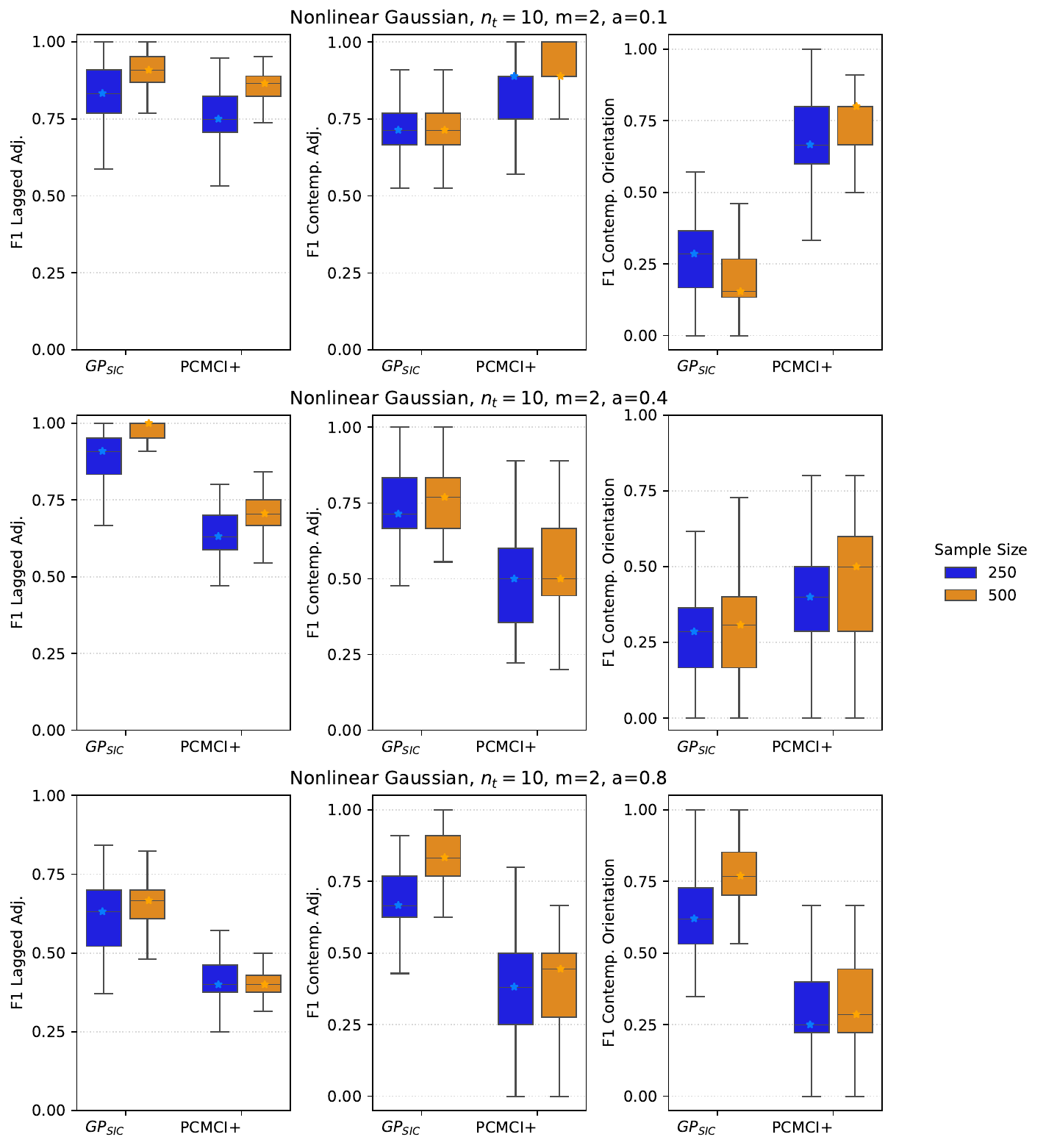}
    \includegraphics[width=0.48\linewidth]{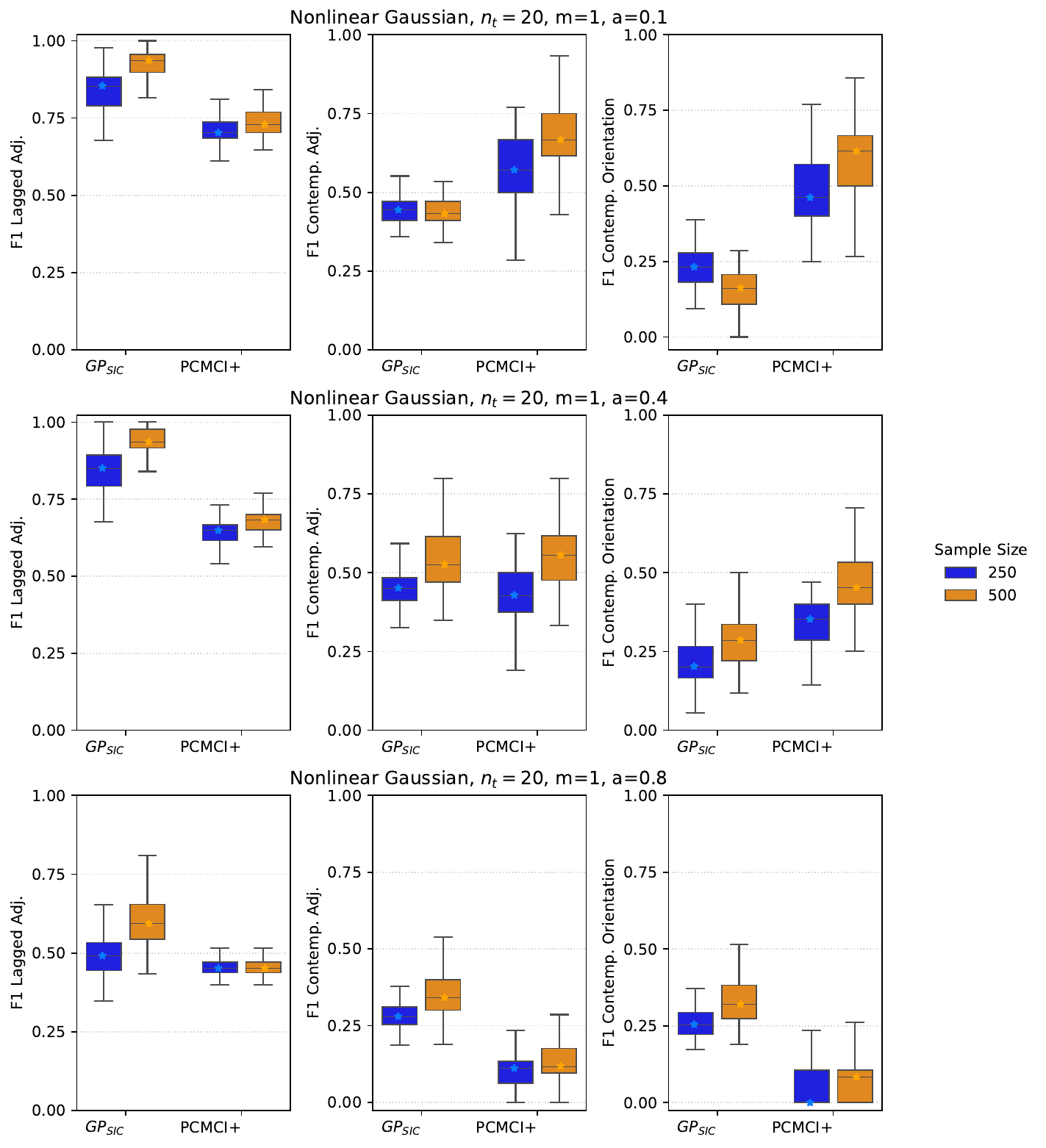}
    \caption{F1 scores for the contemporaneous nonlinear, Gaussian noise simulated experiments.}
    \label{fig:placeholder}
\end{figure}

\clearpage

\subsection*{pH Plant Application}
\begin{figure}[h]
    \centering
    \includegraphics[width=0.5\linewidth]{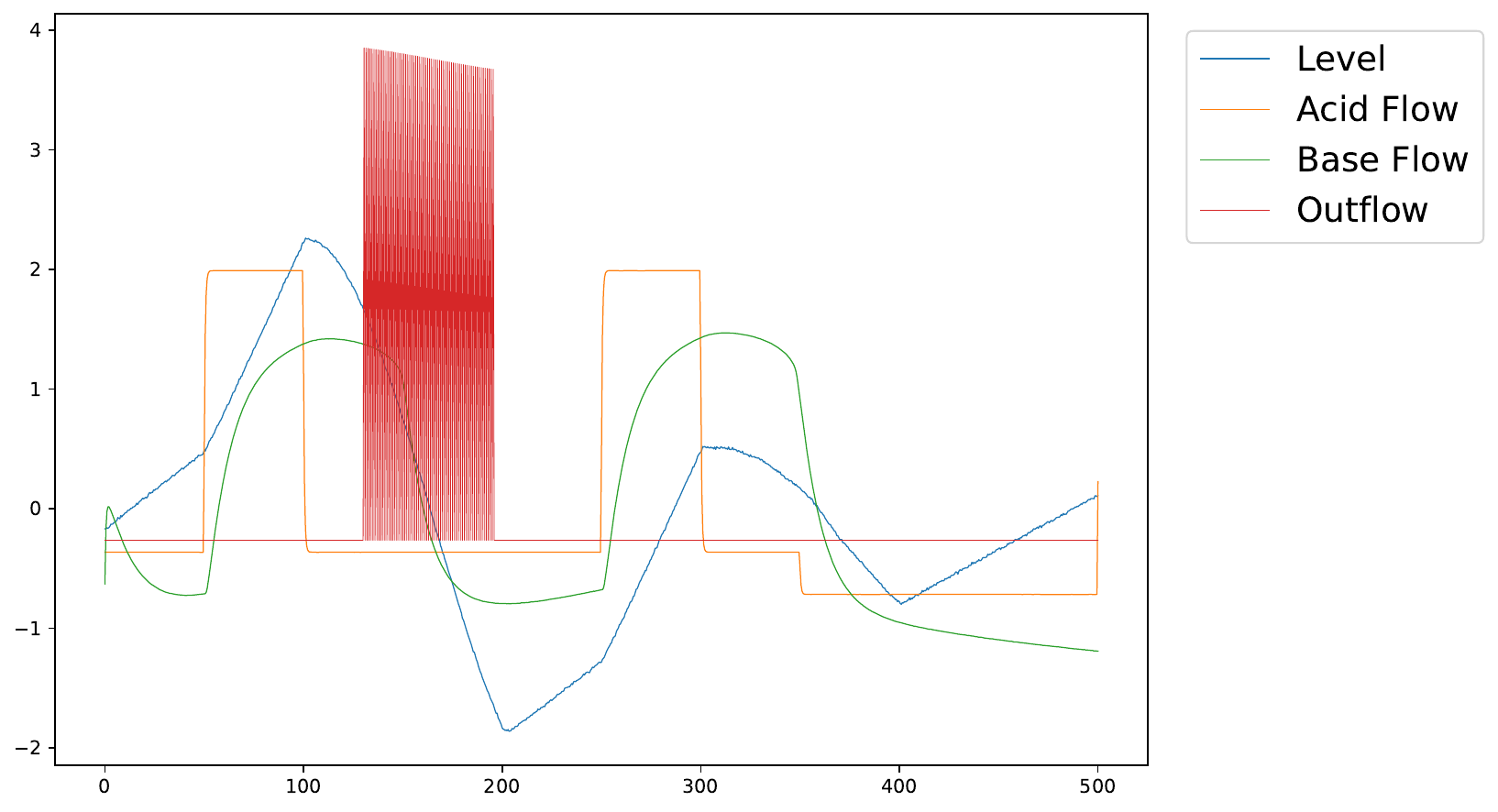}
    \includegraphics[width=0.5\linewidth]{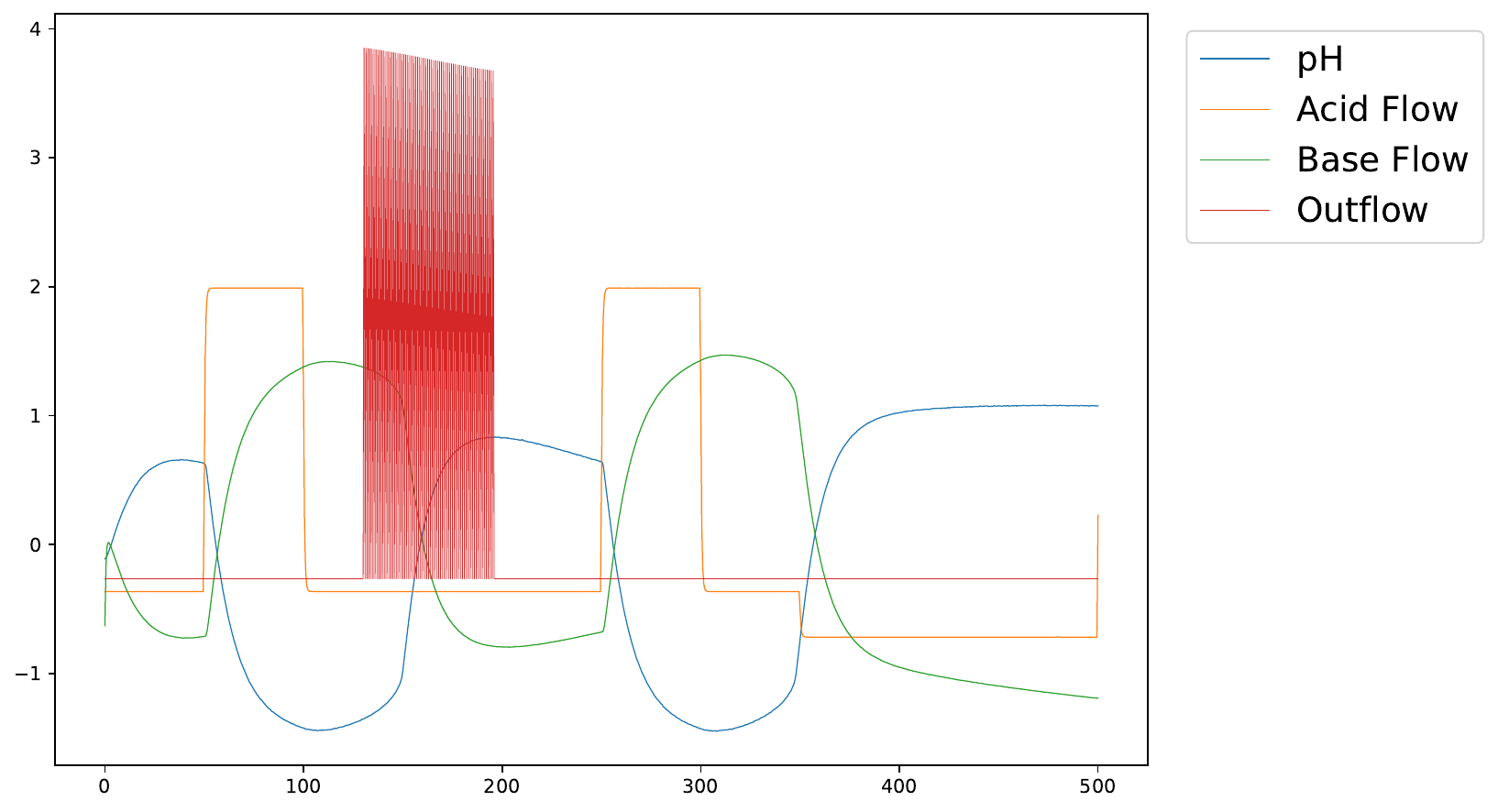}
    \caption{The normalised time series data for the level interactions (top) and pH interactions (bottom) for one MC run in the regulatory mode.}
    \label{fig:timeseries}
\end{figure}
\begin{figure}
    \centering
    \includegraphics[width=0.95\linewidth]{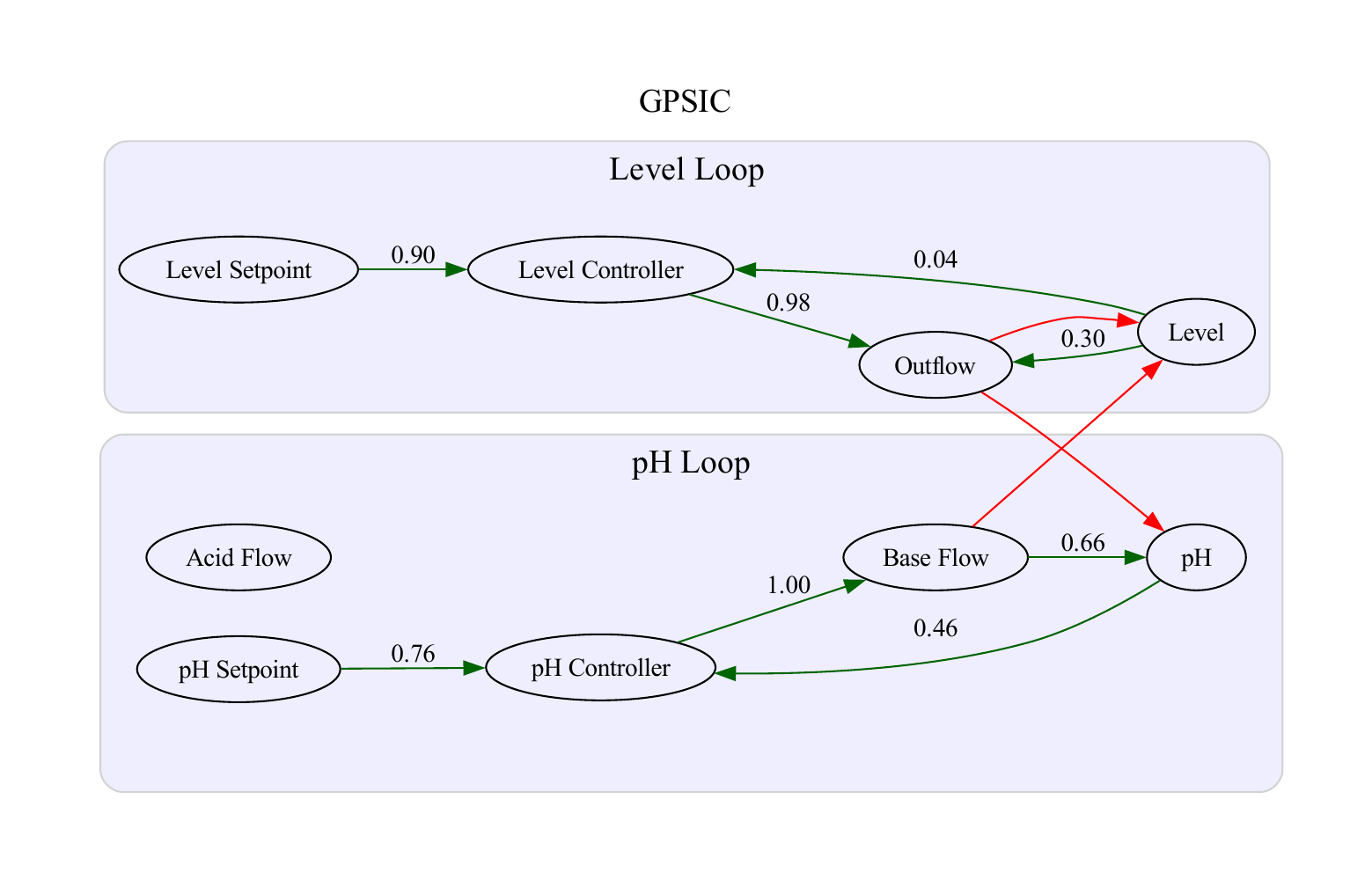}
   \includegraphics[width=0.95\linewidth]{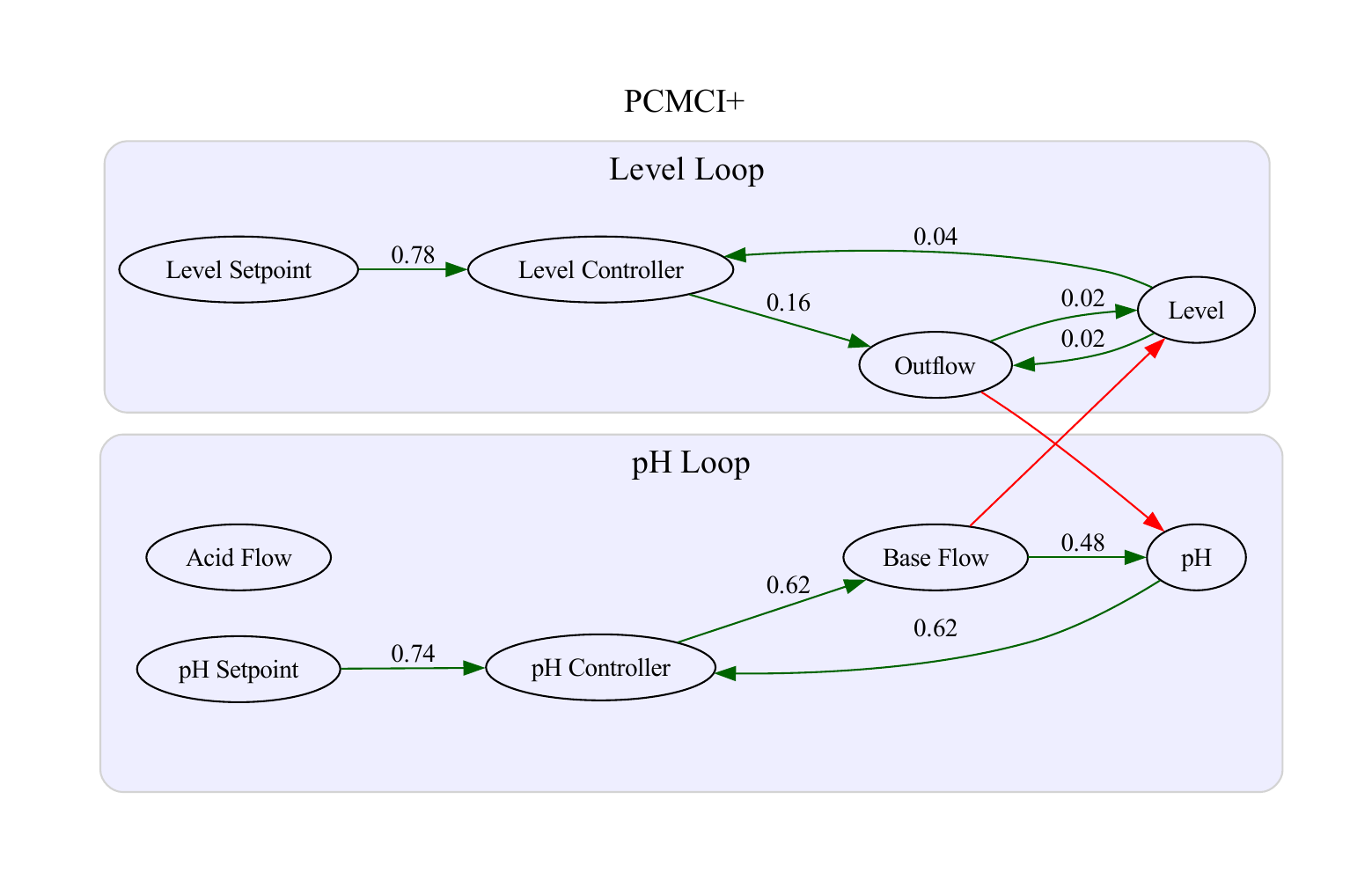}
    \caption{The causal diagram recovered from the $GP_{SIC}$-based algorithm and PCMCI$+$ for the pH neutralisation plant data in \emph{servo} mode. Green arrows correspond to true edges and are labeled with the rate of identified edges for 50 MC simulations. Red arrows correspond to false negatives. False positives were excluded from the diagrams for clarity.}
    \label{fig:phPlant_diagrams}
\end{figure}
\begin{figure}
    \centering
     \includegraphics[width=0.95\linewidth]{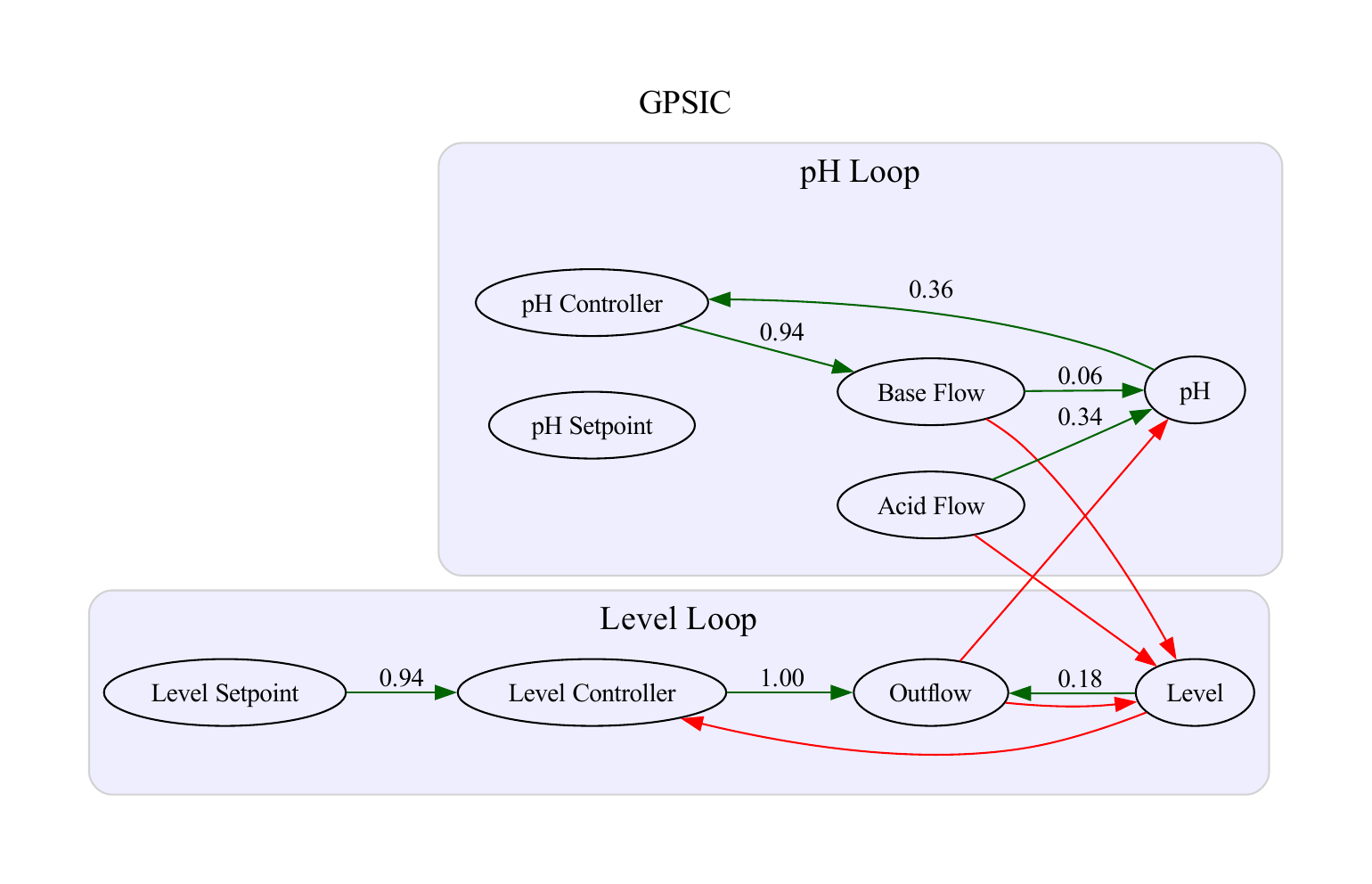}
   \includegraphics[width=0.95\linewidth]{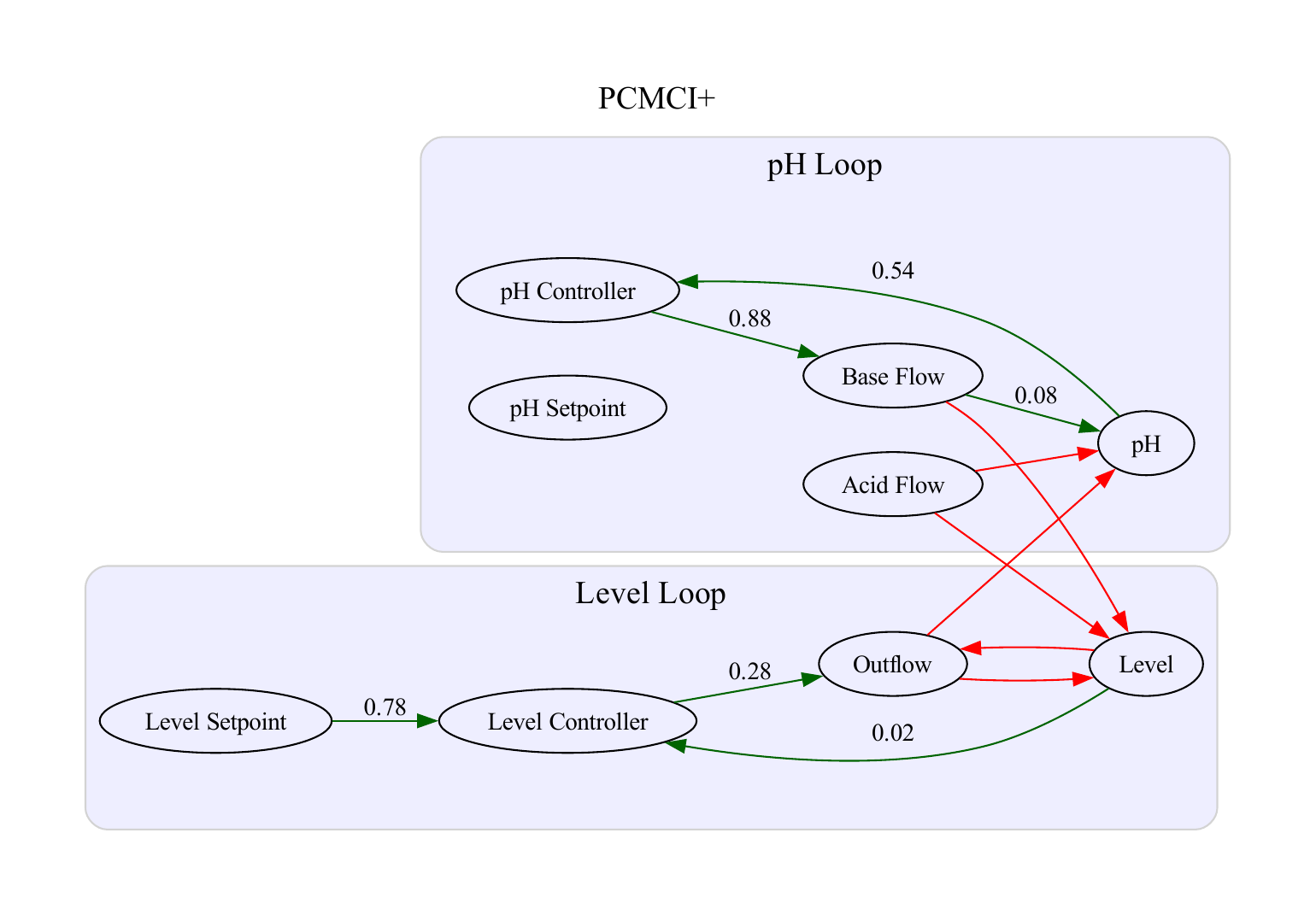}
    \caption{The causal diagram recovered from the $GP_{SIC}$-based algorithm and PCMCI$+$ for the pH neutralisation plant data in \emph{regulatory} mode. Green arrows correspond to true edges and are labeled with the rate of identified edges for 50 MC simulations. Red arrows correspond to false negatives. False positives were excluded from the diagrams for clarity.}
    \label{fig:phPlant_diagrams}
\end{figure}

\clearpage
\subsection*{Method Implementations}

\paragraph{LsNGC}
The Large-Scale Nonlinear Granger Causality method was evaluated following the implementation described in \citep{wismuller_large-scale_2021}.
That is, we used False Discovery Rate correction for multiple tests, and the number of cluster centers was set to $c_f = 25$ and $c_g = 5$ for each simulated dataset.

\paragraph{KGC}
The Kernel Granger Causality method was evaluated according to the implementation described in \citep{marinazzo_kernel_2008}. The method was evaluated with a Square Exponential kernel. The lengthscale for the SE kernel was set according to the heuristic $l = 2n_tm$.

\paragraph{PCMCI}
The PCMCI and PCMCI$+$ methods, proposed in \citep{runge_detecting_2019, runge_discovering_2022}, were evaluated using the Tigramite library. There are several different Conditional Independence Tests that can be selected as the first phase of the PCMCI algorithm for causal identification. We used the Gaussian Process Distance Correlation (GPDC) conditional independence test, which is suitable for identifying a variety of nonlinear (and linear) causal relationships \citep{runge_detecting_2019}. We used the default setting of no multiple testing correction for both PCMCI and PCMCI$+$.

\paragraph{KPCR and $\bf{GP_{SIC}}$}
We have made the code available for the KPCR and $GP_{SIC}$ methods. We have also included our implementations for the GP-based GC methods \citep{amblard_gaussian_2012, cui_gaussian_2022, zaremba_statistical_2022}. An analogous method to the GP method from \cite{cui_gaussian_2022} was implemented by removing the SIC penalisation from the $GP_{SIC}$ method. The GP GLRT \cite{zaremba_statistical_2022} was implemented with a critical value $\alpha=0.05$.

\subsection*{Simulation Setup}

\paragraph{Two Variable Systems}

The Logistic One-way and Logistic Two-way systems \citep{martinez-sanchez_decomposing_2024, sugihara_detecting_2012} are defined by 
\begin{align}
    \label{twospecies}
    &x_1(t+1) = x_1(t)(r_1 - r_1x_1(t) - \gamma_{1,2}x_2(t)) + \epsilon(t)~, \\
    &x_2(t+1) = x_2(t)(r_2 - r_2x_2(t) - \gamma_{2,1}x_1(t)) + \epsilon(t)~,
\end{align}
where $\epsilon(t)$ is additive Gaussian noise with variance $0.01$, and the initial values were set to $x_1(0) = 0.2$ and $x_2(0) = 0.4$. The first $50$ time points were excluded. The Logistic One-way system has the parameter values: $r_1 = 3.7$, $r_2 = 3.7$, $\gamma_{2,1} = 0.2$, $\gamma_{1,2} = 0$. The Logistic Two-way system has the parameter values: $r_1 = 3.5$, $r_2 = 3.9$, $\gamma_{2,1} = 0.2$, $\gamma_{1,2} = 0.01$.

The Stochastic Linear system \citep{martinez-sanchez_decomposing_2024, richardson_1922} is defined as 
 \begin{align}
     &x_1(t + 1) = 0.95\sqrt{2}x_1(t) - 0.9025x_1(t - 1) + \epsilon_1(t) ~, \\
     &x_2(t + 1) = 0.5x_1(t - 1) + \epsilon_2(t) ~, 
 \end{align}
 where $\epsilon_i \sim N (0, 1)$. The Stochastic Nonlinear system \citep{martinez-sanchez_decomposing_2024, bueso_explicit_2020} is defined as 
 \begin{align}
     &x_1(t + 1) = 3.4x_1(t)(1 - x_1(t)^2)\exp(-x_1(t - 1)^2) + \epsilon_1(t) ~, \\
     &x_2(t + 1) = 3.4x_2(t)(1 - x_2(t)^2)\exp(-x_2(t)^2) + \frac{x_1(t-1)x_2(t)}{2} + \epsilon_2(t)
 \end{align}
where $\epsilon_i \sim N (0, 1)$.

\paragraph{Three Variable Systems}
The Three Fan-in and Three Fan-out systems \citep{wismuller_large-scale_2021, ma_detecting_2014}  are described by
\begin{align}
    x_j(t+1) = x_j(t)(\gamma_{jj} - \sum_{i=1,2,3}\gamma_{ji}x_i(t)) ~, j=1,2,3 ~,
\end{align}
where the parameters for the fan-in system are $\gamma_{1,1}=4, \gamma_{2,2}=3.6, \gamma_{3,3} = 2.12, \gamma_{3,1}=0.636, \gamma_{3,2} = -0.636$, and all others zero. The fan-out case uses the parameters $\gamma_{1,1}=4, \gamma_{2,2}=3.1, \gamma_{3,3} = 2.12, \gamma_{2,1}=0.21, \gamma_{3,1} = -0.636$, and all others zero. We used initial conditions drawn from a uniform distribution on the interval $[0,1)$ for the Fan-in and Fan-out examples.We discard the first $50$ time points generated for each time series. 

The Mediator, Confounder, Synergistic Collider, and Redundant Collider systems are described in \citep{martinez-sanchez_decomposing_2024} as standard examples of different types of interactions that can occur in systems with more than two variables, where confounding effects may be present. The Mediator system is defined by 
\begin{align}
    &x_1(t + 1) = \sin(x_2(t)) + 0.001\epsilon_1(t) ~, \\
    &x_2(t + 1) = \cos(x_3(t)) + 0.01\epsilon_2(t) ~, \\
    &x_3(t + 1) = 0.5x_3(t) + 0.1\epsilon_3(t) ~. 
\end{align}
The Confounder system is defined by 
\begin{align}
    &x_1(t+1) = \sin(x_1(t) + x_3(t)) + 0.01\epsilon_1(t) ~, \\
    &x_2(t+1) = \cos(x_2(t) - x_3(t)) + 0.01\epsilon_2(t) ~, \\
    &x_3(t+1) = 0.5x_3(t) + 0.1\epsilon_3(t) ~. 
\end{align}
The Synergistic Collider system is defined by 
\begin{align}
    &x_1(t+1) = \sin(x_2(t)x_3(t)) + 0.001\epsilon_1(t) ~, \\
    &x_2(t + 1) = 0.5x_2(t) + 0.1 \epsilon_2(t) ~, \\
    &x_3(t+1) =  0.5x_3(t) + 0.1 \epsilon_3(t) ~.
\end{align}
The Redundant Collider system is defined by 
\begin{align}
    &x_1(t+1) = 0.3x_1(t) + \sin(x_2(t)x_3(t)) + 0.001\epsilon_1(t) ~, \\
    &x_2(t + 1) = 0.5x_2(t) + 0.1 \epsilon_2(t) ~, \\
    &x_3(t+1) \equiv x_2(t+1) ~.
\end{align}
In each example, $\epsilon_i(t) \sim N (0, 1)$, and the first 50 time points were excluded. 

The Synchronization system \citep{martinez-sanchez_decomposing_2024} with various coupling scenarios is defined by 
\begin{align}
    &x_1(t + 1) = r_1x_1(t)(1 - x_1(t))+ \epsilon_1(t), \\
    &x_2(t + 1) = r_2(\frac{x_2(t) + c_{1,2}x_1(t)}{1 + c_{1,2}})(1 - \frac{x_2(t) + c_{1,2}x_1(t)}{1 + c_{1,2}}) + \epsilon_2(t), \\
    &x_3(t+1) = r_3(\frac{x_3(t) + c_{1,2,3}x_1(t) + c_{1,2,3}x_2(t)}{1 + 2c_{1,2,3}})(1 - \frac{x_3(t) + c_{1,2,3}x_1(t) + c_{1,2,3}x_2(t)}{1 + 2c_{1,2,3}}) + \epsilon_3(t).
\end{align}
The parameters $r_1 = 3.68, r_2=3.67, r_3=3.78$ were used for every coupling scenario. The coupling parameters were varied by scenario: One-way intermediate uses $c_{1,2} = 0.1, c_{1,2,3}=0$, One-way strong uses $c_{1,2} = 1, c_{1,2,3} = 0$ and Two-way strong uses $c_{1,2} = 0, c_{1,2,3} = 1$. The noise terms are independently sampled from $\epsilon_i(t) \sim N(0, 1e-5)$.

\paragraph{Moran-Effect Model}
The Moran Effect \citep{moran_statistical_1953, martinez-sanchez_decomposing_2024} model is defined by the equations
\begin{align}
    &R_i(t + 1) = r_iN_i(t)(1 - N_i(t))e^{-\psi_iV(t)} ~, \\
    &N_i(t + 1) = s_iN_i(t) + \max(R_i(t - D_i), 0) ~,
\end{align}
for $i=1,2$. $N_i$ are the variables of the system, $R_i$ are mediators of a shared external influence $V \sim N(0,1)$. Thus, while correlated, the variables $N_1$ and $N_2$ maintain causal independence, representing the ecological phenomenon of simultaneous changes in the sizes of two non-interacting populations due to a shared environmental factor. The parameters we used to simulate the data are as follows: $r_1=3.4, r_2 = 2.9, s_1 = 0.4, s_2 = 0.35, D_i = 4, \psi_1 = 0.5, \psi_2 = 0.6$. The initial value for both $R_i$ was 1, and for both $N_i$ was $0.5$. We drop out the first $50$ generated time points. 

\paragraph{Five Variable Systems}
The Five Linear system \citep{wismuller_large-scale_2021, baccala_partial_2001} is defined as 
\begin{align}
    &x_1(t) = x_1(t) + 0.95\sqrt{2}x_1(t-1)-0.9025x_1(t-2) + \sigma\epsilon(t)~, \\
    &x_2(t) = x_2(t) + 0.5x_1(t - 2) + \sigma\epsilon(t) ~, \\
    &x_3(t) = -0.4x_1(t-3) + \sigma\epsilon(t)~, \\
    &x_4(t) = x_4(t) - 0.5x_1(t-2) + 0.5\sqrt{2}x_4(t-1) + 0.25\sqrt{2}x_5(t-1) + \sigma\epsilon(t)~,\\
    &x_5(t) = x_5(t) - 0.5\sqrt{2}x_4(t-1) + 0.5\sqrt{2}x_5(t-1) + \sigma\epsilon(t) ~,
\end{align}
where $\epsilon(t) \sim N (0, 1)$ and $\sigma = 0.01$. The Five Nonlinear system \citep{wismuller_large-scale_2021, baccala_partial_2001} incorporates nonlinear dependencies via 
\begin{align}
    &x_1(t) = x_1(t) + 0.95\sqrt{2}x_1(t-1)-0.9025x_1(t-2) + \sigma\epsilon(t)~, \\
    &x_2(t) = x_2(t) + 0.5x_1^2(t - 2) + \sigma\epsilon(t) ~, \\
    &x_3(t) = -0.4x_1(t-3) + \sigma\epsilon(t)~, \\
    &x_4(t) = x_4(t) - 0.5x_1^2(t-2) + 0.5\sqrt{2}x_4(t-1) + 0.25\sqrt{2}x_5(t-1) + \sigma\epsilon(t)~,\\
    &x_5(t) = x_5(t) - 0.5\sqrt{2}x_4(t-1) + 0.5\sqrt{2}x_5(t-1) + \sigma\epsilon(t) ~,
\end{align}
where $\epsilon(t) \sim N (0, 1)$ and $\sigma = 0.01$. We dropped out the first $50$ time points that were generated, and used initial conditions drawn from the standard normal distribution $N(0,1)$.

\paragraph{Eight Variable System}
We adapt the Eight Species system from \citep{martinez-sanchez_decomposing_2024, leng_partial_2020} via the model
\begin{align}
    &x_1(t + 1) = x_1(t)(3.9 - 3.9x_1(t)) + \epsilon_1(t) ~,\\
    &x_2(t + 1) = x_2(t)(3.5 - 3.5x_2(t)) + \epsilon_2(t) ~, \\
    &x_3(t + 1) = x_3(t)(3.62 - 3.62x_3(t) - 0.35x_1(t) - 0.35x_2(t)) + \epsilon_3(t) ~, \\
    &x_4(t + 1) = x_4(t)(3.75 - 3.75x_4(t) - 0.35x_2(t)) + \epsilon_4(t) ~, \\
    &x_5(t + 1) = x_5(t)(3.65 - 3.65x_5(t) - 0.35x_3(t)) + \epsilon_5(t) ~, \\
    &x_6(t + 1) = x_6(t)(3.72 - 3.72x_6(t) - 0.35x_3(t)) + \epsilon_6(t) ~, \\
    &x_7(t + 1) = x_7(t)(3.57 - 3.57x_7(t) - 0.35x_6(t)) + \epsilon_7(t) ~, \\
    &x_8(t + 1) = x_8(t)(3.68 - 3.68x_8(t) - 0.35x_6(t)) + \epsilon_8(t) ~,
\end{align}
where $\epsilon_i(t) \sim N(0, 0.005)$. The first $50$ time points are dropped out. 

\paragraph{20 and 30 Nonlinear Systems}
The 20 and 30 node nonlinear systems were simulated according to the model 
\begin{align}
    x_j(t+1) = ax_j(t) + \sum_icf(x_i(t - \tau_i)) + \epsilon_j(t), \epsilon_j(t) \sim N(0, 1),
\end{align}
where $a=0.4, c=0.4$ and $f(x) = (1 + 5xe^{-\frac{x^2}{20}})x$. The number of edges in the true causal graph was selected according to $L = \lfloor 1.5n_t\rfloor$, and edges were randomly assigned. 

\paragraph{Contemporaneous Systems}

The Contemporaneous systems were simulated in a similar manner to the numerical experiments in \cite{runge_discovering_2022}. We simulated the data according to the model
\begin{align}
    x_j(t+1) = ax_j(t) + \sum_icf(x_i(t - \tau_i)) + \epsilon_j(t), \epsilon_j(t) \sim N(0, 1),
\end{align}
where the autocorrelation and causal magnitude parameters were selected as either $(a, c) = (0.4, 0.4)$ or $(a, c) = (0.8, 0.4)$ for the nonlinear examples to assess the effect of varying degrees of autocorrelation. For the linear case with $n_t=5, m=3$, we tested $a=\{0.4, 0.5\}$.
The nonlinear function used for the contemporaneous systems was $f(x) = (1 + 5xe^{-\frac{x^2}{20}})$. The number of edges in the true causal graph was selected according to $L = \lfloor 1.5n_t\rfloor$, with  $30\%$ contemporaneous edges, and edges were randomly assigned. 

\end{document}